\documentclass[10pt
  ]{article}

\usepackage[top=0.9in, bottom=0.9in, left=1.0in, right=1.0in]{geometry}

\usepackage{wrapfig}
\usepackage[round]{natbib}

\usepackage{hyperref}

\usepackage{amsmath}
\usepackage{amssymb}
\usepackage{adjustbox}
\usepackage{mathtools}
\usepackage{amsthm}

\usepackage{caption}

\theoremstyle{plain}
\newtheorem{theorem}{Theorem}%

\newtheorem{lemma}[theorem]{Lemma}

\theoremstyle{definition}

\theoremstyle{remark}

\usepackage[utf8]{inputenc} %
\usepackage[T1]{fontenc}    %
\usepackage{url}            %
\usepackage{booktabs}       %
\usepackage{amsfonts}       %
\usepackage{nicefrac}       %
\usepackage{xcolor}         %

\usepackage{amsmath,amsfonts,bm,amsthm,amssymb,mathrsfs,algorithm,algorithmic}
\usepackage{hyperref,url,booktabs,enumitem,multicol,graphicx,caption,subcaption}

\usepackage{adjustbox}
\usepackage{hhline}
\usepackage{multirow}
\usepackage{arydshln}

\def\eqref#1{(\ref{#1})}

\def\1{\bm{1}}

\newcommand{\E}{\mathbb{E}}

\newcommand{\ba}[1]{\begin{align}#1\end{align}}

\usepackage{stackengine}

\makeatletter
\newcommand{\distas}[1]{\mathbin{\overset{#1}{\kern\z@\sim}}}%

\newcommand{\cL}{\mathcal{L}}

\newcommand{\beqs}{\vspace{0mm}\begin{eqnarray}}
\newcommand{\eeqs}{\vspace{0mm}\end{eqnarray}}
\newcommand{\barr}{\begin{array}}
\newcommand{\earr}{\end{array}}

\def\gL{{\mathcal{L}}}

\newcommand{\cv}[0]{{\boldsymbol{c}}}

\newcommand{\xv}{\boldsymbol{x}}
\newcommand{\yv}{\boldsymbol{y}}
\newcommand{\zv}{\boldsymbol{z}}

\newcommand{\epsilonv}{\boldsymbol{\epsilon}}

\newcommand{\given}{\,|\,}

\usepackage{amsmath}
\usepackage{amssymb}
\usepackage{mathtools}
\usepackage{amsthm}
\usepackage{graphicx}
\usepackage{cleveref}
\usepackage{multirow}

\usepackage{tabularray}
\UseTblrLibrary{booktabs}

\usepackage{color, colortbl}
\definecolor{hl_color}{gray}{0.925}

\usepackage{bbm}
\usepackage{algorithm}
\usepackage{algorithmic}

\usepackage{authblk}

\usepackage{geometry}

\usepackage[final,expansion=alltext]{microtype}
\usepackage[english]{babel}
\usepackage[parfill]{parskip}
\usepackage{afterpage}
\usepackage{framed}

\renewenvironment{abstract}%
{%
  \vskip 0.075in%
  \centerline%
  {\large\bf Abstract}%
  \vspace{0.5ex}%
  \begin{quote}%
}
{
  \par%
  \end{quote}%
  \vskip 1ex%
}

\usepackage{longtable}
\usepackage{mdframed}

\usepackage{listings}
\definecolor{codegreen}{rgb}{0,0.6,0}
\definecolor{codegray}{rgb}{0.5,0.5,0.5}
\definecolor{codepurple}{rgb}{0.58,0,0.82}
\definecolor{backcolour}{rgb}{0.95,0.95,0.92}

\lstdefinestyle{mystyle}{
    backgroundcolor=\color{backcolour},   
    commentstyle=\color{codegreen},
    keywordstyle=\color{magenta},
    numberstyle=\tiny\color{codegray},
    stringstyle=\color{codepurple},
    basicstyle=\ttfamily\footnotesize,
    breakatwhitespace=false,         
    breaklines=true,                 
    captionpos=b,                    
    keepspaces=true,                 
    numbers=left,                    
    numbersep=5pt,                  
    showspaces=false,                
    showstringspaces=false,
    showtabs=false,                  
    tabsize=2
}

\title{
\Large \textbf{%
Few-Step Diffusion via Score identity Distillation} 
}

\author{
  Mingyuan Zhou\footnote{UT Austin. Part of the work was done while the author was at Google. Now visiting Google DeepMind.},\quad
  Yi Gu\footnote{UT Austin.},\quad
  Zhendong Wang\footnote{Part of the work was done while the author was at UT Austin. Now at Microsoft.}\\
  \texttt{mingyuan.zhou@mccombs.utexas.edu},  \texttt{\{yigu,zhendong.wang\}@utexas.edu}
}

\begin{document}

\maketitle

\begin{figure*}[!ht]
    \centering
     \vspace{-6.0mm}
    \includegraphics[width=0.9\textwidth]{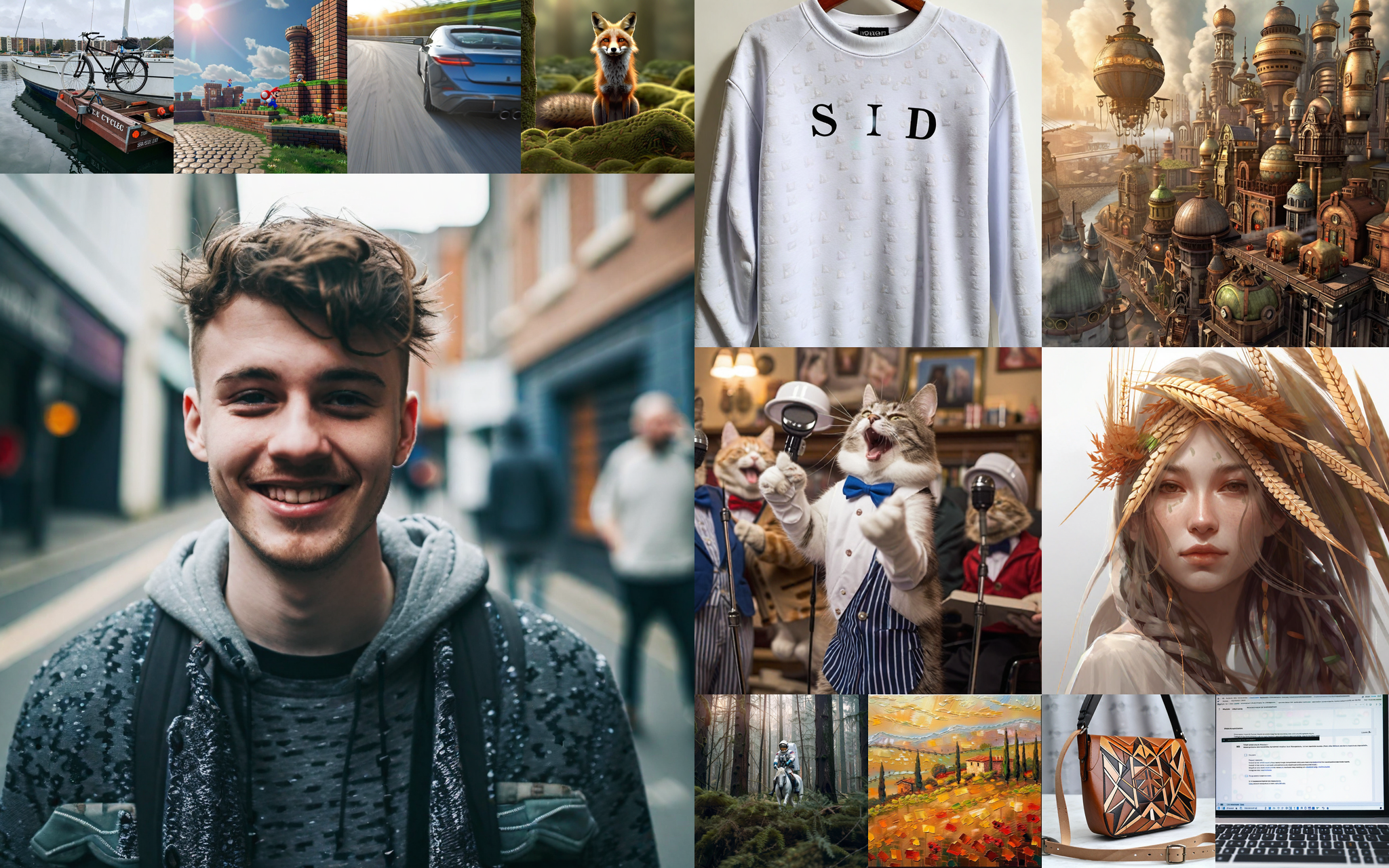}
    \vspace{-2.0mm}
  \caption{\small Example four-step generations at 1024×1024 using our SiD-based multistep distillation method.}
    \label{fig:qualitative}
\end{figure*}

\begin{abstract}

Diffusion distillation has emerged as a promising strategy for accelerating text-to-image (T2I) diffusion models by distilling a pretrained score network into a one- or few-step generator. While existing methods have made notable progress, they often rely on real or teacher-synthesized images to perform well when distilling high-resolution T2I diffusion models such as Stable Diffusion XL (SDXL), and their use of classifier-free guidance (CFG) introduces a persistent trade-off between text–image alignment and generation diversity. We address these challenges by optimizing Score identity Distillation (SiD)—a data-free, one-step distillation framework—for few-step generation. Backed by theoretical analysis that justifies matching a uniform mixture of outputs from all generation steps to the data distribution, our few-step distillation algorithm avoids step-specific networks and integrates seamlessly into existing pipelines, achieving state-of-the-art performance on SDXL at 1024×1024 resolution. To mitigate the alignment–diversity trade-off when real text–image pairs are available, we introduce a Diffusion GAN–based adversarial loss applied to the uniform mixture and propose two new guidance strategies: Zero-CFG, which disables CFG in the teacher and removes text conditioning in the fake score network, and Anti-CFG, which applies negative CFG in the fake score network. This flexible setup improves diversity without sacrificing alignment. Comprehensive experiments on SD1.5 and SDXL demonstrate state-of-the-art performance in both one-step and few-step generation settings, along with robustness to the absence of real images. Our efficient PyTorch implementation, along with the resulting one- and few-step distilled generators, will be released publicly as a separate branch at \href{https://github.com/mingyuanzhou/SiD-LSG}{\texttt{github.com/mingyuanzhou/SiD-LSG}}.

\end{abstract}

\section{Introduction}

Text-to-image (T2I) diffusion models have demonstrated remarkable capabilities in synthesizing photorealistic images \citep{nichol2022glide, ramesh2022hierarchical, saharia2022photorealistic, rombach2022high, podell2024sdxl}. 
They rely on two key components: an iterative refinement-based reverse diffusion process \citep{sohl2015deep,song2019generative, ho2020denoising}, and classifier-free guidance (CFG) \citep{ho2022classifier}, a technique that  enhances image fidelity and text-image alignment when applied with an appropriate guidance~scale. %
Contrasting the conditional score estimation against the unconditional scores,
CFG   %
amplifies the influence of the text %
conditions at each iterative-refinement step of the reverse sampling process \citep{saharia2022photorealistic, rombach2022high}. %
While the use of CFG improves text-image alignment, it often reduces generation diversity. It can also degrade fidelity when guidance scales are too high \citep{saharia2022photorealistic}.
Consequently, selecting an appropriate CFG scale is essential for balancing fidelity, diversity, and text-image alignment in generated outputs.

The reverse diffusion process requires multiple passes through a pretrained score estimation network \citep{song2019generative, ho2020denoising, dhariwal2021diffusion, song2021denoising}, resulting in a computationally expensive  sampling procedure. 
Prior work has aimed to accelerate sampling by reducing the number of refinement steps in the reverse process—from thousands to mere dozens—using advanced solvers for ODEs and SDEs, along with various approximation techniques
\citep{zheng2022truncated, liu2022pseudo, lu2022dpmsolver, karras2022elucidating, meng2023distillation}. However, these methods remain fundamentally constrained by the iterative nature of the reverse diffusion process.
In contrast, diffusion distillation methods—which aim to distill a one-step or few-step generator from a pretrained score estimation network—represent a distinct paradigm that has recently attracted considerable attention \citep{liu2022flow, luo2023latentlora, thuan2024swiftbrush, Sauer2023AdversarialDD, Xu2023UFOGenYF, yin2023onestepDW, yin2024improved, xie2024emdistillationonestepdiffusion, zhou2024score, zhou2024long, zhou2025adversarial}. This shift in how diffusion models are accelerated opens new avenues for efficient and scalable T2I generation.

A prominent class of diffusion distillation methods relies on matching the forward-diffused model and data distributions to align the generator with clean data distribution \citep{thuan2024swiftbrush, Sauer2023AdversarialDD, Xu2023UFOGenYF, yin2023onestepDW, yin2024improved, zhou2024score, zhou2024long, zhou2025adversarial}. These methods are built upon the same distribution-matching principle as Diffusion GAN \citep{wang2022diffusion}, which shows that adding noise helps align the supports of model and data distributions, thereby enabling meaningful gradients to guide the model toward the data. However, they differ in how the noisy data distribution is represented—using real noisy samples, a pretrained score estimation network, or both.

Despite recent progress, several key challenges remain in scaling distillation to high-resolution diffusion models, such as Stable Diffusion XL (SDXL) \citep{podell2024sdxl} at 1024×1024 resolution.
\begin{itemize}
    \item 
First, existing diffusion distillation methods—especially when targeting high-resolution models like SDXL—often depend on high-quality real text–image pairs, teacher-synthesized images via reverse diffusion, or both \citep{yin2024improved}. This dependence limits their applicability in data-free settings and constrains their effectiveness even when real data are available.

\item %

Second, in high-resolution T2I generation, a one-step generator may be insufficient, and multiple steps are often needed to reduce artifacts and refine outputs.

\item Third, current diffusion-distillation-based T2I frameworks continue to face trade-offs when selecting CFG scales.
Across both SoTA data-free methods \citep{thuan2024swiftbrush, zhou2024long} and those that require real data \citep{yin2023onestepDW, yin2024improved, dao2024swiftbrush}, a persistent challenge emerges: the model that achieves high CLIP scores~\citep{radford2021learning}, indicating strong text–image alignment, is often distinct from the one that achieves low Fréchet Inception Distance (FID) \citep{heusel2017gans}, indicating high generation diversity. This discrepancy arises because alignment- and fidelity-focused models typically require high CFG during distillation, whereas diversity-focused models benefit from lower CFG.
\end{itemize}
Existing methods \citep{liu2022flow, luo2023latentlora, thuan2024swiftbrush, Sauer2023AdversarialDD, Xu2023UFOGenYF, yin2023onestepDW, yin2024improved, xie2024emdistillationonestepdiffusion, zhou2024score, zhou2024long, zhou2025adversarial} have addressed aspects of these challenges, but none are able to solve all of them simultaneously. For example, DMD2 \cite{yin2024improved} achieves SoTA performance in 
distilling SDXL. However, for SDXL, it relies on teacher-synthesized noise–image pairs to train the one-step generator and real data to train the four-step generator. Without teacher-synthesized noise–image pairs, the one-step generator cannot be trained reliably; and in the absence of real data, the performance of the four-step generator drops noticeably. Furthermore, removing DMD2's distribution matching loss improves FID but significantly harms CLIP scores, while omitting its GAN loss yields comparable CLIP scores but degrades FID. Thus, although DMD2 provides a strong framework for high-resolution T2I diffusion distillation, it underperforms in data-free settings and exhibits a persistent trade-off between alignment and diversity—even when auxiliary data are available.

DMD2 builds upon its predecessor, DMD \citep{yin2023onestepDW}, which has been outperformed by Score identity Distillation (SiD) \citep{zhou2024score} in one-step distillation of unconditional and label-conditional diffusion models, and by SiD with Long and Short Guidance (LSG) \citep{zhou2024long} in one-step distillation of SD1.5 \citep{rombach2022high}, a T2I diffusion model at 512×512 resolution. These findings motivate us to optimize SiD for multistep generation while keeping the algorithm simple to implement and supported by theoretical grounding. Since SiD is designed to operate in a fully data-free setting, its theoretically grounded extension to few-step generation is well positioned to address the first two challenges.

To address the third challenge, we enhance SiD with either Zero-CFG or Anti-CFG—two more adaptive strategies for balancing alignment and fidelity when real data are available. Unlike prior methods, both disable CFG in the teacher by default; Zero-CFG removes text conditioning from the fake score network, while Anti-CFG applies negative CFG. These memory-efficient strategies enable effective navigation of the alignment–diversity trade-offs commonly observed  in T2I diffusion~models.

We conduct comprehensive experiments on distilling SD1.5 and SDXL into one- and few-step generators, presenting key findings in the main paper and deferring detailed ablation studies and full results to the Appendix. These results validate the theoretical foundations of our distillation algorithm—which aims to match the uniform mixture of outputs from all generation steps to the data distribution—and establish SoTA performance of the proposed multistep SiD framework, real data-based enhancement techniques, and the new guidance strategies Zero-CFG and Anti-CFG.

\section{Preliminaries and related works}

Denote \( t \) 
as a step in the forward diffusion process that corrupts the clean image  $\xv_0$ by sampling noisy image $\xv_t$ from 
$
q(\xv_t \given \xv_0) = \mathcal{N}(a_t \xv_0, \sigma_t^2 \mathbf{I}). 
$  
For a T2I diffusion model, %
we denote $\phi$ as the pretrained model parameter and define \( f_{\phi}(\xv_t,\cv) \) as the functional approximation of $\E[\xv_0\given  \xv_t,\cv]$, which is the conditional expectation of the real image $\xv_0$ given noisy image $\xv_t$ and text $\cv$. Similarly, we denote $\epsilonv_{\phi}(\xv_t,\cv)$ as the function predicting the noise within $\xv_t$, and adopt  $-\sigma_t^{-1}\epsilonv_{\phi}(\xv_t,\cv)$ as the functional approximation of the true data score $\nabla_{\xv_t} \ln p_\text{data}(\xv_t\given \cv) $.
A pretrained T2I diffusion model estimates the true data score as $\nabla_{\xv_t} \ln p_\text{data}(\xv_t\given \cv) \approx S_{\phi}(\xv_t, \cv)$, where the pretrained teacher score network $S_{\phi}(\xv_t, \cv)$ is related to both $\xv_0$ and $\epsilonv$ predictions  through the following relationship:
\ba{
   -\sigma_t S_{\phi}(\xv_t, \cv):= {\sigma_t^{-1}}({\xv_t-a_t f_{\phi}(\xv_t,\cv)}) =  \epsilonv_{\phi}(\xv_t,\cv). \notag
}
Note that $t$ as an input to the above functions is omitted for brevity. %

Given %
teacher $S_{\phi}(\xv_t, \cv)$, 
the goal of %
guided diffusion distillation is to distill a student model \( p_{\theta}(\xv_g \given \cv) \) from it to generate text-guided random samples: 
$ %
\xv_g = G_{\theta}(\zv, \cv), ~\zv \sim p(\zv), %
$ %
where \( G_{\theta} \) is a neural network parameterized by \(\theta\) that transforms noise \( \zv \sim p(\zv) \) into generated data \( \xv_g \) guided by the text condition \(\cv\). The generator \( G_{\theta} \) is typically initialized from  \( S_\phi \), but is significantly more efficient, requiring only a single or a few evaluations instead of the many iterative refinements performed by the teacher model. While the distribution of \( \xv_g \) is typically implicit, its diffused versions follow semi-implicit distributions \citep{yin2018semi,yu2023semiimplicit} that support various score-related identities \citep{vincent2011connection,zhou2024score}.

\textbf{Diffusion Distillation via Distribution Matching. }
To drive %
$p_{\theta}(\xv \given \cv)$ towards the data distribution $p_{\text{data}}(\xv \given \cv)$, a foundational approach—pioneered by Diffusion GAN \citep{wang2022diffusion}—is to match their noisy counterparts  at any time step~$t$ of the forward diffusion process. 
This ensures that the supports of the two distributions overlap and that the statistical divergence
$ %
\mathcal D(p_{\theta}(\xv_t \given \cv), p_{\text{data}}(\xv_t \given \cv))
$ %
is well-defined and yields meaningful gradients for model parameter $\theta$.
Here, $\mathcal D$ denotes a divergence measure, such as the Jensen--Shannon divergence \citep{goodfellow2014generative,wang2022diffusion,Xu2023UFOGenYF,Sauer2023AdversarialDD},  KL divergence, or Fisher divergence.

The KL divergence, expressed as $\mbox{KL}(p_{\theta}(\xv_t \given \cv)||p_{\text{data}}(\xv_t \given \cv)) $, has emerged as a prominent choice in diffusion distillation \citep{poole2023dreamfusion,wang2023prolificdreamer,luo2023diffinstruct,yin2023onestepDW,thuan2024swiftbrush,yin2024improved,dao2024swiftbrush}. However, it often requires additional support from real data–based adversarial losses or teacher-synthesized data–based regression losses to achieve satisfactory performance \citep{yin2023onestepDW,yin2024improved}. In contrast, the Fisher divergence, expressed as $\E_{\xv_t\sim p_{\theta}(\xv_t)}[
\|S_{\phi}(\xv_t,\cv) - \nabla_{\xv_t}\ln p_{\theta}(\xv_t\given \cv)\|_2^2]$,  offers a compelling alternative: it can rival the teacher in a fully data-free setting \citep{zhou2024score,zhou2024long,luo2024one} and, as demonstrated in unconditional and label-conditional tasks, can convincingly outperform the teacher when access to real images is available \citep{zhou2025adversarial}. Given the SoTA performance achieved by the SiD framework in one-step generation, %
this work focuses on extending the use of Fisher divergence to the few-step generation regime for T2I tasks. We present both a data-free multistep solution that excels in generation fidelity and text-image alignment, and a data-dependent solution that further enhances generation diversity when real data is available.

\textbf{Related Work in T2I Diffusion Distillation. }
Our work builds on SiD and its extensions \citep{zhou2024score, zhou2024long, zhou2025adversarial}, as well as prior advancements in accelerating diffusion models and CFG—particularly in the context of T2I generation.
Diffusion distillation methods can be broadly categorized into trajectory-based and score-based approaches \citep{zhang2025towards, fan2025survey}. Our method falls into the latter category—score-based distillation—with the option to incorporate a Diffusion GAN–based adversarial loss \citep{wang2022diffusion}.
Diffusion distillation methods can also be classified as either data-free or data-dependent, with the latter requiring access to real data, teacher-synthesized data, or both.
Our method can operate in a fully data-free setting, but can also utilize real data—even in limited quantity—via the Diffusion GAN loss to enhance generation diversity.
We defer a detailed review of related work in Appendix~\ref{sec:relatedwork}.

\section{Method}

Leveraging score-related identities within semi-implicit distributions \citep{yin2018semi,yu2023semiimplicit}, SiD trained in a data-free manner has emerged as a SoTA  approach for one-step diffusion distillation \citep{zhou2024score,zhou2024long}. 
In unconditional and label-conditional diffusion models, %
SiD
can closely match—or even surpass—the performance of the teacher model \citep{zhou2024score,zhou2025adversarial}. Moreover, when real data is available, %
SiD can leverage both the teacher and real data during distillation to convincingly outperform the teacher \citep{zhou2025adversarial}. Notably, these strong results are achieved using a single generation step and without the use of CFG.

While SiD performs well in unconditional and label-conditional generation without requiring multiple generation steps, CFG, or access to training images, these components are often considered essential for achieving strong and robust performance in high-resolution T2I tasks addressed in this work.  In this paper, we enhance SiD to effectively address the challenges of T2I generation by optimizing it for few-step generation and introducing two novel guidance strategies. Together, these extensions enable robust and strong performance at high resolutions.

\textbf{Fisher Divergence. } In one-step diffusion, we generate noisy fake images \( \xv_t \) given text $\cv$ via 
 \ba{ \xv_t = a_t \xv_g + \sigma_t \epsilonv_t, %
\quad\xv_g=G_{\theta}(\zv, \cv),\quad\zv, \epsilonv_t \sim \mathcal{N}(0, \mathbf{I}).\label{eq:xv_t_g}
}
 Assume the exact score of the diffused training data, given by \( S_{\phi^*}(\xv_t,\cv) = \nabla_{\xv_t} \ln p_\text{data}(\xv_t \given \cv) \), is approximated by a pretrained score network $S_{\phi}(\xv_t,\cv)$, trained using a denoising diffusion or score matching loss \citep{vincent2011connection,sohl2015deep,song2019generative,ho2020denoising}. Given a time step $t$, text condition \( \cv \), and the teacher \( S_{\phi}(\xv_t,\cv) \), the optimization objective of SiD is the model-based Fisher divergence:
\begin{align} %
\resizebox{.955\columnwidth}{!}{$
\!\gL_{\theta}  = 
\E_{p_{\theta}(\xv_t)}\left[
\|S_{\phi}(\xv_t,\cv) - \nabla_{\xv_t}\ln p_{\theta}(\xv_t\given \cv)\|_2^2\right]
=\E_{ p_{\theta}(\xv_t)}\left[
\|a_t\sigma_t^{-2} (f_{\phi}(\xv_t,\cv)-
f_{\psi^*(\theta)}(\xv_t,\cv)\|_2^2\right],$
}\!\!
\label{eq:MESM} 
\end{align}
where we assume access to the optimal fake score network for the generator $G_{\theta}$:
\ba{
f_{\psi^*(\theta)}(\xv_t,\cv)=\E[\xv_g\given \xv_t,\cv]=(\xv_t + \sigma_t^2 \nabla_{\xv_t} \ln p_\theta(\xv_t\given \cv))/a_t. \label{eq:psistar}
} 
The objective \eqref{eq:MESM} is minimized when the model distribution matches that of the teacher.  While intractable since \( \psi^*(\theta) \) is unknown, it admits an effective alternating optimization solution~\citep{zhou2024score}.

\textbf{Alternating Optimization with CFG Enhancement. }
In T2I diffusion models,
we denote \ba{
 f_{\phi,\kappa}(\xv_t,\cv) =  f_{\phi}(\xv_t,\emptyset) + \kappa[ f_{\phi}(\xv_t,\cv)-f_{\phi}(\xv_t,\emptyset)].\notag
 } 
 as the teacher  enhanced with CFG at scale $\kappa$.
Similarly, we %
denote $f_{\psi,\kappa}(\xv_t,\cv)$ as the fake score network under CFG. 
Ignoring time-dependent loss reweighting coefficients as of now, the Fisher divergence enhanced with CFG is optimized by alternating  between optimizing \( \psi \) and \( \theta \) as~follows: 
\ba{
\resizebox{.95\columnwidth}{!}{$
L_{\psi,t,\kappa_1} %
=\textstyle
\|f_{\psi,\kappa_1}(\xv_t,\cv)-\xv_g\|_2^2%
, ~~
L_{\theta,t,\kappa_{2:4}} %
=(f_{\phi,\kappa_4}(\xv_t,\cv)-f_{\psi,\kappa_2}(\xv_t,\cv))^T(f_{\psi,\kappa_3}(\xv_t,\cv)-\xv_g),
$} \! %
\label{eq:obj-theta_lsg}
}
where scales \( \kappa_{1:4} \) %
can be tuned to achieve different guidance strategies. Typical settings of 
$\kappa$ values corresponding to conventional and LSG guidance strategies, as well as two additional settings later introduced as Zero-CFG and Anti-CFG, are listed in Table\,\ref{tab:Hyperparameters} in Appendix~\ref{sec:detail}.

\subsection{Data-free multistep distillation} \label{sec:multistep}
For diffusion distillation, SiD defines a one-step generator as  
$
\xv_g = G_{\theta}(\sigma_{t_{\text{init}}} \zv, t_{\text{init}}, \cv), ~\zv \sim \mathcal{N}(0, \mathbf{I}),
$
where \( t_{\text{init}} = 625\) is the default  and 
 $\theta$ is initialized from the teacher  
 $\phi$.  
 To extend this to a multistep generator, we set \( \xv_g^{(0)} = 0 \), let \( \lfloor \cdot \rfloor \) denote the floor function, and for \( k = 1, \dots, K \), define  
\ba{
\xv_g^{(k)} = G_{\theta} \big( a_{t_k} \xv_g^{(k-1)} + \sigma_{t_k} \zv_k, \tau_k, \cv \big), \quad \label{eq:G_multistep}
\tau_k = \lfloor (1 - (k-1)/K) t_{\text{init}} \rfloor, \quad \zv_k \sim \mathcal{N}(0, \mathbf{I}).
}
Note that this formulation was previously used to perform multistep inference with a SiD-distilled one-step generator \citep{zhou2024long}. In contrast, our objective is to explicitly optimize \( G_{\theta} \) for multistep inference, such as generation with four steps.
A slightly modified variant—setting \( t_{\text{init}} = T \)—was employed in \citet{kohler2024imagine} and \citet{yin2024improved}, where it is referred to as “backward simulation.”

We consider two strategies to distill the multistep generator: \textit{Final-Step Matching}, which compares only the final output \( \xv_g^{(K)} \) against the data distribution, and \textit{Uniform-Step Matching}, which compares a uniform mixture over the outputs of all generation steps \( \{ \xv_g^{(k)} \}_{k=1}^K \) with the data distribution.

\textbf{Final-Step Matching. }  
This approach distills \( G_\theta \) by backpropagating through the entire generation chain defined in \eqref{eq:G_multistep}, where \( G_\theta \) is nested \( K \) times. This strategy has proven effective in learning diffusion-based multistep policies~\citep{wang2022diffusionrl}. In the context of diffusion distillation, it requires no changes to the SiD algorithm aside from redefining \( \xv_g \) as the nested composition of \( G_\theta \) over \( K \) steps. However, in practice, the chain is typically kept short ($e.g.$, \( K = 5 \) in \citep{wang2022diffusionrl}), as optimization becomes increasingly challenging with larger \( K \). Our experiments show that while increasing \( K \) can improve FID scores for final-step matching, it also slows convergence and tends to degrade the CLIP scores. %

\textbf{Uniform-Step Matching. }  
This approach avoids backpropagating through the entire generation chain. During training, it randomly samples a step \( k \in \{1, \ldots, K\} \) and generates its output as  
\begin{align}
\xv_g^{(k)} = G_{\theta} \big( a_{t_k} \, \text{sg}(\xv_g^{(k-1)}) + \sigma_{t_k} \zv_k, \, \tau_k, \, \cv \big), \label{eq:G_multistep1}
\end{align}  
where \(\text{sg}(\cdot)\) denotes stop-gradient that prevents gradients from propagating into earlier steps.
Ideally, each output—including both final and intermediate steps—should align with the data distribution. From the perspective of dense rewards \citep{yang2024a}, this enables dense supervision across all  steps.

To optimize \( G_\theta \) at step $k$, we forward diffuse 
\( \xv_g^{(k)} \) to a random  timestep $t\in\{1,\ldots,T\}$  as
\ba{
\xv_t^{(k)} = a_t \xv_g^{(k)} + \sigma_t \zv_k,\quad \zv_k \sim \mathcal{N}(0, \mathbf{I}). \label{eq:xtk}
} 
The dense supervision at step $k$ can then be provided by  the corresponding Fisher divergence~as  
\begin{align}
\gL_{\theta}^{(k)} = 
\E_{\xv_t^{(k)}\sim p_{\theta}(\xv_t^{(k)})} \left[
\|S_{\phi}(\xv_t^{(k)}, \cv) - \nabla_{\xv_t^{(k)}} \ln p_{\theta}(\xv_t^{(k)} \given \cv)\|_2^2
\right].
\label{eq:MESM1}
\end{align}

\begin{lemma}\label{lemma:dropk}
Assuming the pretrained teacher score estimation network \( S_\phi \) reaches its theoretical optimum, i.e., \( S_{\phi^*}(\xv_t,\cv) = \nabla_{\xv_t} \ln p_\emph{\text{data}}(\xv_t \given \cv) \), the optimal distributions over \( \xv_g^{(k)} \) are identical for all \( k \) and equal to \( p_\emph{\text{data}}(\xv_g^{(k)} \given \cv) \).

\end{lemma}

To simplify the \( k \)-specific learning objective in \eqref{eq:MESM1}, we state Lemma~\ref{lemma:dropk} and defer its proof to Appendix~\ref{sec: diffgan}.
The result of this lemma shows that the step index \( k \) can be dropped, justifying the use of a single, shared fake score network as defined in~\eqref{eq:psistar}. This network serves as the optimal score estimator for noise images \( \xv_t \) generated from a uniform mixture over outputs from all \( K \) steps:
\begin{align}
\xv_t = a_t \xv_g + \sigma_t \boldsymbol{\epsilon}_t, \quad \xv_g \sim p_{\theta}(\xv_g \mid \cv) = \textstyle\frac{1}{K} \sum_{k=1}^K p_{\theta}(\xv_g^{(k)} \mid \cv), \label{eq:xt2}
\end{align}
where \( \xv_g^{(k)} \) is generated as in \eqref{eq:G_multistep1}. In other words, rather than optimizing a step-specific Fisher divergence with a separate fake score network for each step, one can instead match the uniform mixture of outputs from all \( K \) steps to the data distribution using a single, shared fake score~network.

With this simplification, the Fisher divergence objective in~\eqref{eq:MESM1} %
can be optimized by sampling \( \xv_g \) and \( \xv_t \) from a randomly selected generation step using~\eqref{eq:xt2}, and alternating between optimizing \( \psi \) and \( \theta \) via~\eqref{eq:obj-theta_lsg}. While all SiD variants adopt \eqref{eq:obj-theta_lsg} for alternating optimization, they differ in how generator outputs are utilized: %
final-step matching use only the final output from the $K^{th}$ step, whereas uniform-step matching leverages a uniform mixture of outputs from all \( K \) generation steps.

\subsection{Enhancements with real images, guidance strategies, and coding optimizations} \label{sec:3.2}%

\textbf{Data-Dependent Enhancement. } 
While SiD supports data-free distillation, prior work has enhanced it by incorporating Diffusion GAN \citep{wang2022diffusion}, leveraging real data to compensate for the gap between the true scores and those estimated by the teacher score network \citep{zhou2025adversarial}.
In the context of  label-conditional diffusion models, their results show that access to real data further improves SiD, enabling it to convincingly outperform the teacher in a single generation step, without using~CFG.

Motivated by these findings, we explore combining SiD with CFG, multistep generation, and Diffusion GAN for T2I generation. 
We find that incorporating real text-image pairs  primarily improves SiD's generation diversity, as evidenced by significantly reduced FID scores, while also maintaining or slightly improving text–image alignment, as reflected by modest gains in CLIP scores.
We defer the details of how the Diffusion GAN-based adversarial loss is integrated with SiD, CFG, and multistep generation to Appendix~\ref{app:code},
 but emphasize here that incorporating Diffusion GAN introduces no additional model parameters. Specifically, for the fake score network $f_{\psi}$, we denote the part reused as the Diffusion GAN discriminator by~$D$. For both SD1.5 and SDXL, $f_{\psi}$ is a U-Net, and $D$ extracts a 2D discriminator map from the intermediate latent representation after the mid-block and before the upsampling block by pooling along the channel dimension.

 Similar to SiD trained step-wise with gradient blocking—where no step-specific modifications to the loss are required—the Diffusion GAN loss can be applied to the uniform mixture of all \( K \) generation steps' outputs %
 to compensate for the discrepancy between the true and teacher score networks.

\textbf{Zero-CFG and Anti-CFG. }
Previous works in T2I generation typically apply CFG only to the pretrained score network \( f_{\phi} \), as seen in DMD \citep{yin2023onestepDW}. This corresponds to setting \( \kappa_1 = \kappa_2 = \kappa_3 = 1 \), \( \kappa_4 > 1 \) in \eqref{eq:obj-theta_lsg}, and tuning the guidance scale \( \kappa_4 \) to balance CLIP and FID scores. SiD-LSG~\citep{zhou2024long} extends this approach by applying the same CFG scale to all components, $i.e.$, \( \kappa_1 = \kappa_2 = \kappa_3 = \kappa_4 > 1 \).
Despite achieving SoTA FID in a data-free, single-step generation setting, SiD-LSG~\citep{zhou2024long} remains subject to the same trade-offs observed in diffusion models and their distilled variants: stronger guidance improves text–image alignment and visual quality, while weaker guidance promotes generation diversity. %
In particular, the one-step generation of SiD under two representative LSG settings—\( \kappa_i = 1.5 \) and \( \kappa_i = 4.5 \) for all \( i = 1, 2, 3, 4 \)—is optimized for low FID and high CLIP, respectively, but fails to achieve strong performance on both metrics simultaneously.

To overcome this limitation, we introduce both Zero-CFG and Anti-CFG as two novel guidance strategies. By default, both apply the standard text condition—without CFG—when evaluating the teacher~\( f_{\phi} \). For the fake score network \( f_{\psi} \), Zero-CFG removes text conditioning entirely by setting \( \cv = \emptyset \), corresponding to \( \kappa_1 = \kappa_2 = \kappa_3 = 0 \), \( \kappa_4 = 1 \), while Anti-CFG applies a negative CFG scale, corresponding to \( \kappa_1 = \kappa_2 = \kappa_3 = -1 \), \( \kappa_4 = 1 \).
 In both cases, \( f_{\psi} \) operates with weakened or inverted guidance, while \( f_{\phi} \) retains standard conditioning without amplification. The Diffusion GAN component continues to use the standard text condition throughout. It is worth noting that Zero-CFG is distinct from the ``No CFG'' configuration, where \( \kappa_1 = \kappa_2 = \kappa_3 = \kappa_4 = 1 \). When memory is limited, both Zero-CFG and Anti-CFG are preferable to the LSG setting, which applies CFG to both the teacher and fake score networks. In particular, Zero-CFG—which contrasts a text-conditioned teacher with an unconditional fake score network without CFG—is especially appealing, as it offers the lowest memory footprint and per-iteration computation cost among all guidance strategies.

Under Zero-CFG or Anti-CFG, the SiD loss prioritizes alignment and fidelity by removing or inverting the fake score network’s access to text conditions, thereby forcing the generator to adapt. In parallel, the integrated Diffusion GAN loss enhances generation diversity and fidelity. Experiments show that, when combined with this adversarial loss, both strategies effectively reinforce text–image alignment and rank among the top-performing methods for distilling SD1.5 and SDXL.

\textbf{Implementation optimizations. } To improve distillation efficiency and scalability, we incorporate several %
coding-level optimizations into SiD. These include %
Automatic Mixed Precision (AMP) for optimized memory usage and speed, %
Fully Sharded Data Parallel (FSDP) for model scaling under memory constraints, and tailored multistep generation strategies. Together, these enhancements enable stable training of high-performing one- and few-step diffusion models while minimizing dependencies on non-native PyTorch libraries. Full implementation details are provided in Appendix~\ref{sec:implementation}.

\section{Experiments} \label{sec:experiment}

We focus on distilling SD, the leading open-source platform for T2I diffusion models that operate in the latent space of an image encoder-decoder \citep{rombach2022high}. Specifically, we target SD1.5, which has about 0.9B parameters and requires at least 1.60 GB in FP16, and SDXL, which consists of about 3B parameters and occupies at least 4.78 GB in FP16. We present the details of our method in Algorithm~\ref{alg:sid}.
We summarize the parameter settings and compare memory and computation demands in Tables\,\ref{tab:Hyperparameters} and \ref{tab:Hyperparameters_4step} in Appendix~\ref{sec:detail}. Unless otherwise specified, %
we uniformly apply these settings.

When distilling our one-step or few-step generators in a data-free setting, we utilize the Aesthetics6+ prompt \citep{cherti2023reproducible} for training. To apply data-dependent enhancement, %
we use the 480k text-image pairs %
provided by DMD2~\citep{yin2024improved}. We evaluate generation diversity by computing zero-shot FID on the COCO-2014 validation set. %
We evaluate generation fidelity and text-image alignment by employing the \verb|ViT-g-14-laion2b_s12b_b42k| encoder \citep{ilharco_gabriel_2021_5143773,cherti2023reproducible} to compute the CLIP score. %
The FID and CLIP scores shown in the figures are computed using 30k randomly sampled prompts from the COCO-2014 validation set for SD1.5 and 10k for SDXL.  For the SD1.5 results reported in Table\,\ref{tab:comparison}, we use the evaluation code %
provided by GigaGAN~\citep{kang2023scaling}, which involves  
a pre-defined list of 30k text prompts selected from the COCO-2014 validation set.
For SDXL, we follow DMD2 and use a fixed set of 10k COCO-2014 validation text prompts  to %
compute the corresponding metrics, and further report FID$_{\text{{Patch512}}}$ as FID computed on center-cropped 512x512 images. 

Below we use \textbf{SiD} to denote the data-free version; \textbf{SiD$^a$} to denote the data-enhanced version, where ``$a$'' indicates the use of an adversarial loss; and \textbf{SiD$_2^a$} to denote the two-stage variant that first distills with SiD and then uses its generator to initialize SiD$^a$. We avoid using “SiDA” or “SiD$^2$A” to prevent confusion with prior work~\citep{zhou2025adversarial}, which assumes access to the full training dataset and is optimized without CFG or multistep generation. In contrast, our setting incorporates CFG, supports multistep generation, and uses only limited subsets of real data.

\subsection{Ablation studies on multistep generation, real data usage, and guidance strategies}

To optimize SiD for high-resolution T2I tasks, we introduce three key enhancements: multistep generation, data-dependent optimization, and two new guidance strategies—Zero-CFG and Anti-CFG. We conduct detailed ablation studies using SD1.5 as the base model to evaluate the impact of each component.
We summarize the main findings below and defer the full results to Appendix~\ref{sec:SD1.5}.

We begin with one-step generation to evaluate the effects of data enhancement and the new Zero-CFG guidance strategy. Results are presented in Fig.\,\ref{fig:cfgfree} and Table\,\ref{tab:comparison} in Appendix \ref{sec:SD1.5}. Key observations include:
1)~Our AMP-based reimplementation of SiD matches FP32 in performance while running significantly faster and offering greater stability than FP16;
 2) SiD$^a$ significantly outperforms SiD, its data-free counterpart; and 3) Zero-CFG enhanced with real data 
provides a competitive and memory-efficient alternative to conventional CFG strategies.
Regarding point 2) specifically: under LSG with guidance scale 4.5, SiD achieves 16.59 FID and 0.317 CLIP, while SiD$^a$ achieves \textbf{10.10} FID and \textbf{0.321} CLIP—demonstrating that real data can effectively mitigate the diversity limitations of data-free SiD under high guidance, without sacrificing alignment or fidelity. At scale 1.5, SiD reaches 8.15 FID and 0.304 CLIP, while SiD$^a$ improves to \textbf{7.89} FID with the same CLIP. Notably, this FID of \textbf{7.89} is the lowest to date among all one-step generators distilled from SD1.5 that achieve a CLIP score exceeding 0.3.

We next study multistep SiD$^a$ using either final-step matching, which backpropagates through the entire generation chain, or uniform-step matching, which blocks gradients to isolate updates per step. Results in Fig.\,\ref{fig:cfgfree_multistep} and Table\,\ref{tab:comparison_cfgfree} in Appendix \ref{sec:SD1.5} %
show that final-step matching improves FID but lowers CLIP, suggesting increased diversity at the cost of alignment. Uniform-step matching maintains high CLIP while also benefiting FID as the number of steps increases. 
Although quantitative improvements in FID and CLIP are not consistently observed, visual inspection reveals noticeable gains in image quality—particularly in fine details—from one to four steps, with diminishing returns~thereafter.

\begin{figure}[t]%
  \centering
  \vspace{-3.5mm}
  \includegraphics[width=0.48\textwidth]{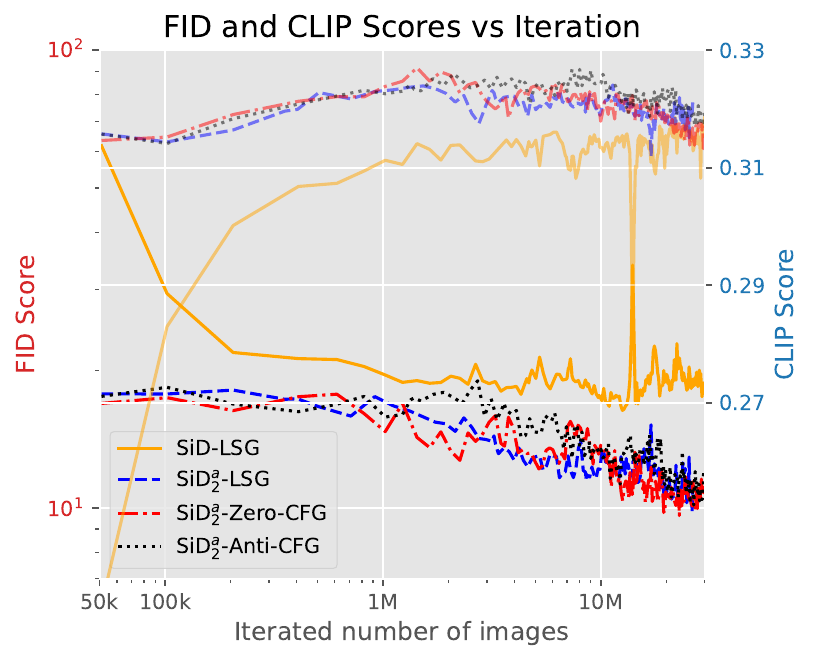}
  \includegraphics[width=0.48\textwidth]{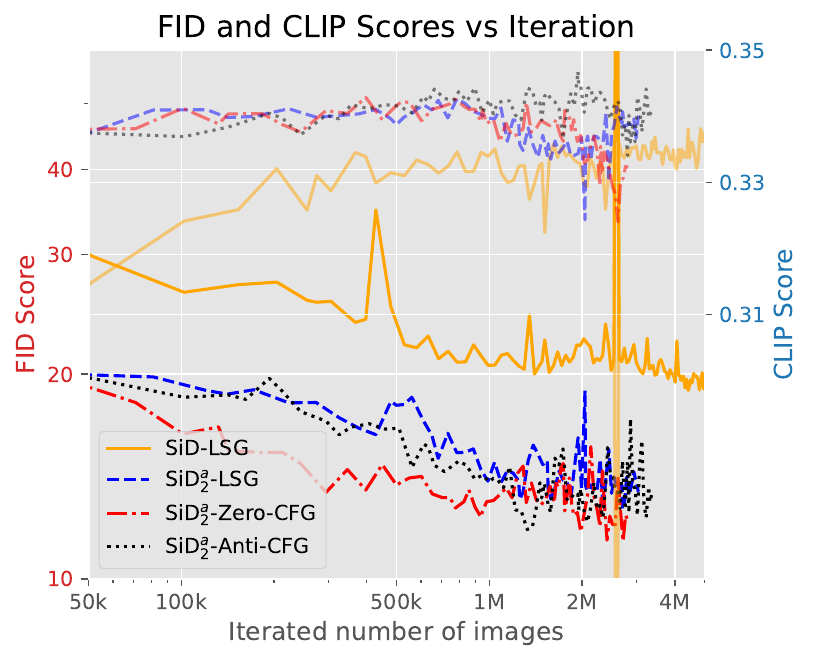}
 \caption{\small \textbf{Left: SD1.5, Right: SDXL}. Comparison of four-step SiD-LSG (data-free) and SiD$_2^a$ models trained with three different guidance strategies. All models use Uniform-Step Matching with four generation steps. SiD-LSG shows no clear conflict between decreasing FID and increasing CLIP. SiD$_2^a$ models, initialized from SiD-LSG and enhanced with real data, exhibit continued FID reduction during training, with CLIP scores peaking early and gradually declining.}
\label{fig:cfgfree_multistep1}
  \vspace{-2mm}
\end{figure}

To make better use of 
limited real training
images, %
we propose a two-stage training strategy.
First, we train a four-step SiD generator in a data-free setting using LSG with a guidance scale of 4.5, which yields strong CLIP scores that remain stable throughout training.
 This generator then serves as initialization  to train a four-step   generator, referred to as SiD$_2^a$, with real data and different guidance strategies (LSG, Zero-CFG, and Anti-CFG). For each setting, we periodically evaluate CLIP and FID, saving one early-stopping checkpoint when CLIP begins to decline (typically before reaching 10M fake images) and another with the lowest FID observed before iterating over 30M fake images. Results are presented in  Fig.\,\ref{fig:cfgfree_multistep1} and in Table\,\ref{tab:comparison_sid2a} (Appendix \ref{sec:SD1.5}), showing that using data-free SiD optimized for four-step generation as a starting point, all three guidance strategies combined with data enhancement lead to improved FID, while maintaining or moderately improving CLIP. The degree of improvement depends on the trade-off preference: favoring higher CLIP with moderately lower FID, where Anti-CFG with early stopping excels with \textbf{0.325} CLIP and 13.35 FID, or comparable CLIP with substantially lower FID, where Zero-CFG performs best with  0.318 CLIP and \textbf{9.86} FID.

\begin{table}[!t]
\centering
\small
\caption{\small  Comparison of SDXL and its accelerated one-step versions on 10k COCO-2014 prompts, following the same evaluation protocol of DMD2 \citep{yin2024improved}. Inference times are estimated using an NVIDIA A100 GPU as reference. Results are sourced from the respective scientific papers, or, when not directly available under the same settings, from \citet{wang2025rectified} and our own computations using the provided model checkpoints. %
}
\label{tab:sdxl_comparison}
\resizebox{.8\textwidth}{!}{
\begin{tabular}{@{}lccccccc@{}}
\toprule
\textbf{Method} & \textbf{Data-Free} & \textbf{Res.} & \textbf{Time ($\downarrow$)} & \textbf{\# NFEs} & \textbf{\# Param.} & \textbf{FID ($\downarrow$)} & \textbf{CLIP ($\uparrow$)}\\ \midrule
    SDXL (CFG=6)~\citep{podell2024sdxl} & N/A & 1024 &  3.5s& 100 & 3B &19.36 & 0.332 \\
     ADD (SDXL-Turbo)~\citep{Sauer2023AdversarialDD} &No &512 & 0.15s & 1 & 3B & 24.57& 0.337\\
    SDXL Lightning  ~\citep{lin2024sdxl-lightning} &No & 1024 & 0.35s & 1 & 3B & 23.92   & 0.316\\
    DMD2 ~\citep{yin2024improved}&No & 1024 & 0.35s & 1 & 3B & 19.01   & 0.336 \\
    Hyper-SDXL \citep{ren2024hypersd}&No&1024&0.35s&1&3B&32.35&0.334\\
    \midrule
     SiD (LSG $\kappa=4.5$) &\textbf{Yes} &1024& 0.35s & 1 & 3B & 21.30	& 0.336\\

    \midrule
    SiD$^a$  (LSG $\kappa=4.5$)&No  &1024& 0.35s & 1 & 3B& \textbf{15.13}	& 0.337\\
        SiD$^a$ (Zero-CFG) &No &1024& 0.35s & 1 & 3B & 15.60 & \textbf{0.341}\\
 \bottomrule
\end{tabular}}
\vspace{-1mm}
\end{table}

\begin{table}[!t]
\centering
\small
\caption{\small Analogous to Table\,\ref{tab:sdxl_comparison}, this table compares different methods for four-step generation on SDXL. It evaluates SiD$_2^a$ under various guidance strategies and examines the effect of early stopping at the point where the CLIP score begins to decline. FID$_{\text{Patch512}}$ denotes FID computed on center-cropped 512×512 images.}
\label{tab:sdxl_comparison1}
\resizebox{.94\textwidth}{!}{
\begin{tabular}{@{}lcccccccc@{}}
\toprule
\textbf{Method}&\textbf{Data-Free} & \textbf{Res.} & \textbf{Time ($\downarrow$)} & \textbf{\# NFEs} & \textbf{\# Param.} & \textbf{FID ($\downarrow$)}& \textbf{FID$_{\text{Patch512}}$ ($\downarrow$)} & \textbf{CLIP ($\uparrow$)}\\ \midrule
     ADD (SDXL-Turbo)~\citep{Sauer2023AdversarialDD}&No &512 & 0.34s & 4 & 3B & 23.19& & 0.334\\
    LCM-LoRA (4 step)~\citep{luo2023latentlora}&No & 1024 & 0.71s & 4 & 3B & 22.16  &  &0.317\\
     SDXL Lightning (4 step) ~\citep{lin2024sdxl-lightning}&No & 1024 & 0.71s & 4 & 3B & 24.56  &   &  0.323\\
    
    DMD2 (w/o GAN) ~\citep{yin2024improved}&\textbf{Yes} & 1024 & 0.71s  & 4 & 3B & 26.90  && 0.328 \\
    DMD2 (w/o Dist. Matching) ~\citep{yin2024improved}&No & 1024 & 0.71s  & 4 & 3B & {13.77} & & 0.307 \\
    DMD2 ~\citep{yin2024improved}&No & 1024 & 0.71s  & 4 & 3B & 19.32 &23.35  & 0.332 \\
    Hyper-SDXL \citep{ren2024hypersd}&No&1024&0.71s&4&3B&18.99&23.99&0.341\\
    PeRFlow-XL \citep{yan2024perflow}&No &1024&0.71s& 4&3B&20.99  &&0.334\\
    {Rectified Diffusion-XL (Phased)~\citep{wang2025rectified}} &No & 1024 & 0.71s & 4 & 3B & 19.71   &23.22 &  0.340\\
    \midrule
SiD (LSG $\kappa=4.5$) &\textbf{Yes} & 1024& 0.71s & 4 & 3B & 21.04	 &26.42&0.340\\

     \midrule
         SiD$_2^a$ (LSG $\kappa=4.5$, early-stop)&No & 1024& 0.71s & 4 & 3B &  17.19  &19.21&0.342\\
        SiD$_2^a$ (Zero-CFG, early-stop) &No & 1024& 0.71s & 4 & 3B & 15.92   &20.43 &0.342\\
        SiD$_2^a$ (Anti-CFG, early-stop) &No & 1024& 0.71s & 4 & 3B & 15.82   &20.18& \textbf{0.344}\\
        \midrule
        SiD$_2^a$ (LSG $\kappa=4.5$)  &No & 1024& 0.71s & 4 & 3B &  14.49 & 20.01&0.339\\
         SiD$_2^a$ (Zero-CFG)&No &1024& 0.71s & 4 & 3B &  \textbf{13.25}  &18.25&0.335\\
        SiD$_2^a$ (Anti-CFG) &No &1024& 0.71s & 4 & 3B &  {13.79} &\textbf{17.59} &{0.342}\\
 \bottomrule
\end{tabular}}
\vspace{-1mm}
\end{table}

\subsection{Distilling SDXL with SiD and its enhanced versions}

SDXL has approximately three times as many parameters as SD1.5. and requires approximately 10 times more computation per iteration on eight H100 80GB GPUs under our implementation, which utilizes both AMP and FSDP. Following the observations from the ablation studies on SD1.5, we first compare SiD and SiD$^a$ with one-step generation against prior one-step generation methods on SDXL. The results in Table\,\ref{tab:sdxl_comparison} show that the data-free SiD already performs on par with the best existing one-step generators in terms of CLIP, and is only outperformed by the teacher and DMD2—a SoTA distillation method that requires pretraining its one-step generator with a regression loss using 10K teacher-synthesized noise–image pairs to achieve robust performance. By contrast, one-step SiD trained in a data-free setting exhibits no robustness issues and requires neither real nor teacher-synthesized data to perform reliably.
 When provided with 480K real text–image pairs, SiD$^a$ significantly improves performance under either the conventional LSG strategy or the newly proposed Zero-CFG strategy. Notably, SiD$^a$ under Zero-CFG achieves the best CLIP score (0.341) and the second-best FID (15.60), just behind SiD$^a$ with LSG. Additionally, it requires less memory and runs faster, as no CFG is applied.

The results in Table\,\ref{tab:sdxl_comparison1} show that the data-free SiD, optimized for four-step generation, already performs on par with Rectified Diffusion-XL \citep{wang2025rectified}—the SoTA four-step generator in terms of CLIP—while delivering competitive FID. When provided with 480K real text–image pairs, SiD$^a$ further improves FID across all three tested guidance strategies. However, these FID gains often come at the cost of reduced CLIP scores for both LSG and Zero-CFG. 

As shown in Table\,\ref{tab:sdxl_comparison1}, applying early stopping around the point where CLIP begins to decline—typically within 1M fake images for SDXL—helps maintain modest CLIP improvements while still achieving FID scores significantly better than the prior SoTA method DMD2. Notably, SiD$_2^a$ under Anti-CFG with early stopping achieves the best CLIP score (\textbf{0.344}) and a highly competitive FID (15.82) among all compared methods, while SiD$_2^a$ under Zero-CFG attains the best FID score (\textbf{13.25}) and a competitive CLIP score (0.335). Overall, the Anti-CFG version of multistep SiD achieves a well-balanced trade-off between reducing FID and improving CLIP on distilling SDXL into a four-step generator.

These results suggest that SiD, even in a purely data-free setting, can already deliver highly competitive performance for both one- and few-step generation. This finding aligns with earlier observations that SiD—requiring only access to the pretrained teacher diffusion model—achieves SoTA or near-SoTA results when distilling models such as EDM, EDM2, SD1.5, and SD2.1-base. With just 480K real images—significantly fewer than the full training set—SiD can be further enhanced to achieve substantially lower FID while maintaining high CLIP scores.

\subsection{Qualitative evaluations, limitations, and future work}\label{sec:limitations}

For the SiD variants shown in Table\,\ref{tab:sdxl_comparison1}, we provide visualizations in Fig.\,\ref{fig:qualitative} for SiD$_2^a$ (LSG $\kappa=4.5$, early-stop), with results for additional settings included in the Appendix.
 Across these cases, SiD enhanced with multistep generation, real images, and new guidance strategies consistently demonstrates strong visual fidelity, text–image alignment, and generation diversity.

An observed limitation, as shown in Fig.\,\ref{fig:cfgfree_multistep} (Appendix~\ref{sec:SD1.5}) and Fig.\,\ref{fig:cfgfree_multistep1}, is that both SiD$^a$ and SiD$_2^a$ typically continue to improve in FID during training, while their CLIP scores peak early and gradually decline. In contrast, the data-free SiD-LSG shows no clear conflict between decreasing FID and increasing CLIP. This mismatch may be due to the limited size of the real dataset (480k), which becomes significantly oversampled by the time SiD$_2^a$ distillation reaches 2M fake images for SD1.5 and 1M for SDXL—approximately when CLIP scores tend to peak. Access to the full training dataset, which contains billions of images and is currently beyond our computational and storage capacity, could help address or alleviate this limitation.

While our method achieves SoTA FID and CLIP scores—indicating strong text–image alignment and effective recovery of the training data distribution—these metrics do not always reflect human-perceived quality, as distillation focuses on distributional fidelity rather than subjective preference. Bridging this gap often requires reinforcement learning (RL)–based preference optimization.

In such scenarios—whether RL is applied post-distillation or integrated during training—preserving generative diversity is critical for exploring the reward landscape effectively. A lack of diverse outputs can hinder optimization and limit the model’s ability to generate preferred content. Since our method retains the teacher model’s generative diversity—and can even surpass it with access to limited real training images—it provides a strong foundation for preference-guided fine-tuning, whether applied post-distillation or jointly during training. We leave this direction for future work.

\section{Conclusion} \label{sec:conclusion}

In this work, we advance diffusion distillation by enhancing the Score identity Distillation (SiD) framework with multistep generation, data-dependent optimization, and novel guidance strategies. Our approach tackles key challenges in high-resolution text-to-image generation—most notably, the reliance on real or teacher-synthesized data and the trade-off between alignment and diversity induced by classifier-free guidance (CFG). We propose two new guidance strategies, Zero-CFG and Anti-CFG, which remove or invert text conditioning in the fake score network while disabling CFG in the teacher. When paired with a Diffusion GAN–based adversarial loss, these strategies improve generation diversity without compromising alignment or fidelity.
Critically, our multistep distillation algorithm is both theoretically grounded and practically simple: it aligns the model distribution with the data distribution by matching a uniform mixture of outputs from all generation steps—a formulation that avoids step-specific networks and is easily integrated into existing training pipelines. Extensive experiments on SD1.5 and SDXL demonstrate that our distilled models achieve state-of-the-art performance in both one- and few-step generation, surpassing prior methods in FID and CLIP, while remaining robust with or without access to real images and maintaining efficiency and scalability. Our native PyTorch implementation—using DDP or FSDP with AMP—along with the distilled one- and few-step generators, will be released publicly.

\small
\bibliographystyle{plainnat}
\bibliography{%
zhougroup_ref}
\normalsize

\newpage
\appendix
\onecolumn

\begin{center}
    \Large\textbf{%
    Few-Step Diffusion via Score identity Distillation:
    Appendix}
\end{center}

\section{Additional Qualitative Examples}

\begin{figure}[h]
\includegraphics[width=\linewidth]{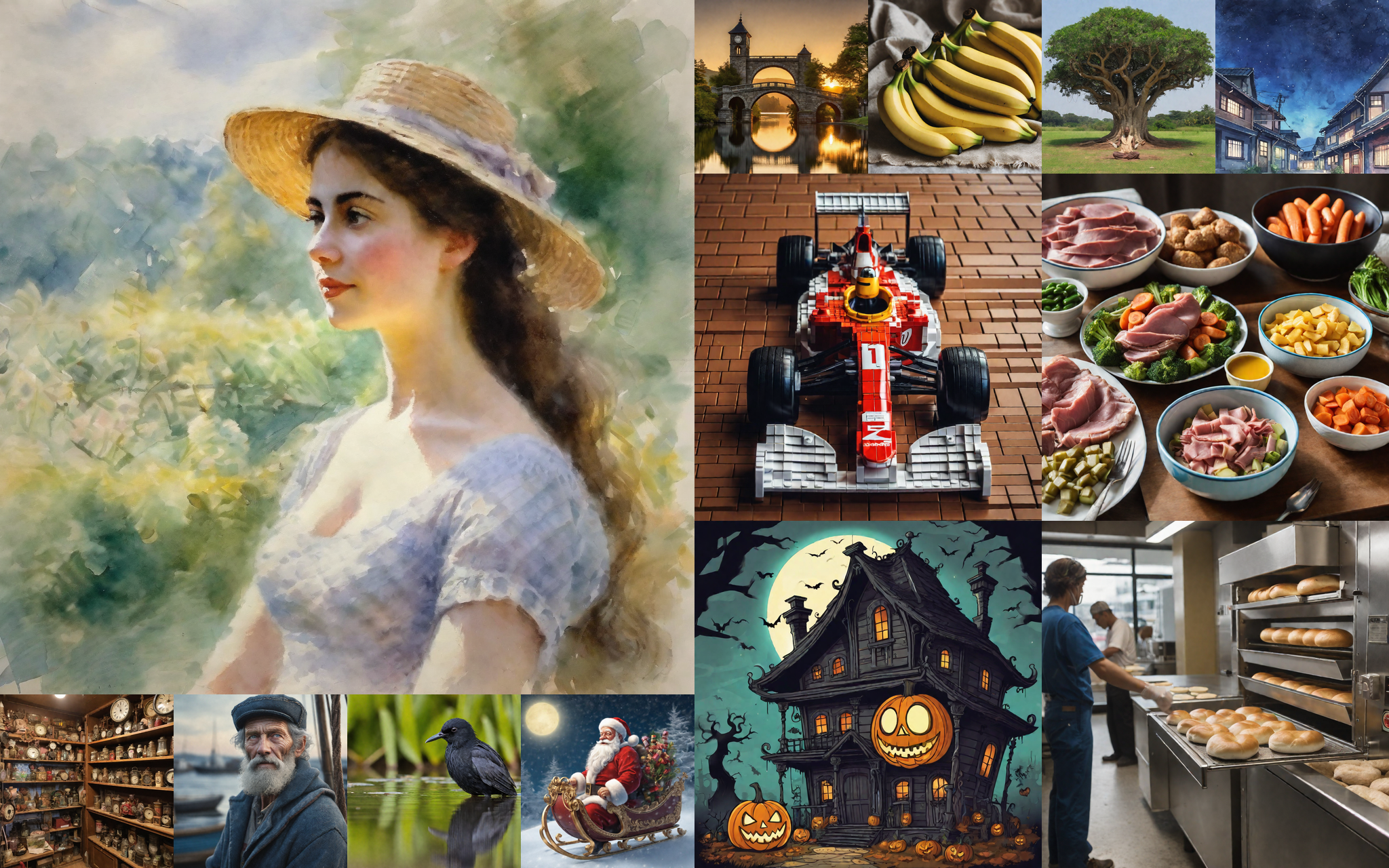}
\caption{Example four-step generations at 1024$\times$1024 resolution using our SiD-based multistep distillation method (SiD with LSG). Note that SiD is a data-free distillation method that does not require access to real images.}
    \label{fig:more_sid_examples}
\end{figure}

\begin{figure}[h]
    \centering
    \includegraphics[width=1\linewidth]{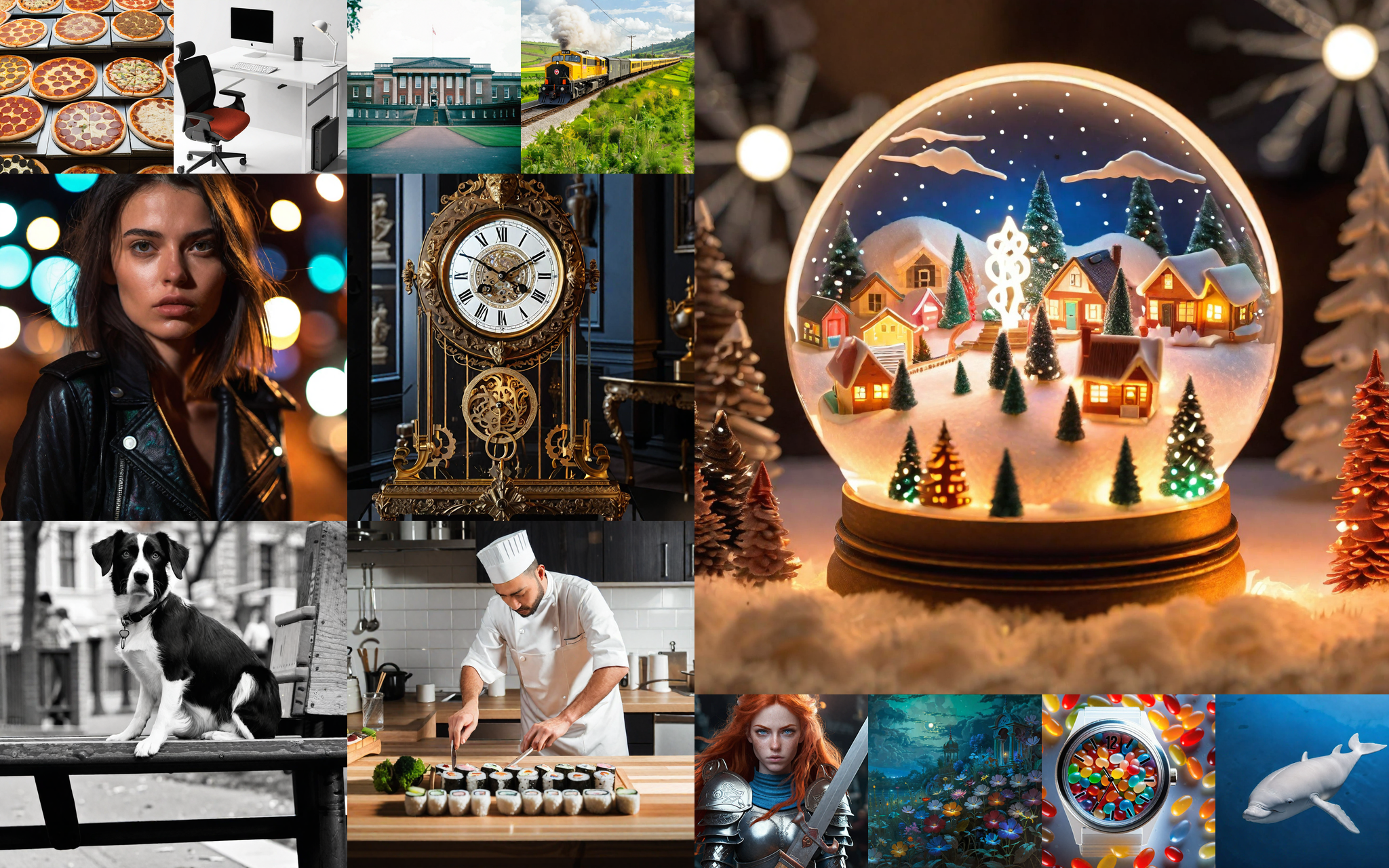}
\caption{Example four-step generations at 1024$\times$1024 resolution using our SiD-based multistep distillation method (SiD$_2^{\alpha}$ with Zero-CFG). Note that SiD$_2^{\alpha}$ initializes its generator from SiD, which is data-free, and continues training with access to a limited number of real text-image pairs.}

    \label{fig:qualitative_zero}
\end{figure}
\begin{figure}
    \centering
    \includegraphics[width=1\linewidth]{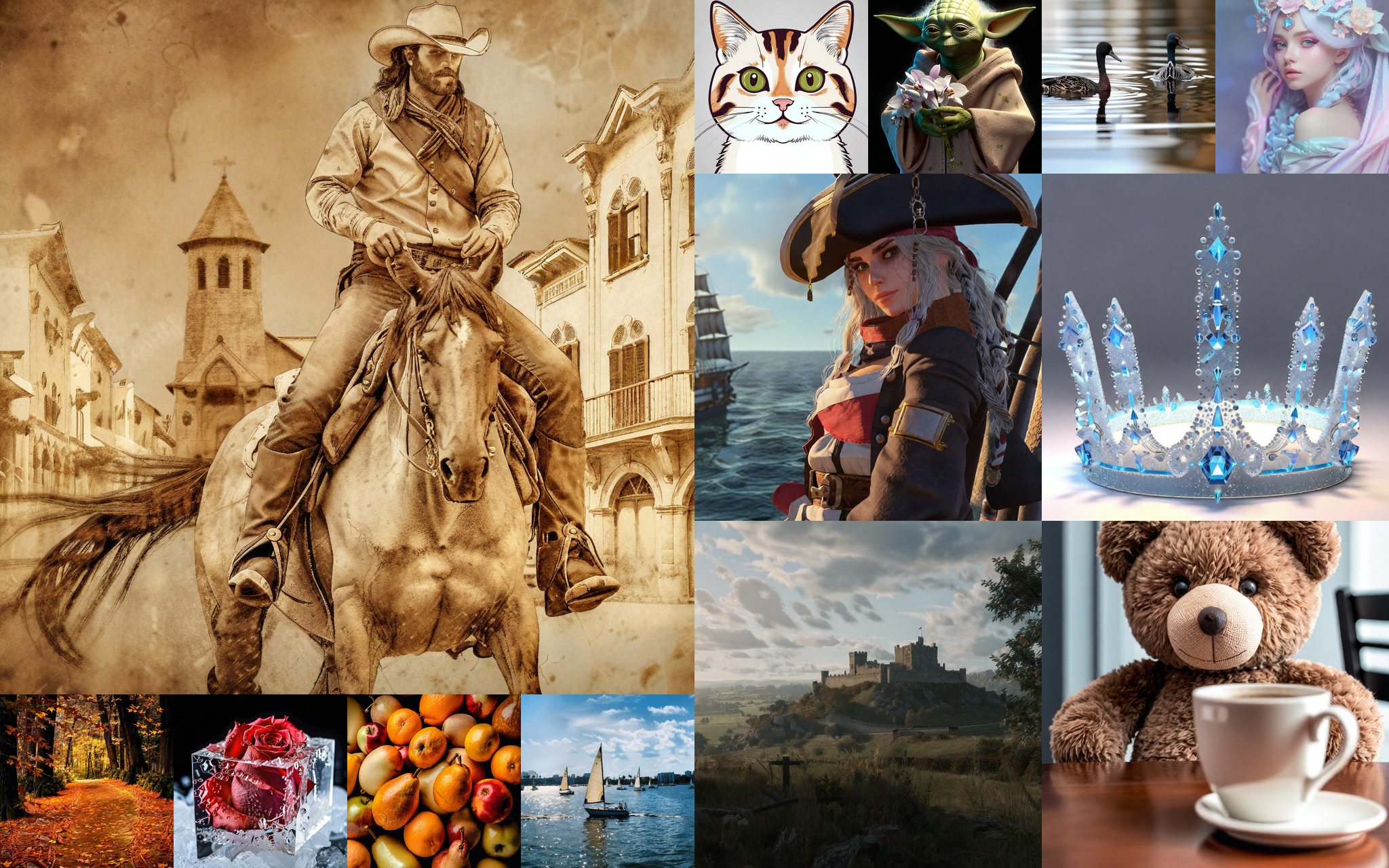}
    \caption{Example four-step generations at 1024$\times$1024 resolution using our SiD-based multistep distillation method (SiD$_2^{\alpha}$ with Anti-CFG).  Note that SiD$_2^{\alpha}$ initializes its generator from SiD, which is data-free, and continues training with access to a limited number of real text-image pairs.}

    \label{fig:qualitative_anti}
\vspace{7mm}

    \centering
    \includegraphics[width=1\linewidth]{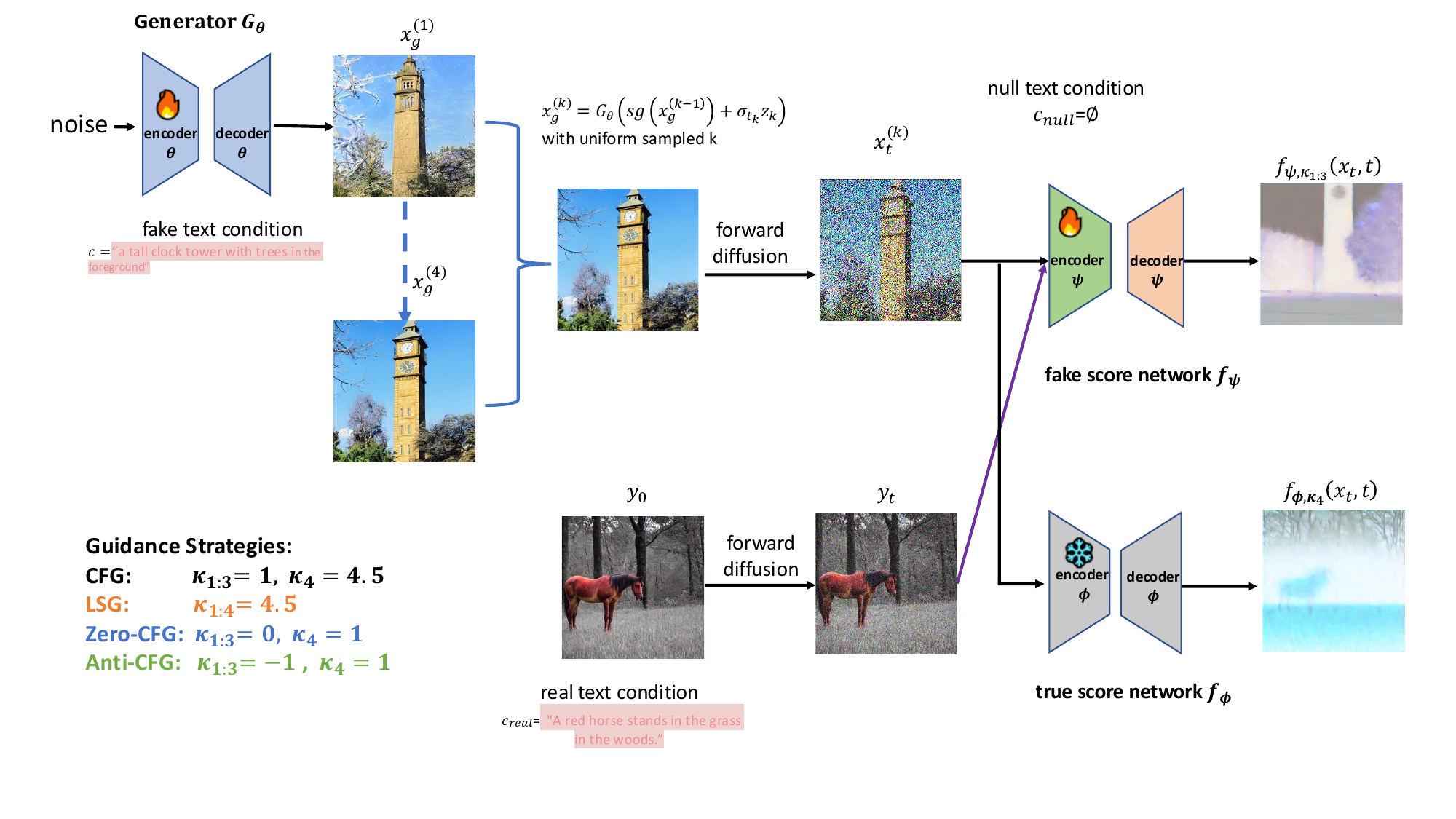}
    \vspace{-14mm}
    \caption{
    Illustration of SiD with multistep generation, data-dependent enhancement, and different guidance strategies.
    }\label{fig:model}
    \vspace{-4mm}
\end{figure}

\begin{figure*}[!ht]
\centering
   SD1.5 distilled with four-step SiD (LSG) FID = 16.60 , CLIP =0.319
 \begin{minipage}[b]{0.193\textwidth}
        \centering
        \includegraphics[width=\textwidth]{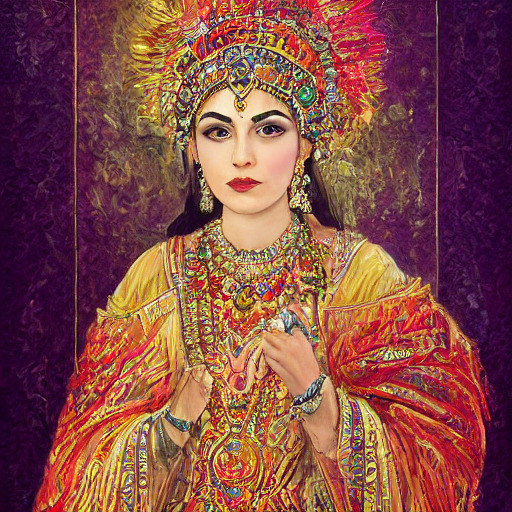}
     \\
    \end{minipage}~
    \begin{minipage}[b]{0.193\textwidth}
        \centering
        \includegraphics[width=\textwidth]{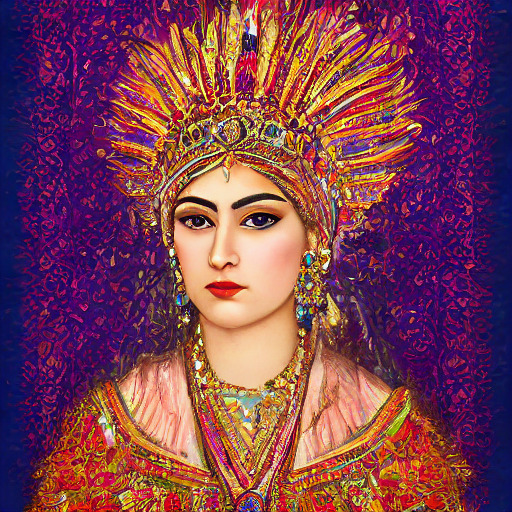}
     \\
    \end{minipage}~
    \begin{minipage}[b]{0.193\textwidth}
        \centering
        \includegraphics[width=\textwidth]{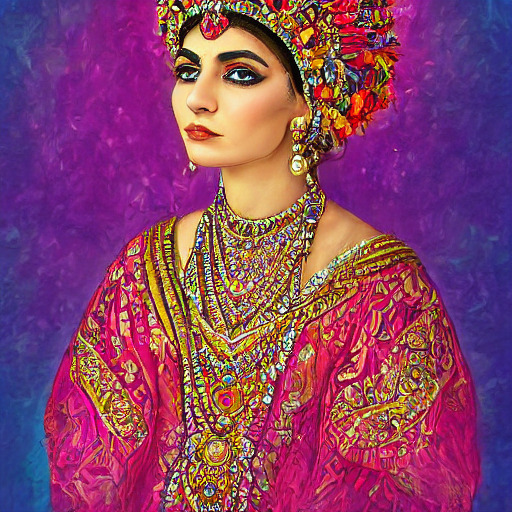}
     \\
    \end{minipage}~
    \begin{minipage}[b]{0.195\textwidth}
        \centering
        \includegraphics[width=\textwidth]{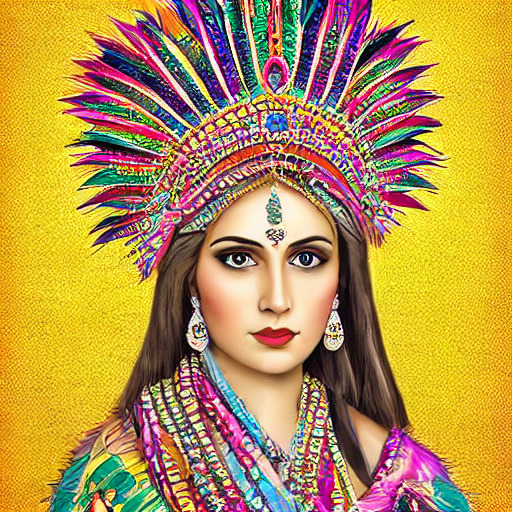}
     \\
    \end{minipage}~
    \begin{minipage}[b]{0.193\textwidth}
        \centering
        \includegraphics[width=\textwidth]{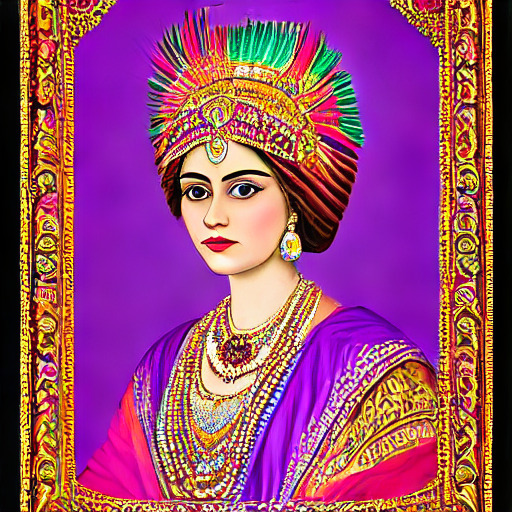}
     \\
    \end{minipage}
   SD1.5 distilled with four-step SiD$_2^{\alpha}$ (LSG ): FID = 12.52
, CLIP = 0.322
 \begin{minipage}[b]{0.193\textwidth}
        \centering
        \includegraphics[width=\textwidth]{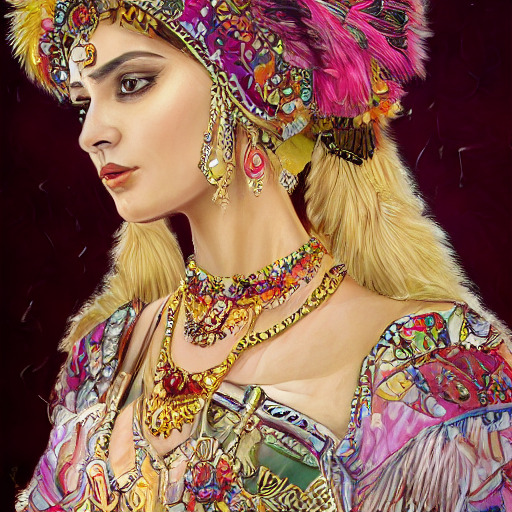}
     \\
    \end{minipage}~
    \begin{minipage}[b]{0.193\textwidth}
        \centering
        \includegraphics[width=\textwidth]{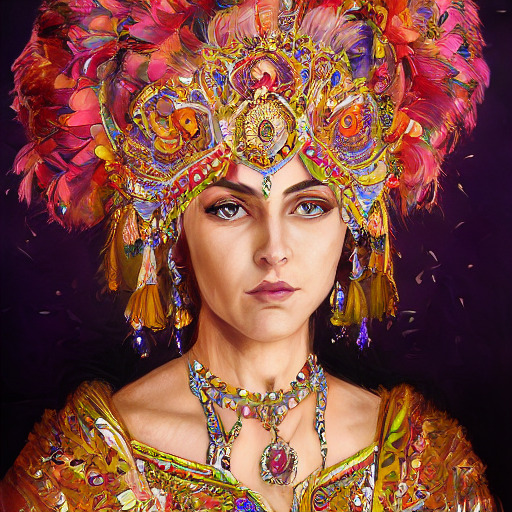}
     \\
    \end{minipage}~
    \begin{minipage}[b]{0.193\textwidth}
        \centering
        \includegraphics[width=\textwidth]{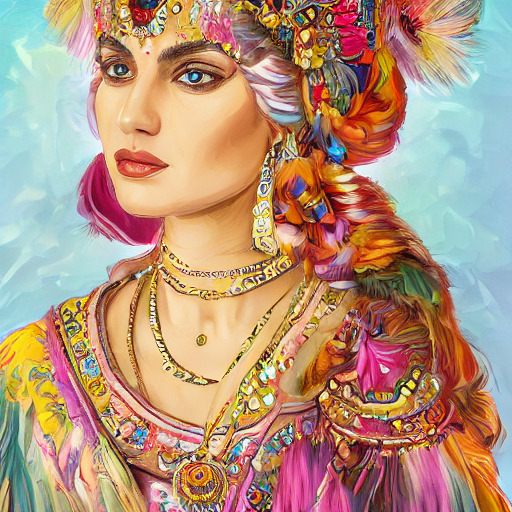}
     \\
    \end{minipage}~
    \begin{minipage}[b]{0.193\textwidth}
        \centering
        \includegraphics[width=\textwidth]{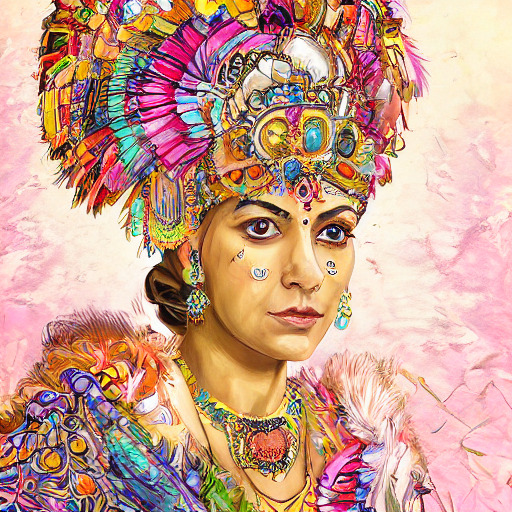}
     \\
    \end{minipage}~
    \begin{minipage}[b]{0.195\textwidth}
        \centering
        \includegraphics[width=\textwidth]{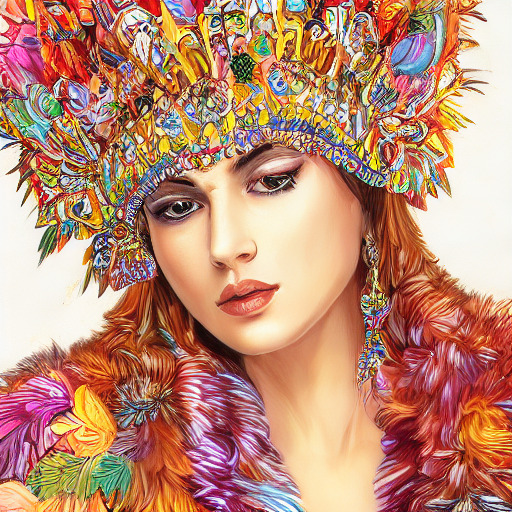}
     \\
    \end{minipage}
    \\
 SD1.5 distilled with four-step SiD$_2^a$ (Zero-CFG) FID  = 14.71 , CLIP = 0.324
 \begin{minipage}[b]{0.193\textwidth}
        \centering
        \includegraphics[width=\textwidth]{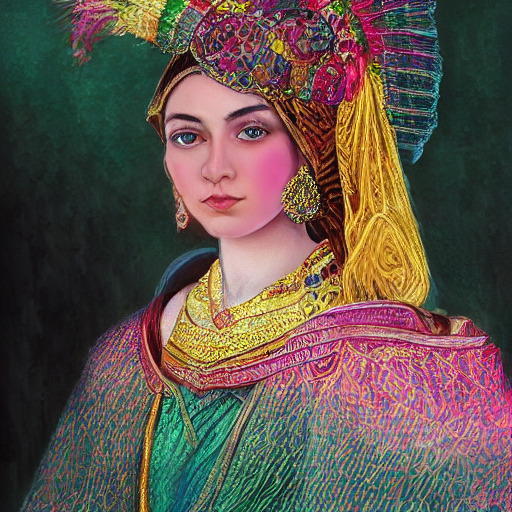}
     \\
    \end{minipage}~
    \begin{minipage}[b]{0.193\textwidth}
        \centering
        \includegraphics[width=\textwidth]{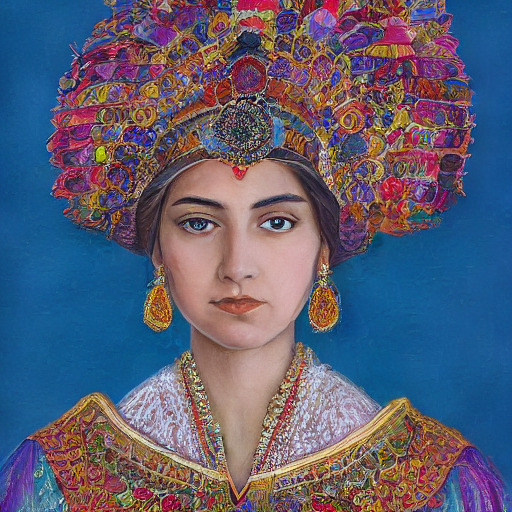}
     \\
    \end{minipage}~
    \begin{minipage}[b]{0.193\textwidth}
        \centering
        \includegraphics[width=\textwidth]{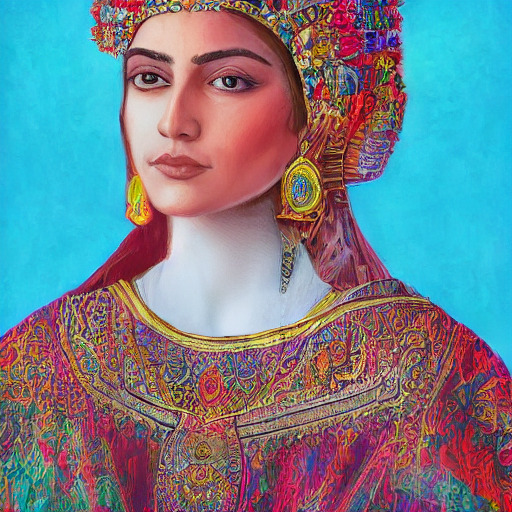}
     \\
    \end{minipage}~
    \begin{minipage}[b]{0.193\textwidth}
        \centering
        \includegraphics[width=\textwidth]{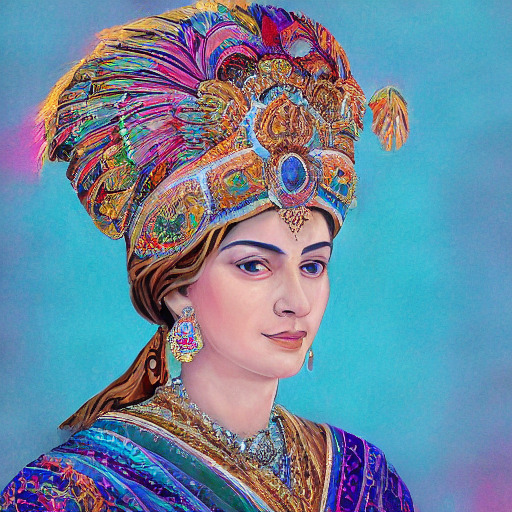}
     \\
    \end{minipage}~
    \begin{minipage}[b]{0.195\textwidth}
        \centering
        \includegraphics[width=\textwidth]{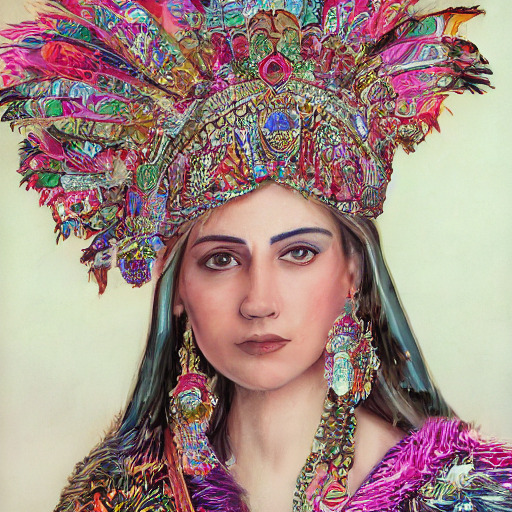}
     \\
    \end{minipage}
    \\
    SD1.5 distilled with four-step SiD$_2^{\alpha}$ (Anti-CFG): FID = 13.35
, CLIP = \textbf{0.325}
 \begin{minipage}[b]{0.193\textwidth}
        \centering
        \includegraphics[width=\textwidth]{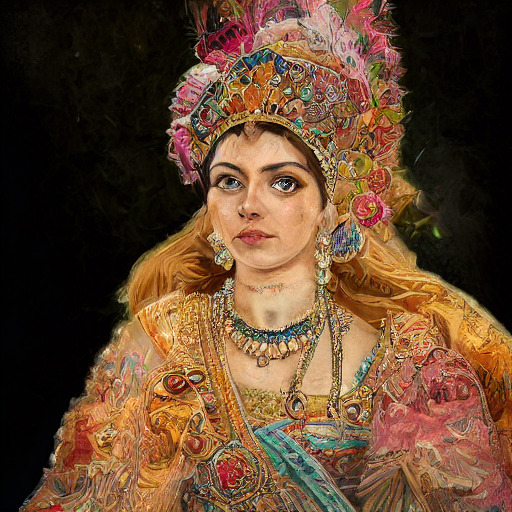}
     \\
    \end{minipage}~
    \begin{minipage}[b]{0.193\textwidth}
        \centering
        \includegraphics[width=\textwidth]{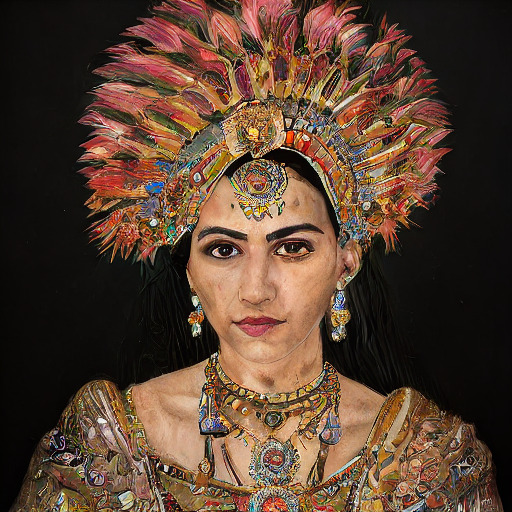}
     \\
    \end{minipage}~
    \begin{minipage}[b]{0.193\textwidth}
        \centering
        \includegraphics[width=\textwidth]{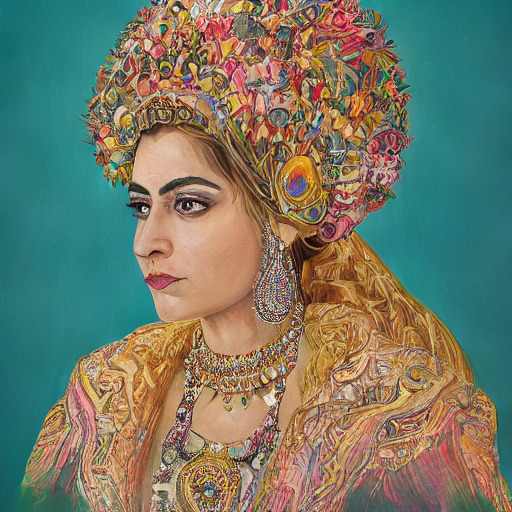}
     \\
    \end{minipage}~
    \begin{minipage}[b]{0.193\textwidth}
        \centering
        \includegraphics[width=\textwidth]{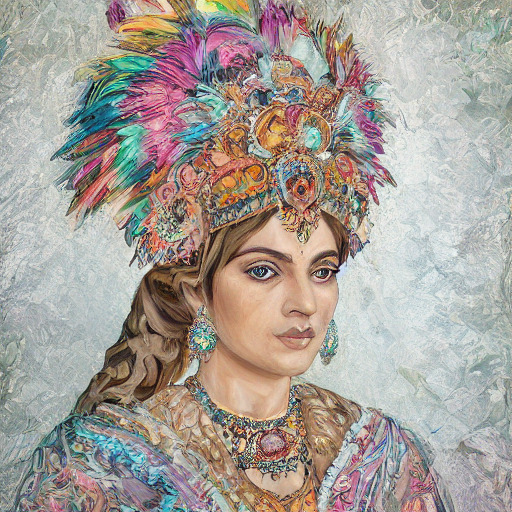}
     \\
    \end{minipage}~
    \begin{minipage}[b]{0.195\textwidth}
        \centering
        \includegraphics[width=\textwidth]{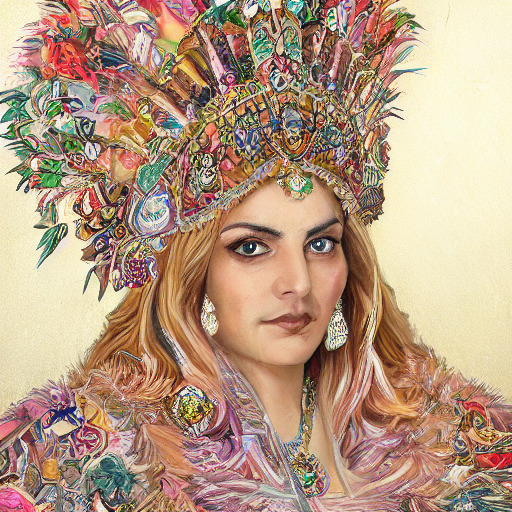}
     \\
    \end{minipage}

    \caption{\small 
We present a visual comparison of SiD-distilled four-step generators from SD1.5 using different guidance strategies. SiD with LSG achieves a balanced performance, with an FID of 16.60 and a CLIP score of 0.319, demonstrating that this data-free approach is already competitive with leading SD1.5 distillation methods. Building on this, SiD-based initialization followed by enhancement with limited real data further improves generation diversity, maintaining strong visual fidelity and better adherence to the text prompt: 1) \textbf{SiD$_2^a$ (LSG)} significantly increases generation diversity while further improving text--image alignment.
  2) Both \textbf{SiD$_2^a$ (Zero-CFG)} and \textbf{SiD$_2^a$ (Anti-CFG)} surpass the baseline SiD in diversity, with SiD$_2^a$ (Anti-CFG) achieving the highest CLIP score.
All images in the comparison are generated from the same prompt:  
\textit{``A regal female portrait with an ornate headdress decorated with colorful gemstones and feathers, her robes rich with intricate designs and bright hues. 8K, best quality, fine details.''}
\normalsize}
    \label{fig:qualitative_2}
\end{figure*}
\vspace{5mm}

\pagebreak 
\begin{figure*}[!ht]
\centering
    
    SDXL distilled with four-step SiD (LSG): FID = 21.04
, CLIP = 0.340
 \begin{minipage}[b]{0.193\textwidth}
        \centering
        \includegraphics[width=\textwidth]{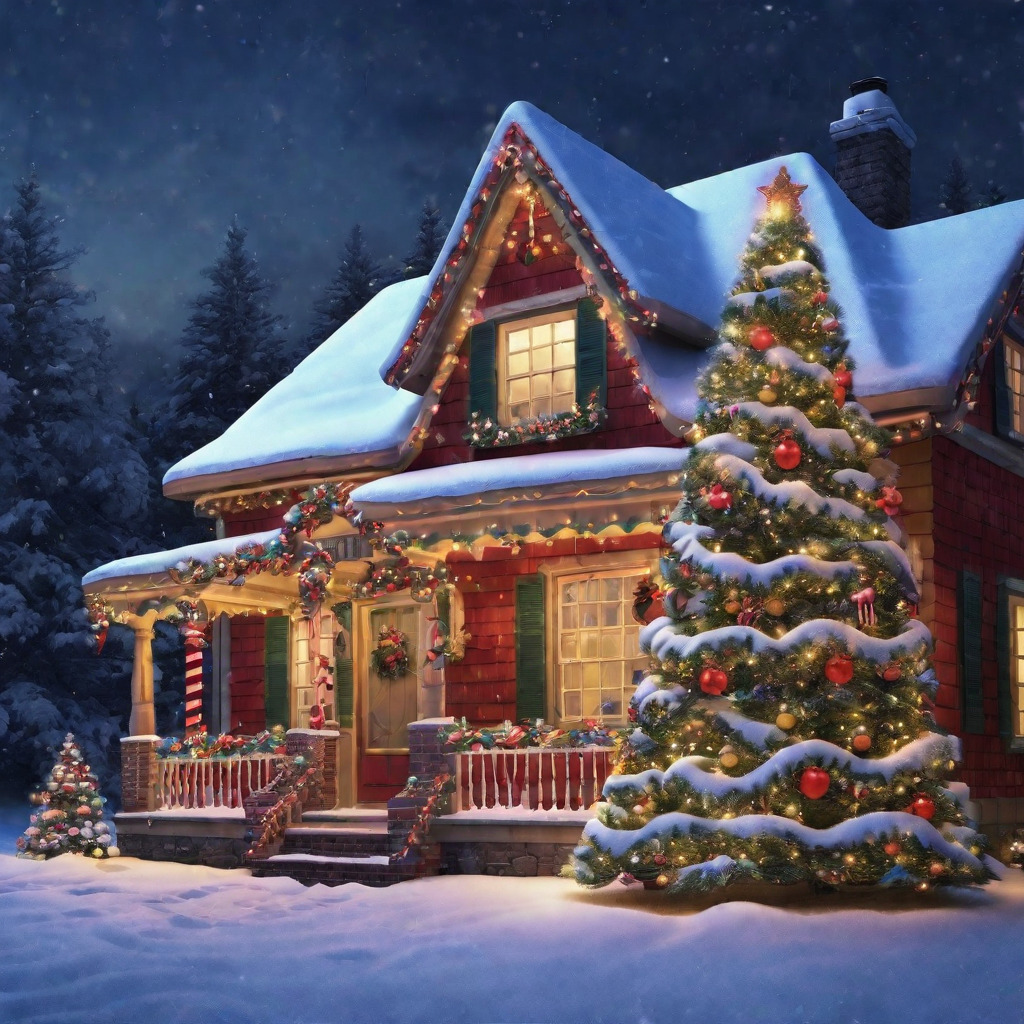}
     \\
    \end{minipage}~
    \begin{minipage}[b]{0.193\textwidth}
        \centering
        \includegraphics[width=\textwidth]{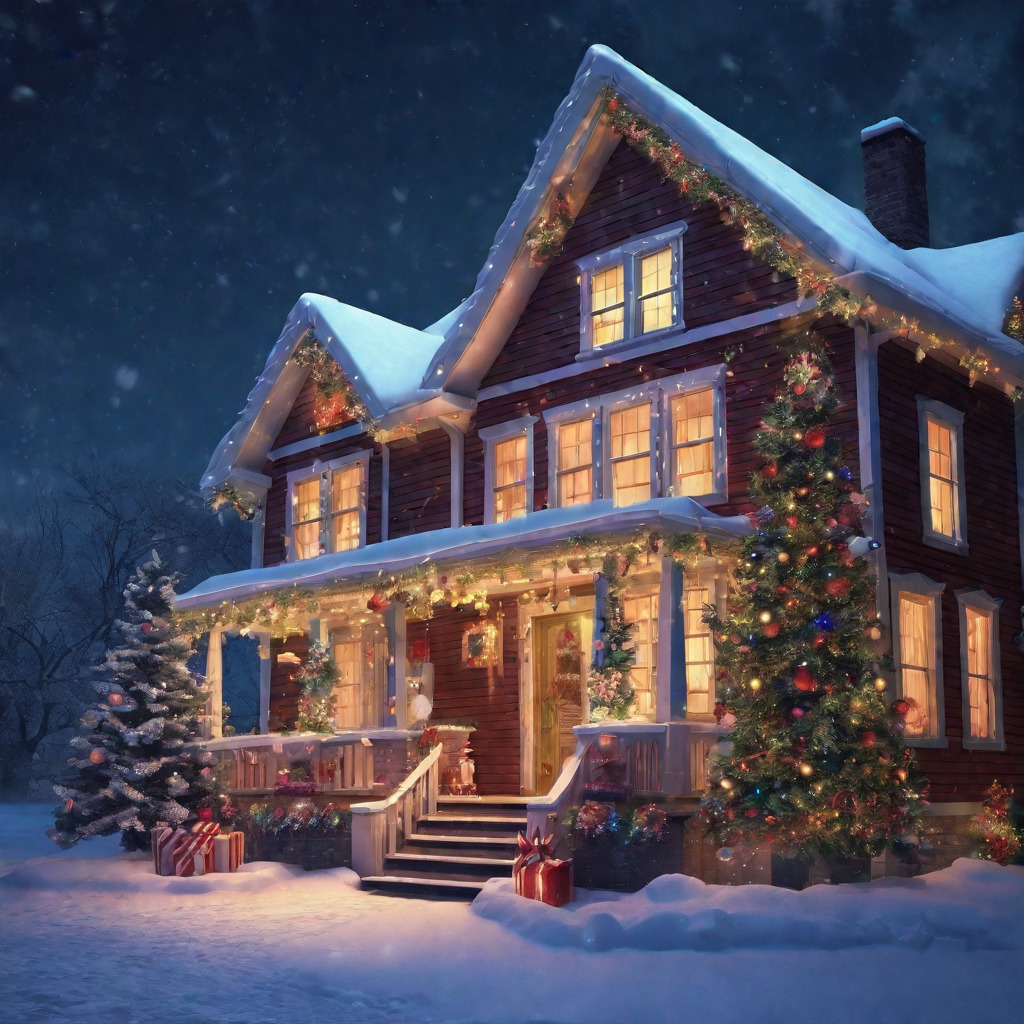}
     \\
    \end{minipage}~
    \begin{minipage}[b]{0.193\textwidth}
        \centering
        \includegraphics[width=\textwidth]{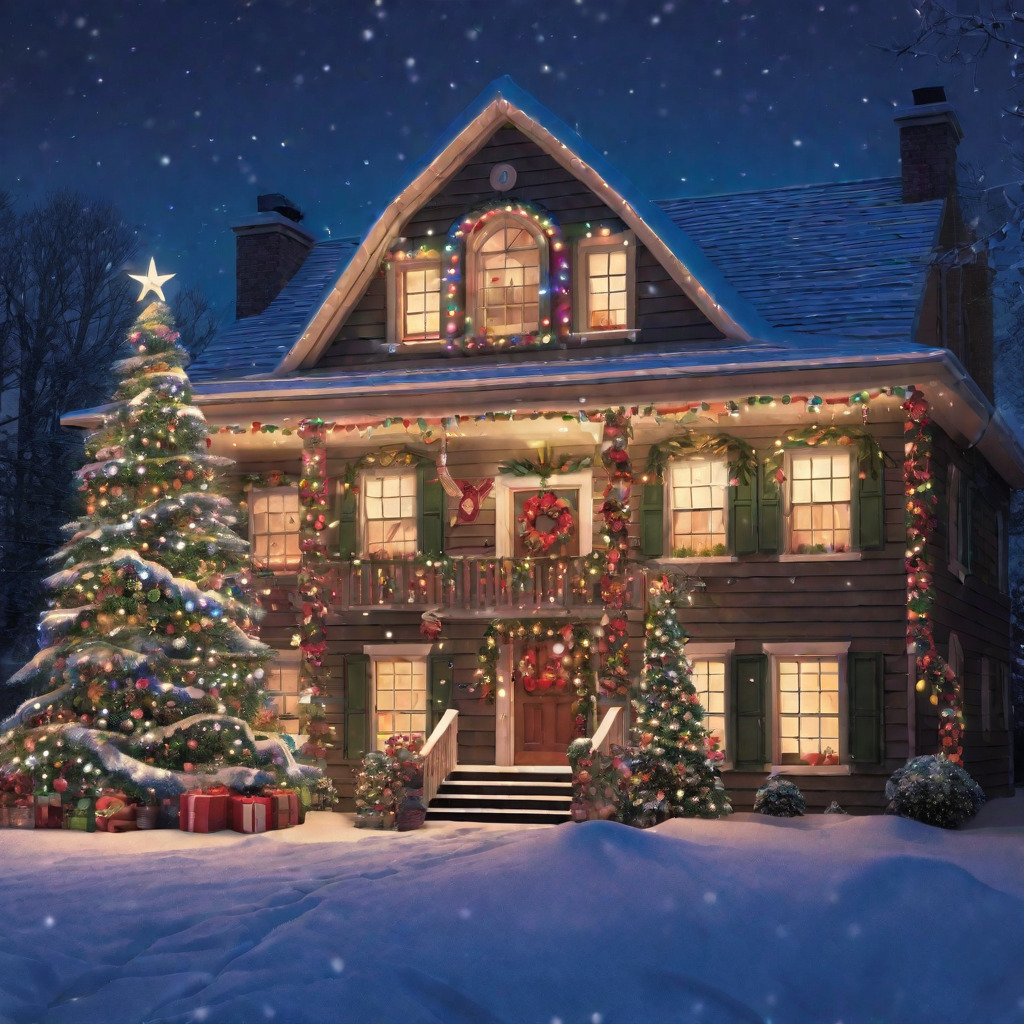}
     \\
    \end{minipage}~
    \begin{minipage}[b]{0.195\textwidth}
        \centering
        \includegraphics[width=\textwidth]{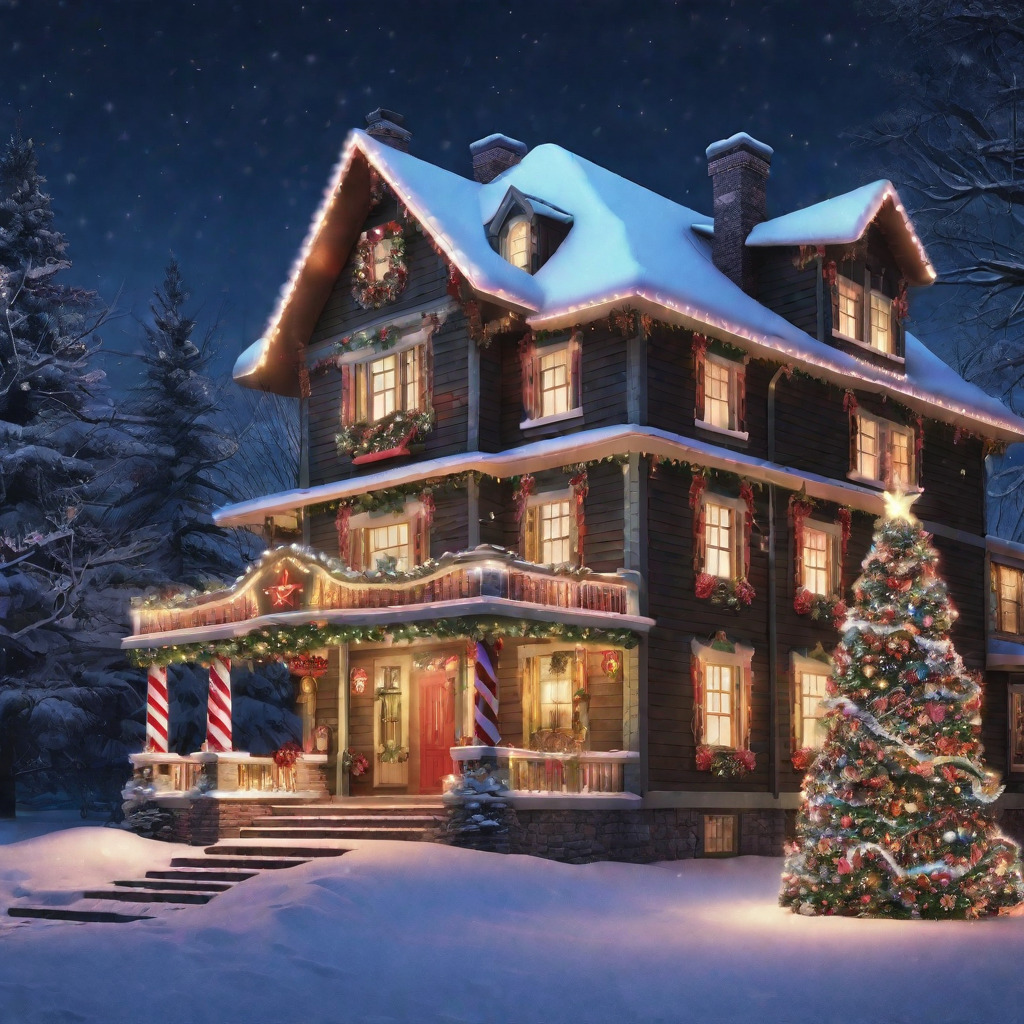}
     \\
    \end{minipage}~
    \begin{minipage}[b]{0.193\textwidth}
        \centering
        \includegraphics[width=\textwidth]{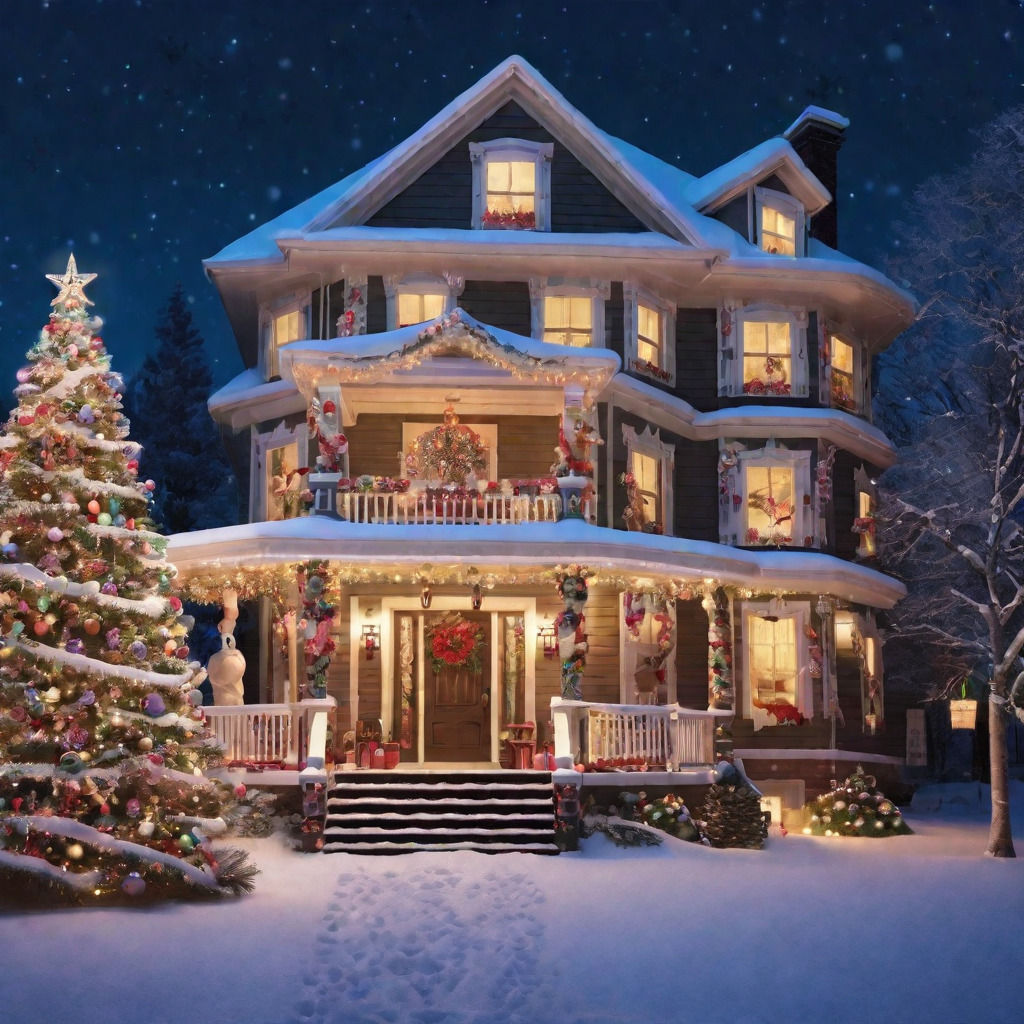}
     \\
    \end{minipage}
  \centering  SDXL distilled with four-step SiD$_2^{\alpha}$ (LSG): FID = 17.19, CLIP = 0.342\\
 \begin{minipage}[b]{0.193\textwidth}
        \centering
        \includegraphics[width=\textwidth]{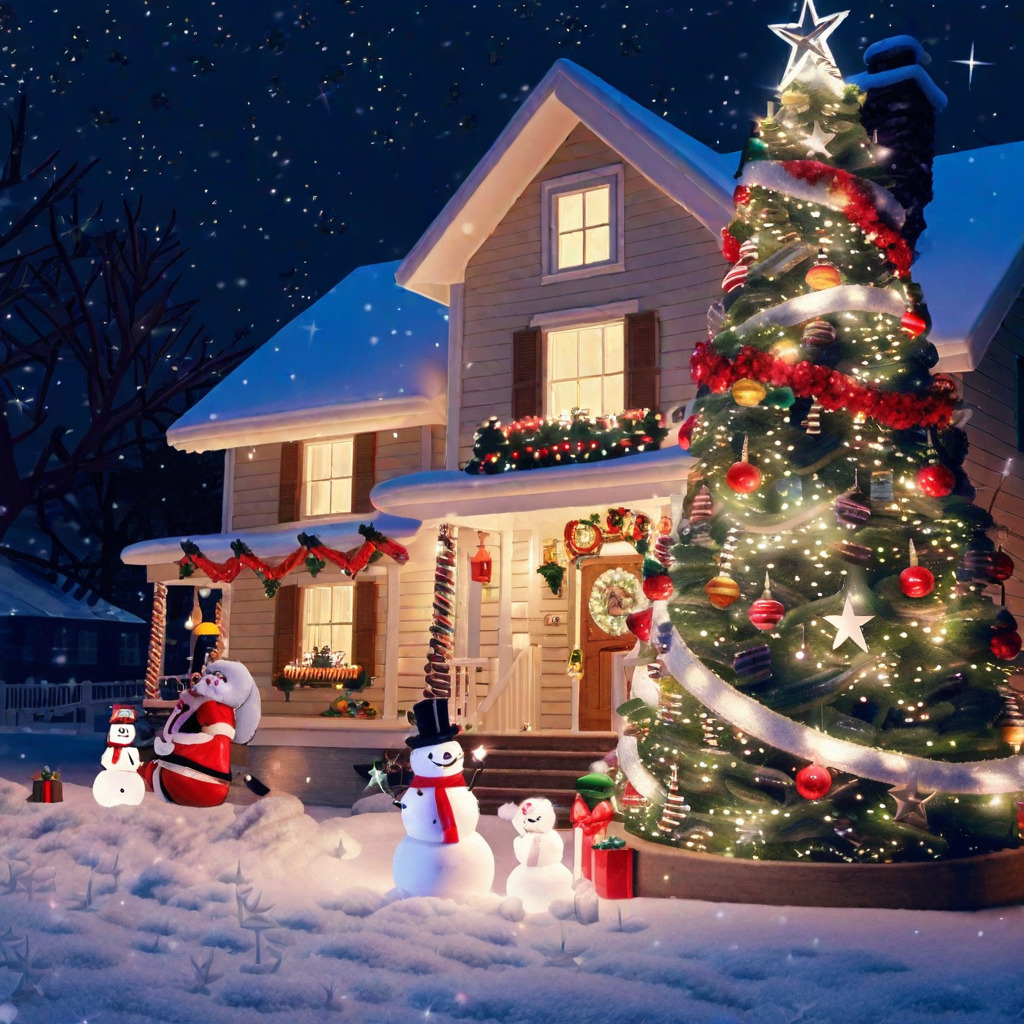}
     \\
    \end{minipage}~
    \begin{minipage}[b]{0.193\textwidth}
        \centering
        \includegraphics[width=\textwidth]{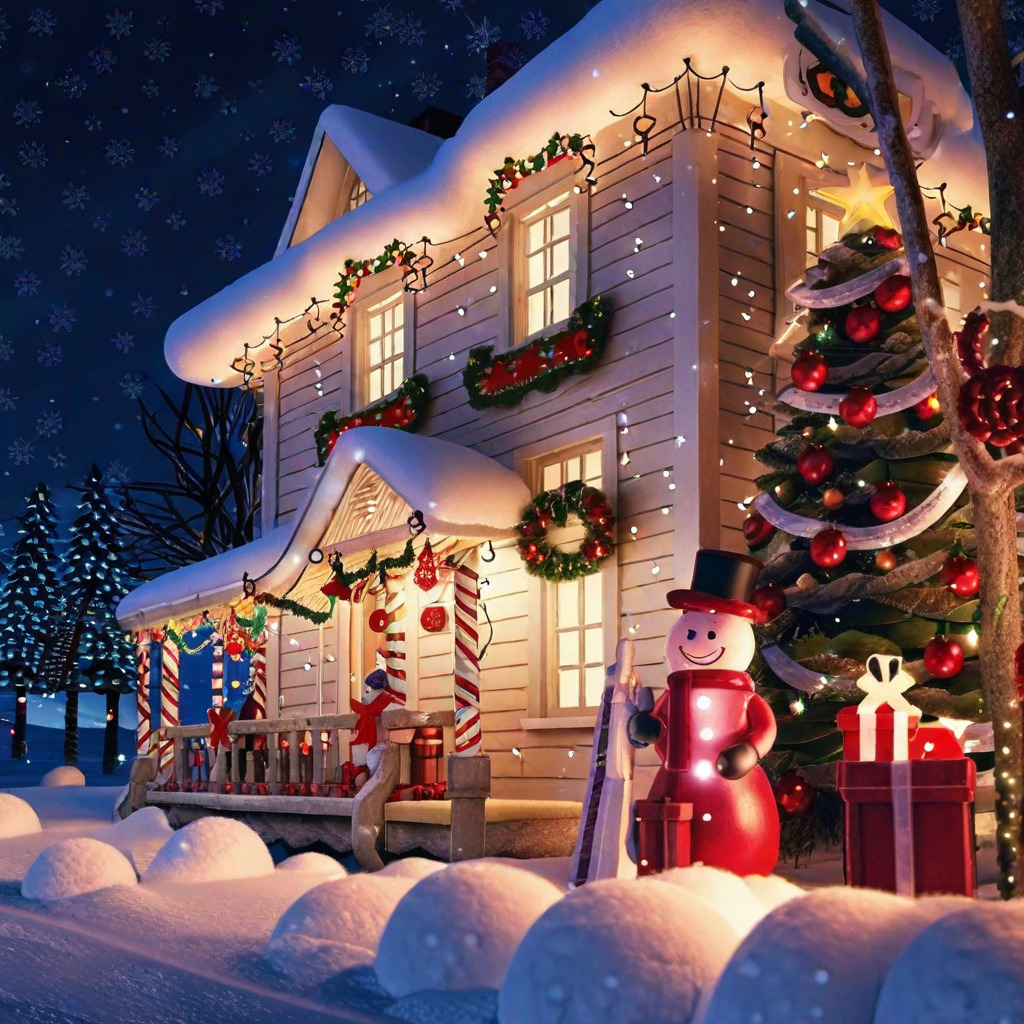}
     \\
    \end{minipage}~
    \begin{minipage}[b]{0.193\textwidth}
        \centering
        \includegraphics[width=\textwidth]{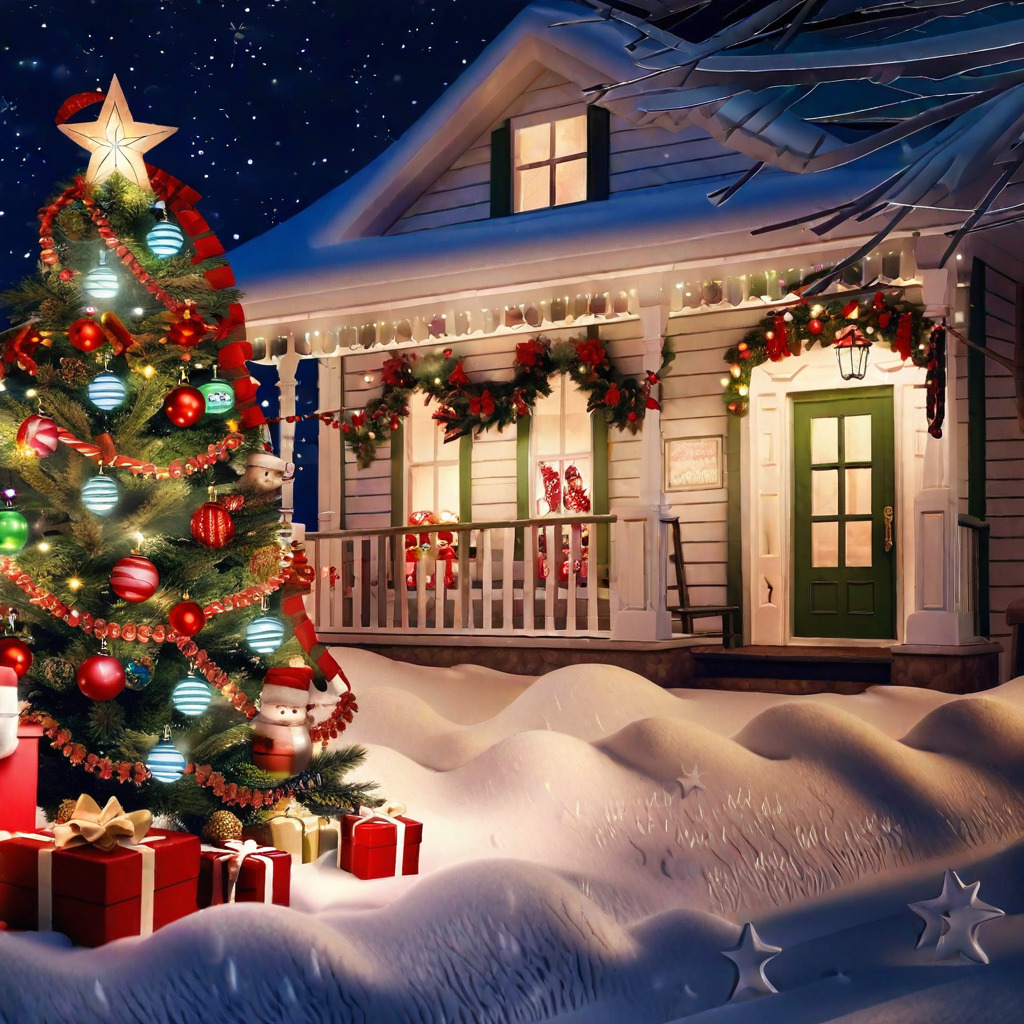}
     \\
    \end{minipage}~
    \begin{minipage}[b]{0.193\textwidth}
        \centering
        \includegraphics[width=\textwidth]{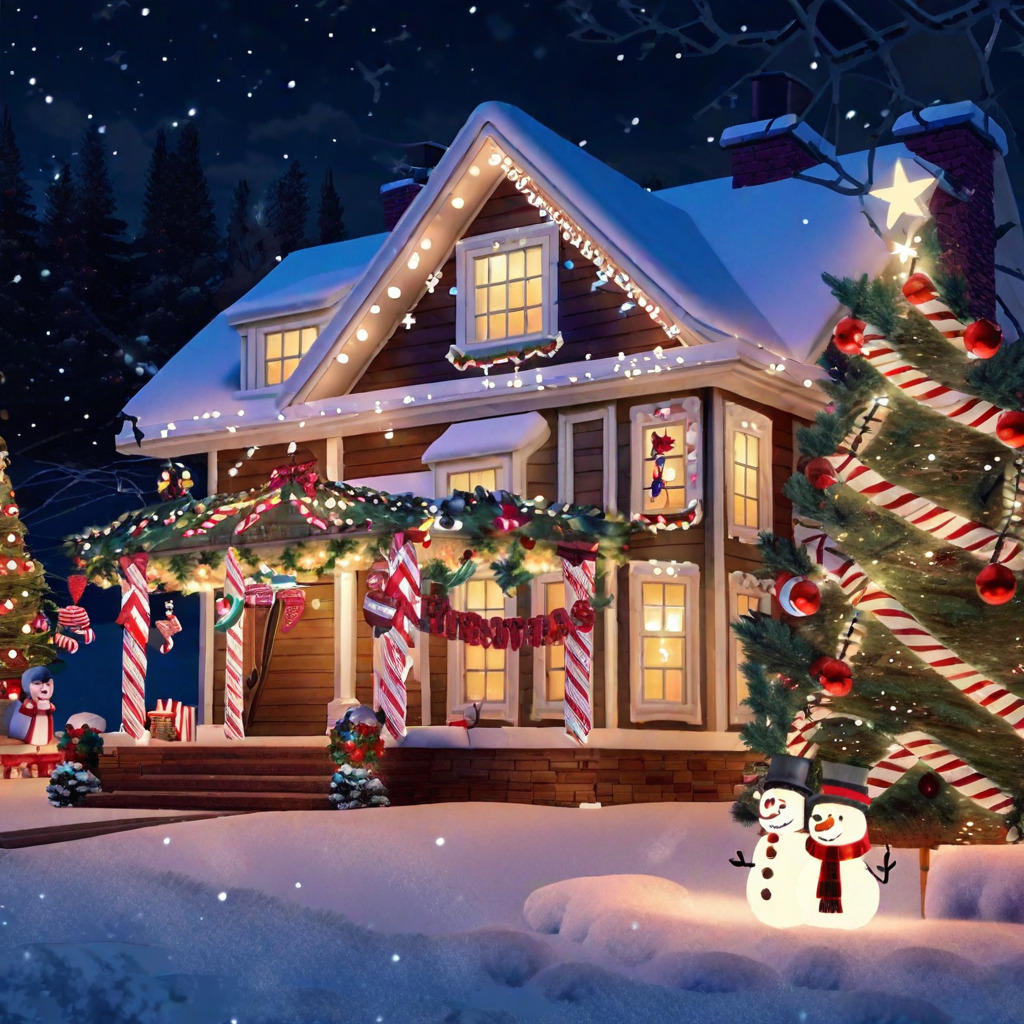}
     \\
    \end{minipage}~
    \begin{minipage}[b]{0.195\textwidth}
        \centering
        \includegraphics[width=\textwidth]{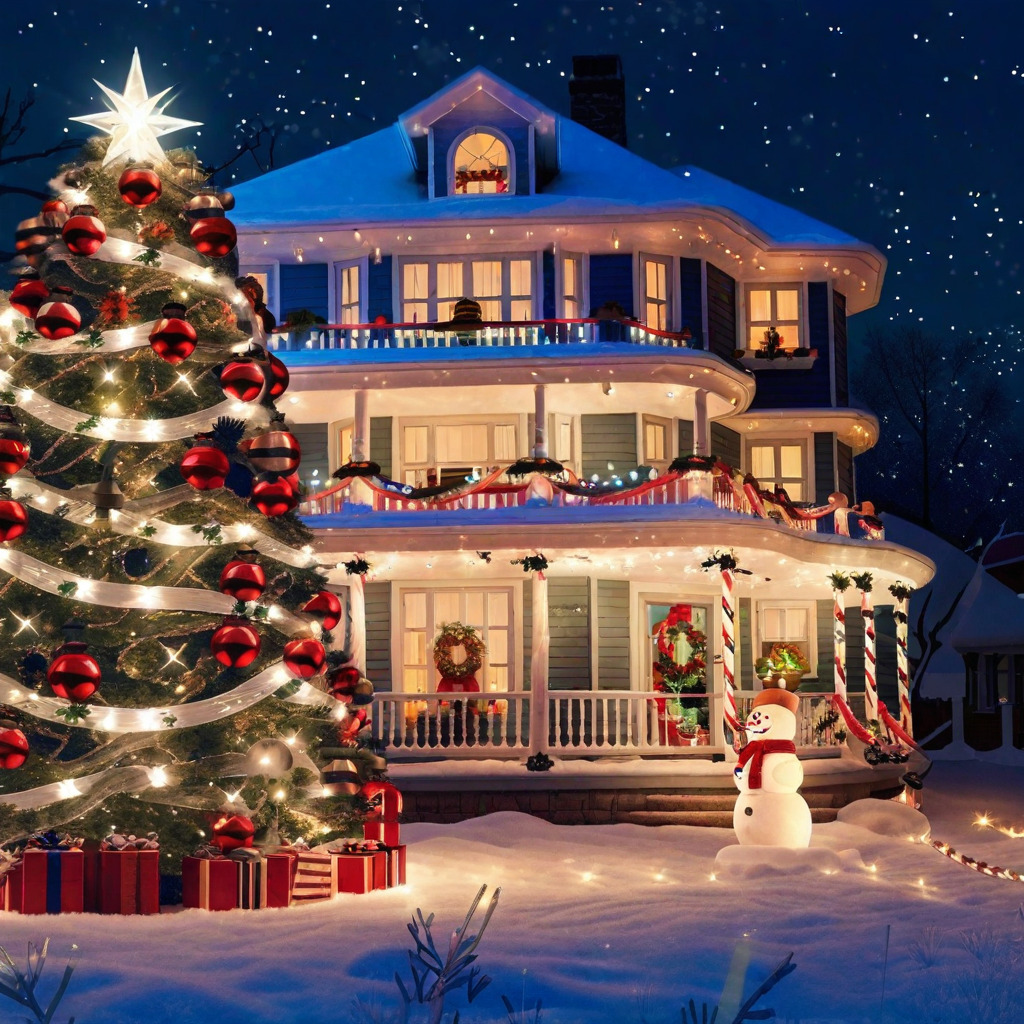}
     \\
    \end{minipage}
    \\
    
    SDXL distilled with four-step SiD$_2^{\alpha}$ (Zero-CFG): FID = 15.92
, CLIP = 0.342
    \begin{minipage}[b]{0.193\textwidth}
        \centering
        \includegraphics[width=\textwidth]{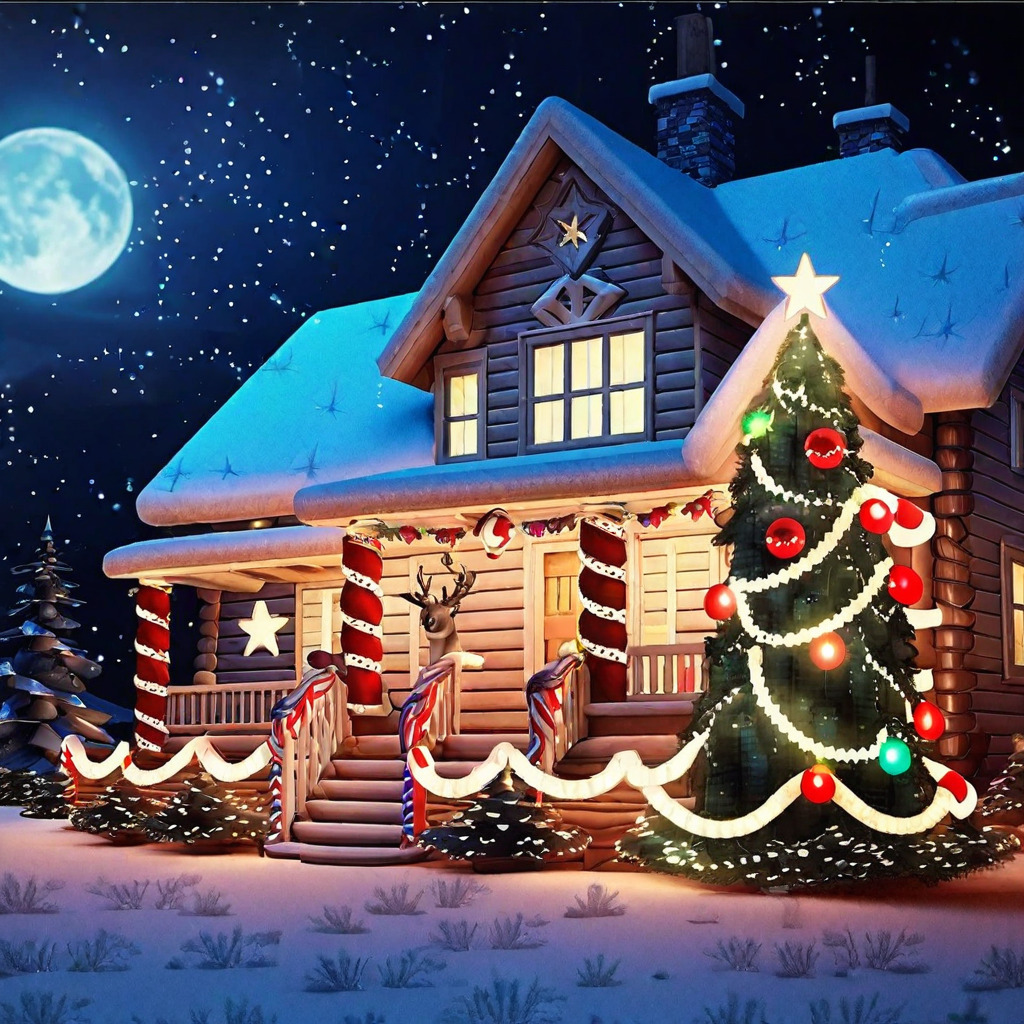}
     \\
    \end{minipage}~
    \begin{minipage}[b]{0.193\textwidth}
        \centering
        \includegraphics[width=\textwidth]{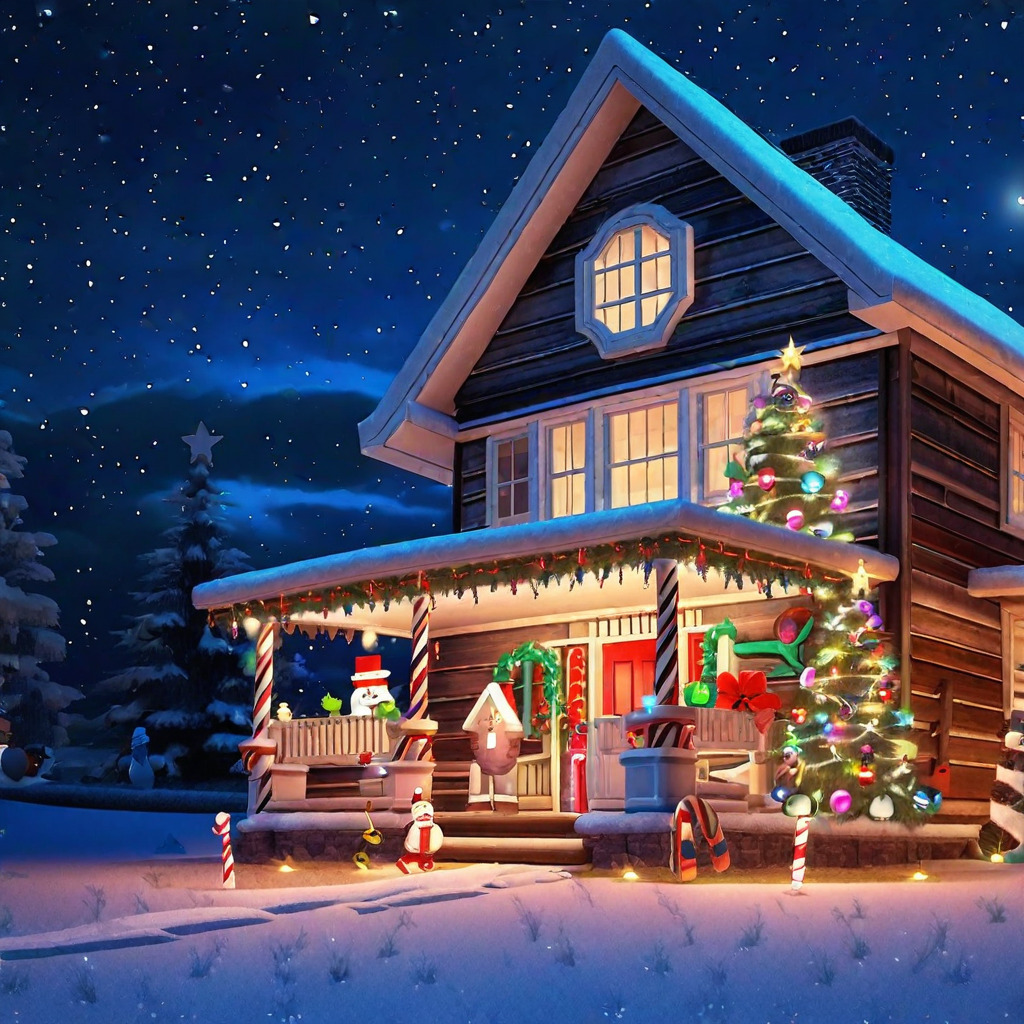}
     \\
    \end{minipage}~
    \begin{minipage}[b]{0.193\textwidth}
        \centering
        \includegraphics[width=\textwidth]{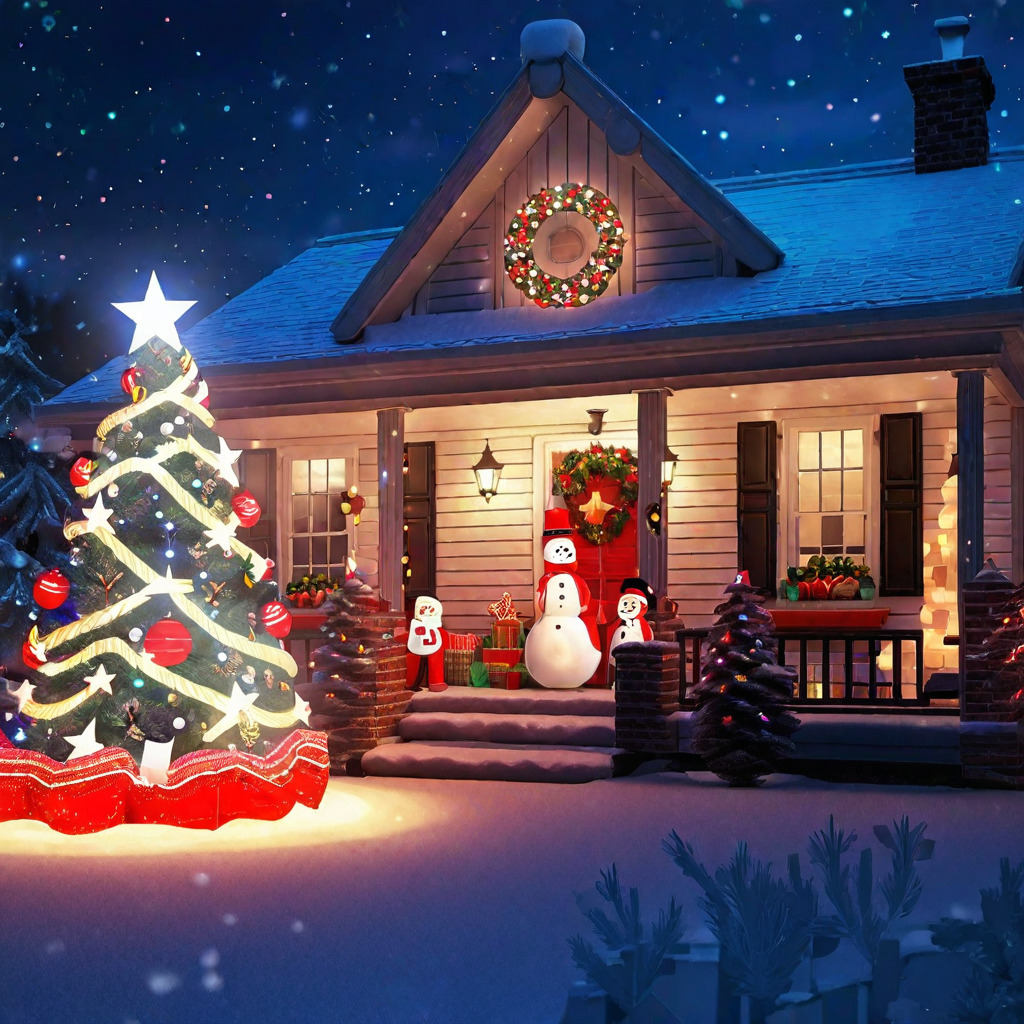}
     \\
    \end{minipage}~
    \begin{minipage}[b]{0.193\textwidth}
        \centering
        \includegraphics[width=\textwidth]{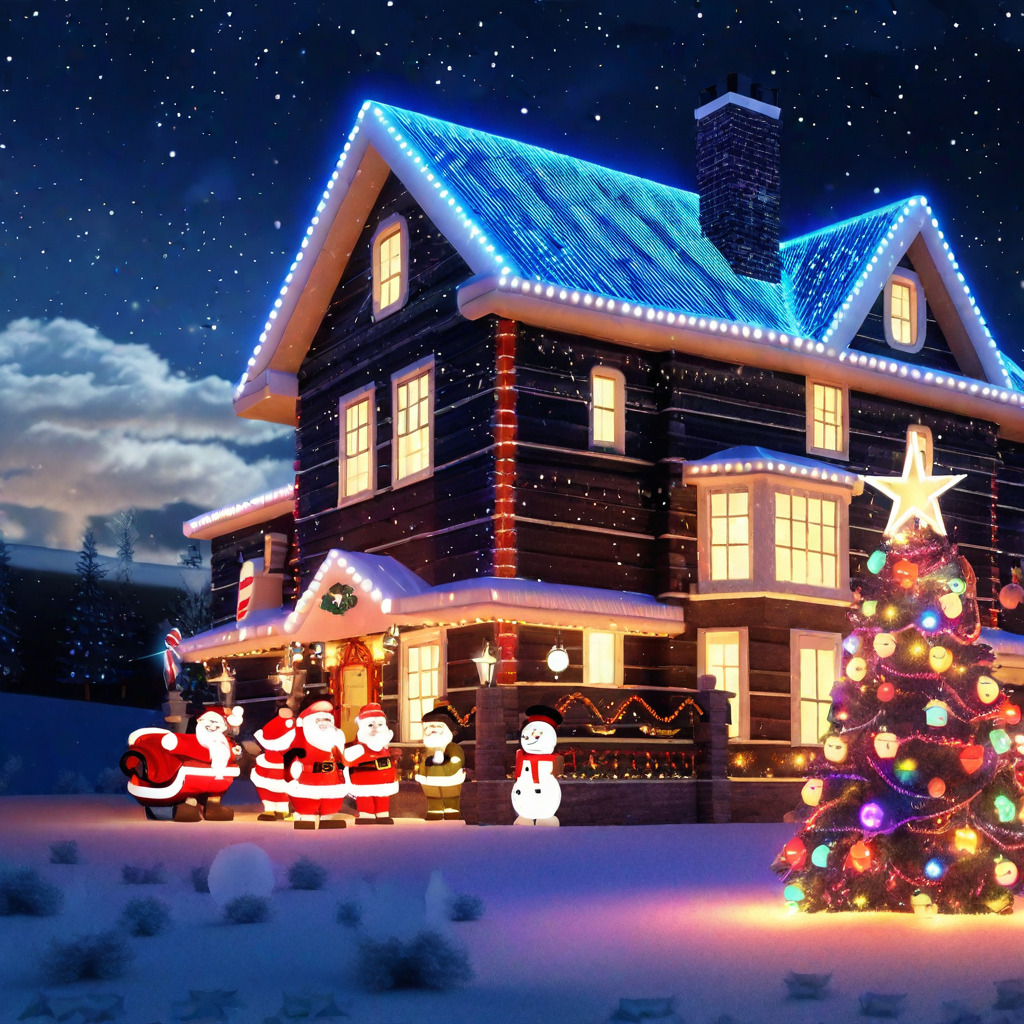}
     \\
    \end{minipage}~
    \begin{minipage}[b]{0.195\textwidth}
        \centering
        \includegraphics[width=\textwidth]{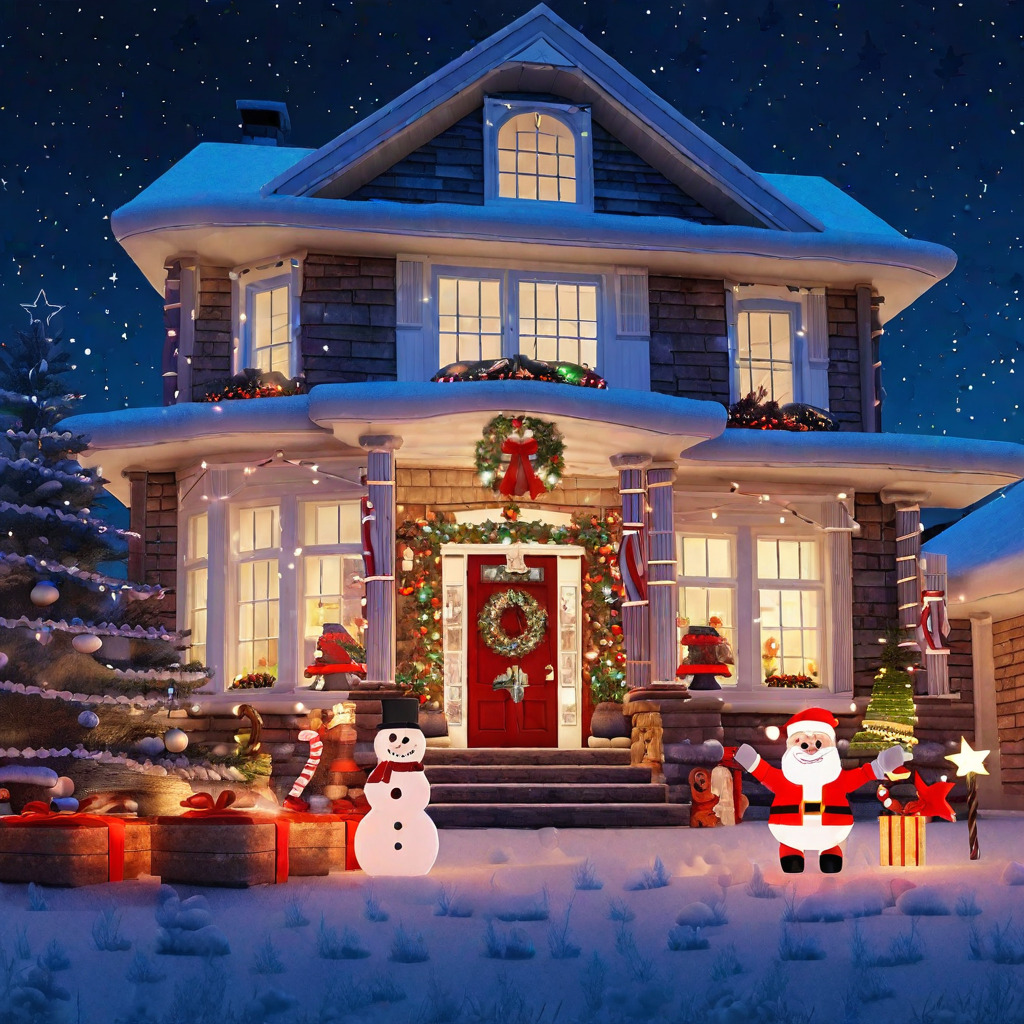}
     \\
    \end{minipage}\\
    \centering
SDXL distilled with four-step SiD$_2^{\alpha}$ (Anti-CFG): FID = 15.82
, CLIP = \textbf{0.344
}
    
 \begin{minipage}[b]{0.193\textwidth}
        \centering
        \includegraphics[width=\textwidth]{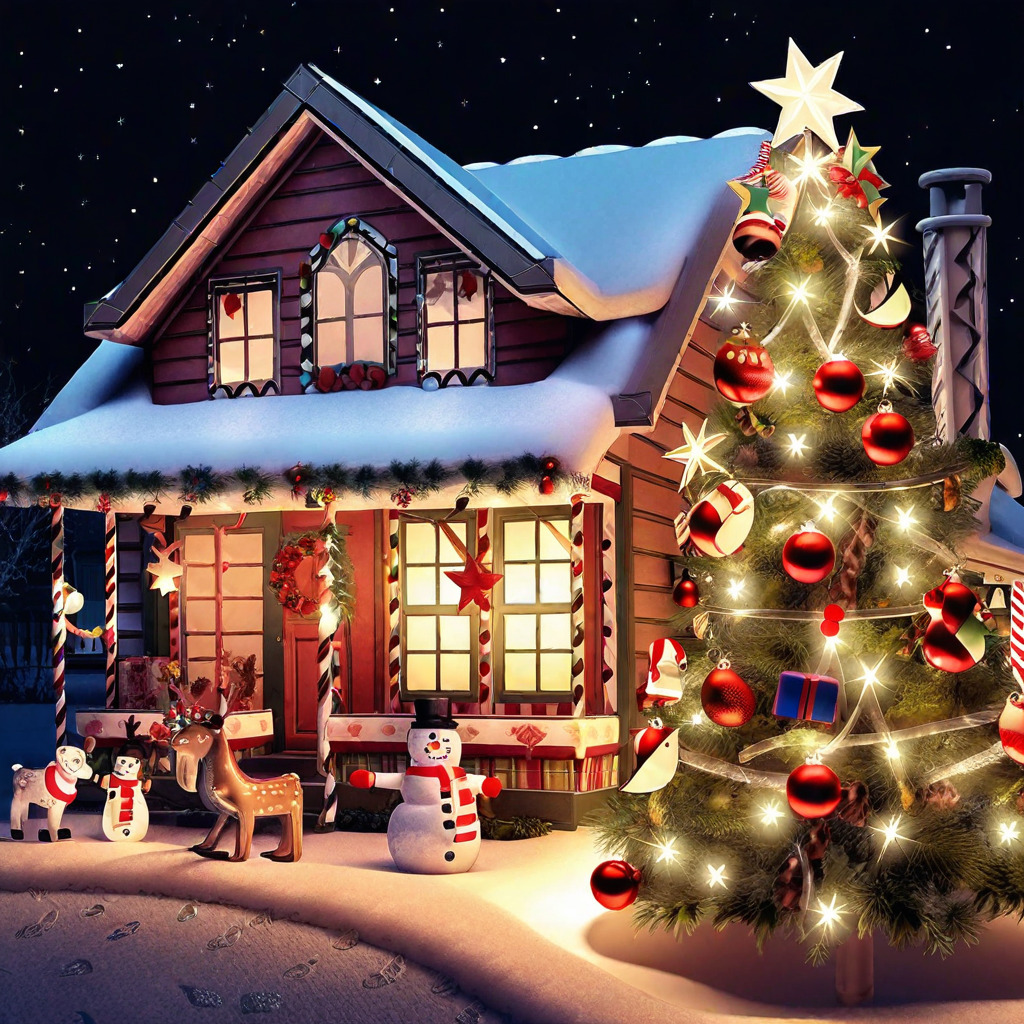}
     \\
    \end{minipage}~
    \begin{minipage}[b]{0.193\textwidth}
        \centering
        \includegraphics[width=\textwidth]{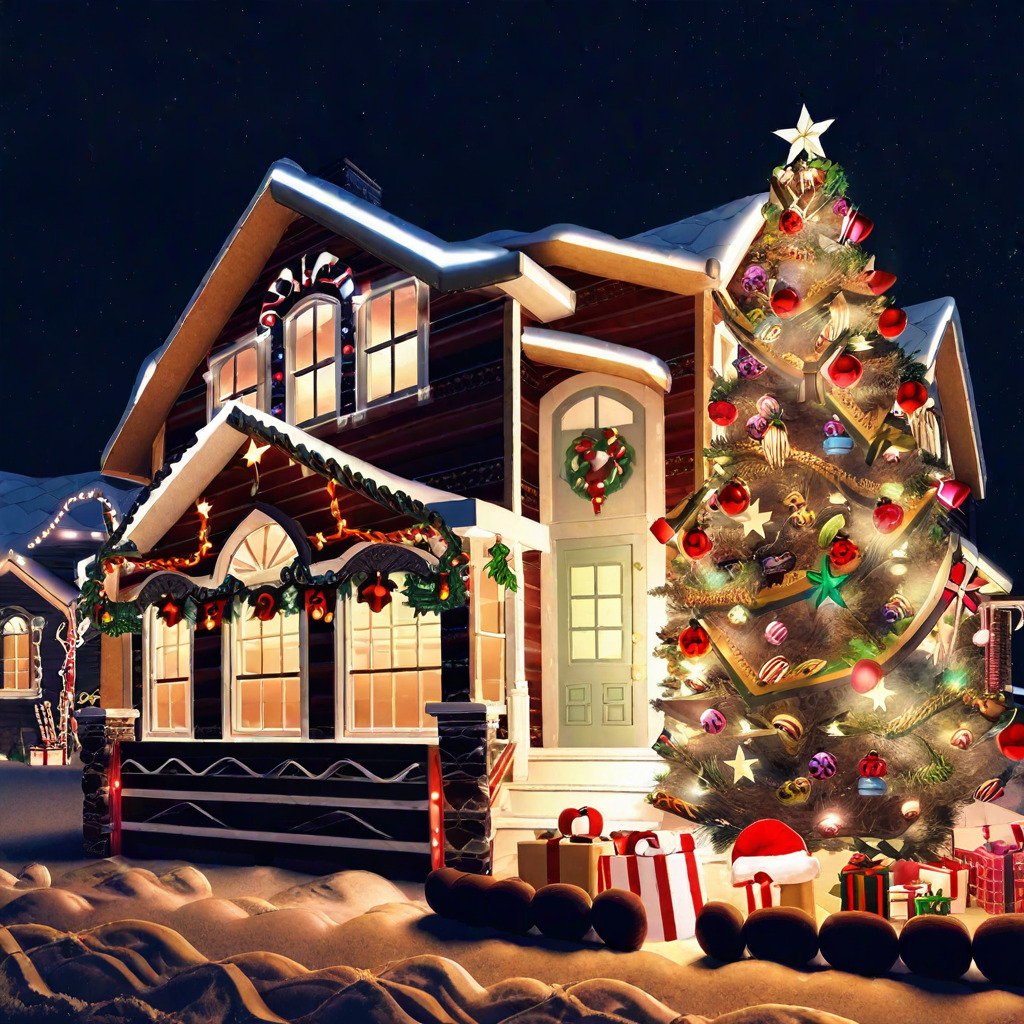}
     \\
    \end{minipage}~
    \begin{minipage}[b]{0.193\textwidth}
        \centering
        \includegraphics[width=\textwidth]{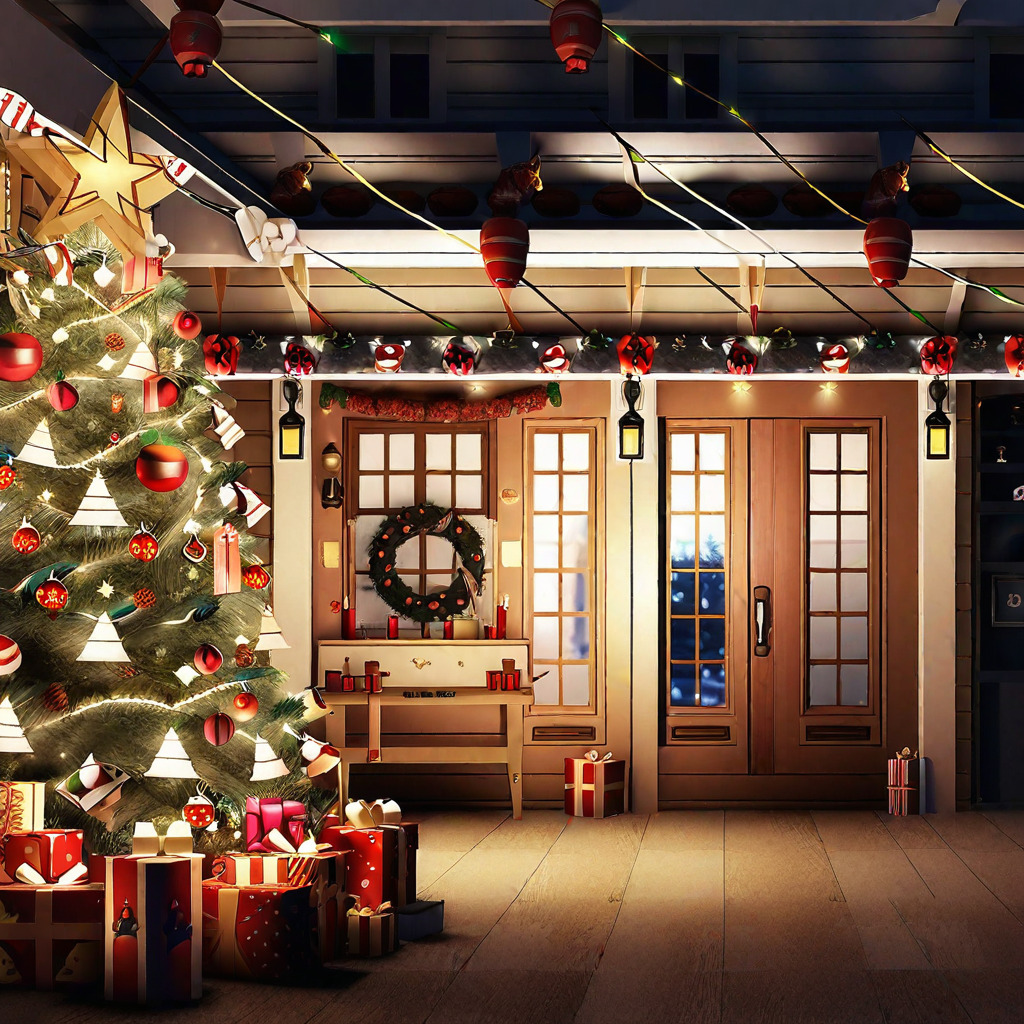}
     \\
    \end{minipage}~
    \begin{minipage}[b]{0.193\textwidth}
        \centering
        \includegraphics[width=\textwidth]{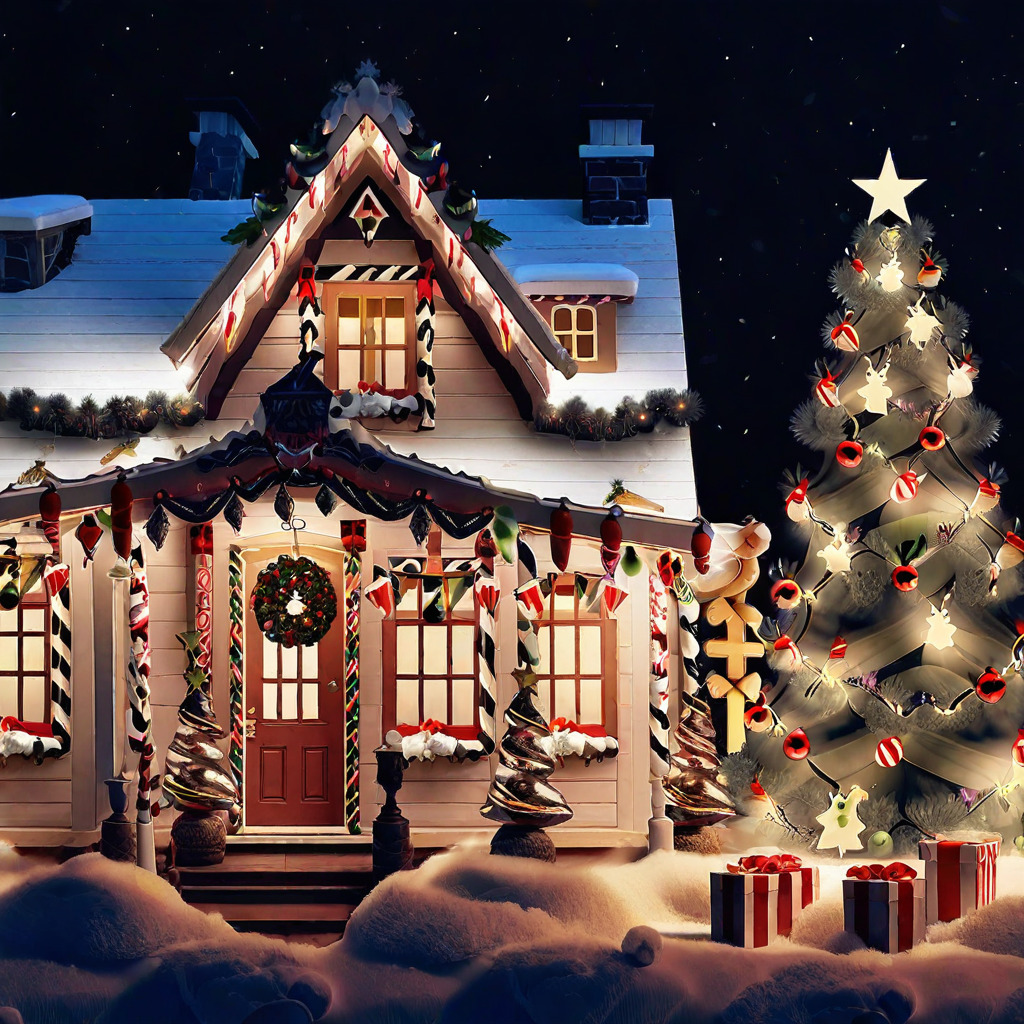}
     \\
    \end{minipage}~
    \begin{minipage}[b]{0.195\textwidth}
        \centering
        \includegraphics[width=\textwidth]{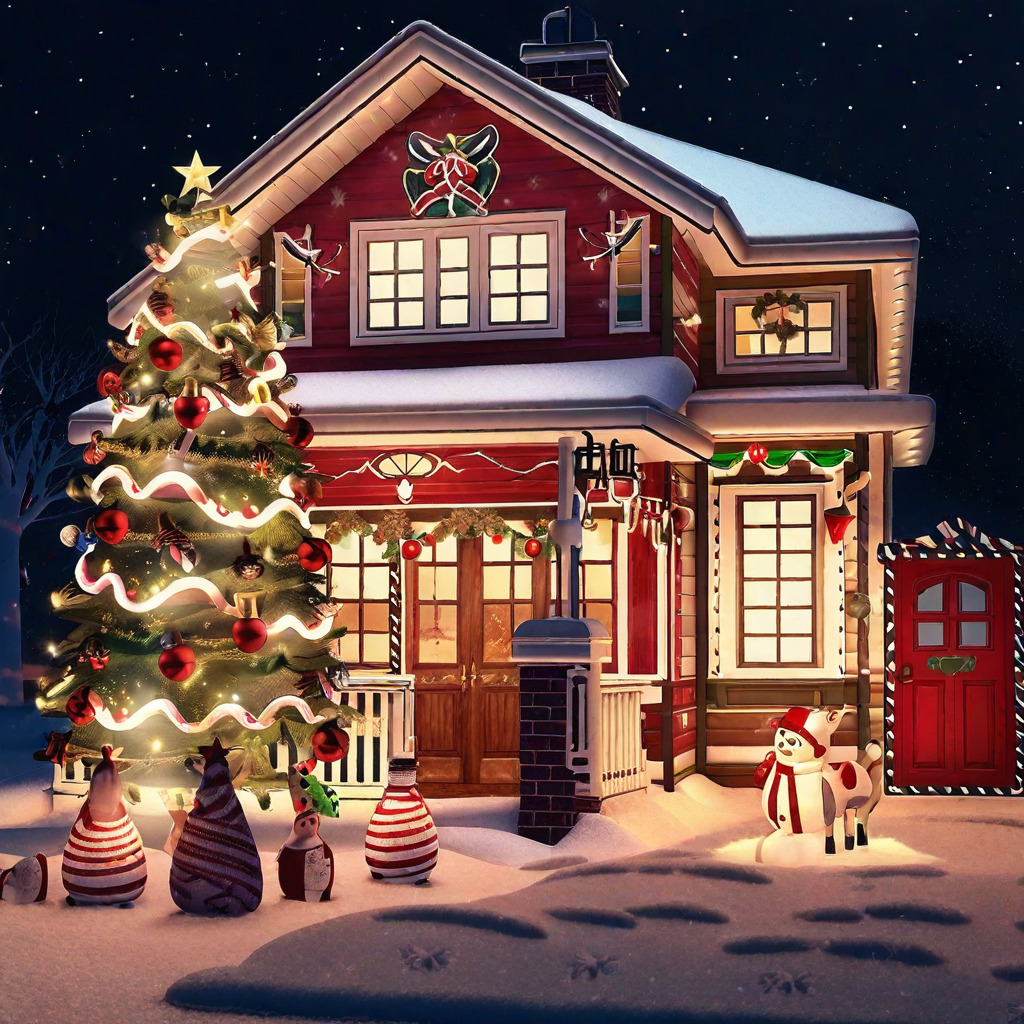}
     \\
    \end{minipage}
\caption{\small 
Analogous to Fig.\,\ref{fig:qualitative_2}, we present a visual comparison of SiD-distilled four-step generators from SDXL using different guidance strategies. SiD with LSG delivered results on par with leading SDXL distillation methods, reaching an FID of 21.04 and a CLIP score of 0.340 while staying free of real image data. Building upon these strong data-free results, SiD-based initialization followed by enhancement with limited real data further improves generation diversity while preserving high visual quality and enhancing prompt-image alignment: All variants \textbf{SiD$_2^a$ (LSG)}, \textbf{SiD$_2^a$ (Zero-CFG)} and \textbf{SiD$_2^a$ (Anti-CFG)} significantly enhances generation diversity compared to baseline SiD in diversity. In particular, SiD$_2^a$ (Anti-CFG) achieving both the highest CLIP score and the lowest FID. All images in our comparison are generated using the same text prompt: \textit{``Christmas decorated house, cinematic wallpaper HD quality.''} }

    \label{fig:qualitative_1}
\end{figure*}
\vspace{5mm}

\begin{algorithm}[th]
\caption{SiD for T2I Generation with multistep generation and data-dependent enhancement }\label{alg:sid}
\begin{algorithmic}[1]
\small
\STATE \textbf{Input:} Pretrained score network $f_\phi$, generator $G_\theta$, generator score network $f_\psi$, guidance scales $\kappa_1=\kappa_2=\kappa_3=0$, $\kappa_4=1$, $t_{\text{init}}=625$, $t_{\min}=20$, $t_{\max}=979$, $\lambda_{\text{sid}}=1$, {$\lambda_{\text{adv},\psi}=10$}, $\lambda_{\text{adv},\theta}=0.001$,  image size $(W,H,C_c)$, latent discriminator map size $(W',H')$, number of generation steps $K$, Final-Step Matching or Uniform-Step Matching, optimizer learning rates $\eta_{\psi}$ and $\eta_\theta$.
\STATE \textbf{Initialization} $\theta \leftarrow \phi, \psi \leftarrow \phi$, $D(\cdot)\leftarrow \text{encoder}(\psi) $ 
\REPEAT
\IF{The number of generation steps is $K = 1$}
    \STATE Sample $\zv \sim \mathcal{N}(0, \mathbf{I})$ and set $\xv_g = G_\theta(t_{\text{init}}, \zv)$
\ELSE
    \IF{Final-Step Matching}
        \STATE \textbf{Final-Step Matching:} Sample $\xv_g$ in $K$ steps using Eq.~\eqref{eq:G_multistep}
    \ELSE
        \STATE \textbf{Uniform-Step Matching:} Sample $k \in \{1, \ldots, K\}$; sample $\xv_g$ in $k$ steps using Eq.~\eqref{eq:G_multistep1}
    \ENDIF
\ENDIF
\STATE Sample $t\in\{t_{\text{min}},\ldots,t_{\text{max}}\}, ~\epsilonv_t\sim \mathcal{N}(0,\mathbf{I})$ and $\xv_t = a_t\xv_g + \sigma_t \epsilonv_t$
\STATE Update $\psi$ with \eqref{eq:obj-psi-sida}:
\STATE ~~~~$\cL_{\psi}^{(\text{sid$^a$})} = {\frac{a_t^2}{\sigma_t^2}} 
\|f_{\psi,\kappa_1}(\xv_t, \cv)-\xv_g\|_2^2
+  (4+\frac{a_t^2}{\sigma_t^2}){\lambda_{\text{adv},\psi}} {L}_{\psi}^{(\text{adv})} = \|\epsilonv_{\psi,\kappa_1}(\xv_t,\cv)-\epsilonv_t\|_2^2 + (4+\frac{a_t^2}{\sigma_t^2} ){\lambda_{\text{adv},\psi}} {L}_{\psi}^{(\text{adv})}$
\STATE ~~~~$\psi = \psi - \eta_\psi \nabla_\psi \cL_{\psi}^{(\text{sid$^a$})}$

\IF {num\_imgs $\ge$ 20K}
    \STATE Sample $t\in\{t_{\text{min}},\ldots,t_{\text{max}}\}$, compute $a_t$, $\sigma_t$, and  $\omega_t$, as defined in \citet{zhou2024long} 
\STATE Set $b=0$ if num\_imgs $\leq$40k and $b=1$ otherwise

        \STATE Update $G_\theta$ with \eqref{eq:obj-theta-sida}:
        \STATE ~~~~$\cL_{\theta}^{(\text{sid$^a$})} = \displaystyle \big(\frac{1}{2}\big)^b \lambda_{\text{sid}} \left( 
        \frac{\omega(t)a_t^2}{{\sigma_t^4} } (f_{\phi,\kappa_4}(\xv_t,\cv)-f_{\psi,\kappa_2}(\xv_t,\cv))^T(f_{\psi,\kappa_3}(\xv_t,\cv)-\xv_g) \right)+ \frac{b}{2}\lambda_{\text{adv},\theta} \frac{\omega(t)a_t^2}{2\sigma_t^4} %
        {C_cWH}%
        \cL_{\theta}^{(\text{adv})}$
        \STATE ~~~~$\theta = \theta - \eta_\theta \nabla_\theta \cL_{\theta}^{(\text{sid$^a$})}$
\ENDIF
\UNTIL{the FID or CLIP plateaus, or the training budget is exhausted}
\STATE \textbf{Output:} $G_\theta$ optimized for $K$-step inference.
\normalsize
\end{algorithmic}
\end{algorithm} %

\section{Broader Impact}\label{sec:broaderimpact}
Our work contributes to the sustainability of generative AI by reducing the computational cost of sampling from state-of-the-art diffusion models. At the same time, the efficiency and scalability of our distillation framework may lower the barrier to deploying models trained on uncurated or sensitive data, raising potential misuse concerns. Addressing these risks will require continued efforts by the research community to establish clear ethical guidelines and practical safeguards for responsible  deployment of generative AI models.

\section{Related Work}\label{sec:relatedwork}
\citet{thuan2024swiftbrush} and \citet{nguyen2024snoopi} extended KL-based distillation \citep{poole2023dreamfusion, luo2023diffinstruct, wang2023prolificdreamer} to distill pre-trained Stable Diffusion models. DMD \citep{yin2023onestepDW} also followed a KL-based approach but introduces an additional regression loss, while DMD2 \citep{yin2024improved} reduced reliance on this term by incorporating adversarial loss and separate time scaling during distillation. Both UFOGen \citep{Xu2023UFOGenYF}  and adversarial diffusion distillation  \citep{Sauer2023AdversarialDD}  introduced Diffusion GANs \citep{wang2022diffusion,xiao2021tackling} into diffusion distillation. 
\citet{lin2024sdxl} further extended adversarial diffusion distillation with progressive distillation \citep{salimans2022progressive} to develop SDXL-Lightning.  

\citet{Luo2023LatentCM} applied consistency distillation \citep{song2023consistency} to text-guided latent diffusion models. \citet{kang2024diffusion2gan} leveraged the ODE trajectory of a teacher diffusion model to generate noise-to-image pairs, which were then used to supervise the training of a one-step generator. To further enhance distillation performance, they also incorporated techniques such as mix-and-match augmentation.  Rectified flow~\citep{liu2022rectified,yan2024perflow}, which leverages matched noise-to-image pairs generated by the ODE solver of a pretrained flow-matching model, is another popular approach for one- or few-step T2I generation. \citet{wang2025rectified} demonstrated that the effectiveness of rectification in rectified flow primarily stems from using these matched pairs generated by the ODE solver of a pretrained diffusion model—without requiring the model to adhere to the linear interpolation forms of flow-matching or \(v\)-prediction, which were previously thought to be essential for rectification. Nonetheless, like rectified flow, its performance remains constrained by the capacity of the teacher model’s ODE solver, albeit to a lesser degree.
Another promising and complementary direction for enhancing one-step diffusion distillation is employing a mixture of multiple same-sized or smaller generators \citep{song2024multi}, which has shown clear improvements in both performance and efficiency when combined with DMD.    %

SwiftBrushV2 \citep{dao2024swiftbrush} extends SwiftBrush by initializing its generator with a one-step model pre-distilled using SD-Turbo \citep{Sauer2023AdversarialDD}, and further improves performance through preference optimization and post-training techniques. Hyper-SD \citep{ren2024hypersd} applies consistency distillation \citep{song2023consistency,kim2023consistency} and progressive distillation to segmented ODE trajectories and incorporates human feedback–based reward signals to enhance generation results. NitroFusion \citep{chen2024nitrofusion} introduces a pool of specialized discriminator heads to distill high-quality 1–4 step generators, building on few-step teacher models pre-distilled by earlier methods \citep{yin2024improved,ren2024hypersd}. These post-training or joint distillation-and-refinement approaches—whether through initialization or fine-tuning with pre-distilled one- or few-step teachers—could benefit from our work, which offers stronger one- and few-step generators as more effective starting points prior to post-training or preference optimization.

CFG \citep{ho2022classifier} is widely used in T2I diffusion distillation to promote high-fidelity image generation. However, existing methods often give limited attention to the trade-off between text-image alignment and generation diversity. In this paper, we directly tackle this challenge and propose a novel approach that achieves state-of-the-art results in both FID and CLIP scores.  While various preference optimization and post-training techniques could further enhance performance, they are beyond the scope of this work. Our focus is on improving the distillation process from the teacher model—optionally leveraging the real data used to train the teacher—without introducing additional loss terms that may be directly tied to or confounded with the evaluation metrics.

\section{SiD Enhanced with Limited Real Images}
\label{sec: diffgan}

\begin{proof}[Proof of Lemma~\ref{lemma:dropk}]
If the pretrained teacher score estimation network \( S_{\phi} \) has reached its theoretical
optimum such that \( S_{\phi}(\xv_t,\cv) = \nabla_{\xv_t} \ln p_\text{data}(\xv_t \mid \cv) \), then
the Fisher divergence at step \( k \), as defined in \eqref{eq:MESM1}, is zero if and only if \( p_\theta(\xv_g^{(k)} \mid \cv) = p_\text{data}(\xv_g^{(k)} \mid \cv) \). Therefore, the Fisher divergence objectives at all \( K \) steps share the same optimal solution: \( p_\theta(\cdot \mid \cv) = p_\text{data}(\cdot \mid \cv) \).
\end{proof}

Following how the adversarial components are added into SiD to enhance both its distillation efficiency and one-step performance for unconditional and label-conditional models \citep{zhou2025adversarial}, 
we  explore the integration of adversarial learning into mutlistep SiD enhanced with CFG.

Taking a batch of paired text $\cv$ and real noisy images $\yv_t$, where $\yv_t$  has dimensions $(W,H,C_c)$, and adding noisy fake images 
$\xv_t$ generated as in \eqref{eq:xv_t_g} for these text prompts,  
we define the fake score network's loss %
aggregated over a GPU batch $\mathcal B$ 
as
\ba{
{L}_{\psi}^{(\text{sid$^a$})}= %
\sum_{(\cv,\yv_t,\xv_t)\in \mathcal B}\left[
\mbox{SNR}_t L_{\psi,t,\kappa_{1}} %
+  {\gamma_t}{\lambda_{\text{adv},\psi}}{L}_{\psi}^{(\text{adv})}\right],
\label{eq:obj-psi-sida}
}
where $\mbox{SNR}_t=\frac{a_t^2}{\sigma_t^2}$ denotes the signal-to-noise ratio (SNR) at time $t$, $\gamma_t$ is a weighting coefficient,  $L_{\psi,t,\kappa_{1}}^{\text{(sid-lsg)}}$  is the data-free SiD-LSG loss defined as in %
\eqref{eq:obj-theta_lsg}, $\lambda_{\text{adv},\psi}$ is used to balance the loss scales of both terms, and ${L}_{\psi}^{(\text{adv})}$ is the Diffusion GAN's discriminator loss defined as
\ba{
{L}_{\psi}^{(\text{adv})}
&=%
\frac{1}{2|\mathcal B|W'H' }\sum_{(\cv,\yv_t,\xv_t)\in \mathcal B}\sum_{i'=1}^{W'}\sum_{j'=1}^{H'} %
\left[\ln ( D(\cv,\yv_t)[i',j'])+\ln(1- D(\cv,\xv_t)[i',j'])\right],
}
where %
$|\mathcal B|$ is the per GPU batch size, $D(\cv,\xv_t)$ the 2D discriminator map of size $(W',H')$  given input $(\cv,\xv_t)$, and $D(\cv,\xv_t)[i,j]$ the $(i,j)th$ element. %
We set $\lambda_{\text{adv},\psi}=10$, 
unless specified otherwise.

Notably, the conditioning variable \( \cv \) used to generate fake images does not need to match the prompts associated with real images. We observed no significant difference between these settings during the early stages of training. However, when the number of real text-image pairs is limited, allowing prompts that are not directly tied to the observed real images can be beneficial. In our experiments, we are constrained to 480k real text-image pairs, so there is potential value in allowing the \( \cv \) used for training the fake score network and/or the generator to differ from those associated with the real data. Given that we have significantly more text prompts than paired text-image data, we default to allowing fake images to be generated using prompts not drawn from the observed real pairs, while permitting variation across experimental runs.

The  weighting coefficient  $\gamma_t$ in \eqref{eq:obj-psi-sida} balances the loss of SiD and adversarial regularization.
 We   set  it as \( \gamma_t = 4 + \text{SNR}_t \) to 
 help ensure that timesteps with low \( \text{SNR}_t \) sufficiently contribute to learning~\( D \). By default, our experiments were conducted under this setting. We also performed ablation studies using both \( \gamma_t = \text{SNR}_t \) and \( \gamma_t = 10\text{SNR}_t \), and observed that both configurations delivered comparable performance, with no significant differences.  %

We now define the generator loss of SiD$^a$, aggregated over a GPU batch, as follows:
\ba{
 \cL_{\theta}^{(\text{sid$^a$})}\!=\!\sum_{\xv_t\in \mathcal B} \frac{\omega(t)a_t^2}{2\sigma_t^4} \left(\lambda_{\text{sid}}\cL_{\theta,t,\kappa_{2:4}}%
 + \lambda_{\text{adv},\theta}{C_cWH}%
\cL_{\theta}^{(\text{adv})}\right),%
\label{eq:obj-theta-sida}
}
where  we set $\lambda_{\text{sid}}=1$ and $\lambda_{\text{adv},\theta}=0.001$ by default, unless specified otherwise, and denote  the adversarial loss that measures the average in GPU-batch fakeness as
\ba{
\cL_{\theta}^{(\text{adv})}
=%
\frac1{|\mathcal B|W'H' } {\sum_{(\cv,\xv_t)\in \mathcal B}\sum_{i=1}^{W'}\sum_{j=1}^{H'} \ln D(\cv,\xv_t)[i,j]}.
}
Intuitively speaking, while the  loss component of SiD is matching the noisy model and noisy data distributions by comparing the estimated score and pretrained scores, the  adversarial loss component is matching them by comparing the real and fake data distributions in the latent space of an encoder  that is also roughly the first half part of the fake score estimation network. From another perspective, the discrepancy between the teacher network $\phi$ and its optimal value $\phi^*$, which limits the ultimate potential of the  SiD loss, is compensated for by the Diffusion-GAN loss; the Diffusion-GAN loss by itself mitigates but still suffer from the usual stability issue of GAN training, is well compensated by the regularization provided by the SiD loss. %

The remaining question is whether SiD$^a$ can effectively address the trade-off between high CFG—which enhances text-image alignment and generation fidelity—and low CFG—which promotes generation diversity—without relying on additional loss terms that may be directly tied to or confounded with evaluation metrics. Minimizing such losses risks overfitting to the metrics themselves, potentially leading to reward hacking and obscuring genuine improvements. Therefore, it is important to first evaluate the effectiveness of SiD without these additional losses, before exploring post-training or joint distillation with preference optimization, which we leave as future work. Through extensive empirical studies, we find that while existing LSG guidance strategies developed for  SiD remain effective, there is still room for improvement. This motivates the development of Zero-CFG and Anti-CFG, two new tuning-free guidance strategies, as outlined in Section~\ref{sec:3.2}, that both achieve good balance between generation diversity and text-image alignment when applied to SiD enhanced with limited real data.  %

\section{Training and evaluation details} \label{sec:detail}

\subsection{Implementation-level Optimizations}\label{sec:implementation}

 SiD \citep{zhou2024long} has been successfully applied to distill SD1.5, achieving under one of its LSG settings an FID of 8.15—the lowest among prior data-free one-step generators. Notably, this outperforms the FID of 8.35 achieved by DMD2 \citep{yin2024improved}, which requires access to real data for distillation. However, SiD's reliance on FP32 training limits its efficiency: while FP16-based training can be over six times faster with significantly lower memory consumption, it leads to notable performance degradation.
For SDXL distillation, our implementation supports FP16/BF16 training, producing reasonable but suboptimal models. However, running SDXL under FP32 on 80GB GPUs results in out-of-memory (OOM) errors, underscoring the need for more efficient optimization strategies.

\textbf{AMP for Efficient and Stable Training.} 
To address these issues, we first reimplement SiD-LSG with Automatic Mixed Precision (AMP)\footnote{\href{https://pytorch.org/docs/stable/amp.html}{https://pytorch.org/docs/stable/amp.html}} in PyTorch, which keeps model weights in FP32 but casts forward computations to FP16/BF16 where appropriate. As shown in Table\,\ref{tab:Hyperparameters}, this significantly reduces memory usage, allowing for larger per-GPU batch sizes and enabling SiD to train on SD1.5 with 1 million fake images in three to five hours, depending on the LSG settings, instead of one day on a node with eight H100 80GB GPUs. Additionally, we observe no noticeable performance difference between FP32 and AMP.

While AMP with FP16 requires gradient scaling before backpropagation, it is generally faster than BF16, which does not require scaling. Since we observe no clear performance difference between AMP-FP16 and AMP-BF16 in the early phases of training, we adopt AMP-FP16 for its superior computational efficiency.

\textbf{AMP + FSDP for %
Scalable Training.}  Utilizing AMP alone is not sufficient to address the OOM issue with SDXL distillation under FP32, and thus we extend our optimizations by integrating AMP and Fully Sharded Data Parallel (FSDP)\footnote{\href{https://pytorch.org/docs/stable/fsdp.html}{https://pytorch.org/docs/stable/fsdp.html}}, which shards model weights and optimizer states across GPUs, and turn on  gradient checkpointing. 
By doing this, we are able to run SDXL under FP32 with AMP, in a node of eight H100 GPUs, and eliminate numerical instability issues that hinder convergence, achieving SoTA performance without requiring a number of regularization tricks commonly required in previous methods. 
Comparing to using FP16/BF16 as a compromising solution, these optimizations with AMP and FSDP are crucial for fully unlocking the potential of SiD in SDXL distillation. %
Our code relies on native PyTorch libraries, avoiding third-party containers to ensure maximum flexibility for future code customization and extension.

\textbf{Multistep Generation. }
For final-step matching that adopts the strategy of backprogagation through the entire generation chain of the multistep generator, we empirically mitigate its optimization challenge by reducing both the learning rate of \( G \) and batch size proportionally to \( 1/K \).  For uniform-step matching, we randomly select \( k\in\{1,\ldots,K\} \) independently for each GPU mini-batch when using DDP, whereas under FSDP, the same \( k \) is randomly selected and shared across all GPUs for each iteration. For this reason, for uniform-step matching with FSDP, we reduce the batch size of each iteration from 512 to 128 %
to enable more frequent random switching between different \( k \)~values.

\subsection{Memory and Speed. } \label{sec:memory}

The hyperparameters specific to our study are detailed in Table\,\ref{tab:Hyperparameters}. It is important to note that the time and memory costs reported in Table\,\ref{tab:Hyperparameters} exclude the overhead associated with periodic evaluations of the single-step generator's FID and CLIP scores, as well as the resources required for saving model checkpoints during distillation. These additional costs can vary significantly depending on factors such as the versions of CUDA and Flash Attention, the storage platforms utilized, the frequency of evaluation and checkpointing operations, and the overhead incurred when restarting interrupted jobs from the most recent checkpoints.

Table\,\ref{tab:Hyperparameters} provides a detailed analysis of the computational resources required for the one-step SiD distillation process under different CFG strategies—CFG, LSG, and the newly proposed Zero-CFG—using eight H100 GPUs with 80GB memory as a reference. Table\,\ref{tab:Hyperparameters_4step} further provides a detailed analysis of the computational resources required for the four-step SiD distillation process under different CFG strategies—CFG, LSG, Zero-CFG, and Anti-CFG.

\begin{table*}[!bh]
\caption{\small
Hyperparameter settings and a comparison of distillation time and memory usage across different guidance strategies for distilling a one-step generator with SiD. DDP refers to Distributed Data Parallel, and FSDP to Fully Sharded Data Parallel. AMP denotes Automatic Mixed Precision, which performs computations in FP16 while maintaining model weights in FP32 for stability. In contrast, FP16 refers to full FP16 precision for both computations and weights, with the loss cast to FP32 before backpropagation.
}

\label{tab:Hyperparameters}
\begin{center}
\resizebox{0.85\textwidth}{!}{
\begin{tabular}{ccccc}
\toprule
Computing platform & Hyperparameters & CFG & LSG & Zero-CFG \\
\midrule
\multirow{9}{*}{General Settings} & \multirow{2}{*}{CFG scales} & $\kappa_4>1$ & $\kappa_{1,2,3,4}>1$ & $\kappa_4=1$ \\
& & $\kappa_1=\kappa_2=\kappa_3=1$ & $\kappa_1=\kappa_2=\kappa_3=\kappa_4$ & $\kappa_1=\kappa_2=\kappa_3=0$ \\
& Batch size & \multicolumn{3}{c}{512} \\
& Learning rate & \multicolumn{3}{c}{1e-6} \\
& Half-life of EMA & \multicolumn{3}{c}{Either 100k images or 0 ($i.e.$, no EMA)} \\
& Optimizer under FP32 or AMP & \multicolumn{3}{c}{Adam ($\beta_1=0$, $\beta_2=0.999$, $\epsilon=1\text{e-8}$)} \\
& Optimizer under FP16 & \multicolumn{3}{c}{Adam ($\beta_1=0$, $\beta_2=0.999$, $\epsilon=1\text{e-6}$)} \\
& $\alpha$ & \multicolumn{3}{c}{1} \\
& $\lambda_{\text{sid}}$ & \multicolumn{3}{c}{1} \\
& $\lambda_{\text{adv}}$ & \multicolumn{3}{c}{10} \\
& $\lambda_{\text{adv},\theta}$ & \multicolumn{3}{c}{0.001} \\
& Time parameters & \multicolumn{3}{c}{$(t_{\min}, t_{\text{init}}, t_{\max}) = (20, 625, 979)$} \\
\midrule
& Teacher model & \multicolumn{3}{c}{Stable Diffusion 1.5} \\
& \# of GPUs  & \multicolumn{3}{c}{8xH100 (80G)} \\
& Batch size per GPU & 8 & 4 & 8 \\
SiD$^a$ & \# of gradient accumulation round & 8 & 16 & 8 \\
AMP+DDP & Max memory in GB allocated & 73.9 & 61.2 & 73.7 \\
& Max memory in GB reserved & 74.8 & 61.9 & 74.3 \\
& Time in seconds per 1k images & 11 & 18 & 12 \\
& Time in hours per 1M images & 3 & 5 & 3.4 \\
\midrule
& Teacher model & \multicolumn{3}{c}{Stable Diffusion XL} \\
& \# of GPUs  & \multicolumn{3}{c}{8xH100 (80G)} \\
& Batch size per GPU & 2 & 2 & 4 \\
SiD & \# of gradient accumulation round & 32 & 32 & 16 \\
AMP+FSDP & Max memory in GB allocated & 55.6 & 56.7 & 58.0 \\
& Max memory in GB reserved & 76.9 & 77.2 & 77.3 \\
& Time in seconds per 1k images & 138 & 153 & 83 \\
& Time in hours per 1M images & 38 & 43 & 23 \\
\midrule
& Teacher model & \multicolumn{3}{c}{Stable Diffusion XL} \\
& \# of GPUs  & \multicolumn{3}{c}{8xH100 (80G)} \\
& Batch size per GPU & 2 & 2 & 2 \\
SiD$^a$ & \# of gradient accumulation round & 32 & 32 & 32 \\
FP16+DDP & Max memory in GB allocated & 72.2 & 73.1 & 63.1 \\
& Max memory in GB reserved & 73.8 & 74.8 & 64.5\\
& Time in seconds per 1k images & 86 &101  & 96\\
& Time in hours per 1M images &  24 &28  & 27 \\
\midrule
& Teacher model & \multicolumn{3}{c}{Stable Diffusion XL} \\
& \# of GPUs  & \multicolumn{3}{c}{8xH100 (80G)} \\
& Batch size per GPU & 2 & 2 & 4 \\
SiD$^a$ & \# of gradient accumulation round & 32 & 32 & 16 \\
AMP+FSDP & Max memory in GB allocated & 55.6 & 56.7 & 60.1 \\
& Max memory in GB reserved & 77.2 & 77.2 & 77.3 \\
& Time in seconds per 1k images & 160 & 170 & 117 \\
& Time in hours per 1M images & 44 & 47 & 33 \\
\bottomrule
\end{tabular}
\vspace{-7mm}
}
\end{center}
\end{table*}

\begin{table*}[!bh]
\caption{\small Analogous to Table~\ref{tab:Hyperparameters}, this table compares distillation time and memory usage for training a four-step generator using SiD enhanced with real data. Note that Zero-CFG shows higher memory usage primarily because its efficiency allows for double the batch size per GPU.}

\label{tab:Hyperparameters_4step}
\begin{center}
\resizebox{0.85\textwidth}{!}{
\begin{tabular}{cccccc}
\toprule
Computing platform & Hyperparameters & ~~~~CFG~~~~  & ~~~LSG~~~~ & Zero-CFG & Anti-CFG \\
\midrule
\multirow{9}{*}{General Settings} 
& Learning rate & \multicolumn{4}{c}{1e-6} \\
& Optimizer  & \multicolumn{4}{c}{Adam ($\beta_1=0$, $\beta_2=0.999$, $\epsilon=1\text{e-8}$)} \\
& $\alpha$ & \multicolumn{4}{c}{1} \\
& $\lambda_{\text{sid}}$ & \multicolumn{4}{c}{1} \\
& $\lambda_{\text{adv}}$ & \multicolumn{4}{c}{10} \\
& $\lambda_{\text{adv},\theta}$ & \multicolumn{4}{c}{0.001} \\
& Time parameters & \multicolumn{4}{c}{$(t_{\min}, t_{\text{init}}, t_{\max}) = (20, 625, 979)$} \\
\midrule
& Teacher model & \multicolumn{4}{c}{Stable Diffusion 1.5} \\
& \# of GPUs  & \multicolumn{4}{c}{8xH100 (80G)} \\
& Batch size & \multicolumn{4}{c}{256}  \\
& Batch size per GPU&\multicolumn{4}{c}{4}\\
SiD$_2^a$ & \# of gradient accumulation round & \multicolumn{4}{c}{8} \\
AMP+DDP  & Max memory in GB allocated & &61.3 &57.2&57.3\\
& Max memory in GB reserved & &61.9&57.8&57.8\\
& Time in seconds per 1k images & &21&20&20 \\
& Time in hours per 1M images & &5.8 &5.6 & 5.6\\
\midrule
& Teacher model & \multicolumn{4}{c}{Stable Diffusion XL} \\
& \# of GPUs  & \multicolumn{4}{c}{8xH100 (80G)} \\
& Batch size & \multicolumn{4}{c}{128}  \\
& Batch size per GPU & 2&2&4& 2\\
SiD$_2^a$ & \# of gradient accumulation round & 8&8&4& 8 \\
AMP+FSDP & Max memory in GB allocated & 55.6&56.7&60.1& 44.3\\
& Max memory in GB reserved & 76.7&77.2&77.3& 66.4\\
& Time in seconds per 1k images & 200&220&165& 190 \\
& Time in hours per 1M images & 56&61&46&53\\
\bottomrule
\end{tabular}
\vspace{-7mm}
}
\end{center}
\end{table*}

\clearpage

\section{Sudo Code for Custom UNet for  SiD$^a$}\label{app:code}

\begin{lstlisting}[language=Python, caption={\texttt{Custom UNet for guided SiDA}}, label={lst:original-mpconv},basicstyle=\scriptsize\ttfamily]
class UNet2DConditionOutputWithEncoder(BaseOutput):
    encoder_output: torch.Tensor =None
    sample: torch.Tensor = None

# Custom UNet class for guided SiDA
class UNet2DConditionModelWithEncoder(UNet2DConditionModel):
    @classmethod
    def from_pretrained(cls, pretrained_model_name_or_path, **kwargs):
        # Load the pre-trained U-Net model using the parent class's from_pretrained method
        unet = super().from_pretrained(pretrained_model_name_or_path, **kwargs)
        # Change the class of the loaded model to CustomUNet
        unet.__class__ = cls
        return unet

    
    def forward(
        self,
        sample: torch.Tensor,
        timestep: Union[torch.Tensor, float, int],
        encoder_hidden_states: torch.Tensor,
        class_labels: Optional[torch.Tensor] = None,
        timestep_cond: Optional[torch.Tensor] = None,
        attention_mask: Optional[torch.Tensor] = None,
        cross_attention_kwargs: Optional[Dict[str, Any]] = None,
        added_cond_kwargs: Optional[Dict[str, torch.Tensor]] = None,
        down_block_additional_residuals: Optional[Tuple[torch.Tensor]] = None,
        mid_block_additional_residual: Optional[torch.Tensor] = None,
        down_intrablock_additional_residuals: Optional[Tuple[torch.Tensor]] = None,
        encoder_attention_mask: Optional[torch.Tensor] = None,
        return_dict: bool = True,
        return_flag: str = 'decoder' # extra argument
    ) -> Union[UNet2DConditionOutputWithEncoder, Tuple]:
        

        assert return_flag in ['encoder', 'decoder', 'encoder_decoder'], f"Invalid return_flag: {return_flag}"
        
        .... #this part is the same as UNet2DConditionModel
        
        #inserted between mid_block and upsample_block:

        encoder_output = sample.mean(dim=1, keepdim=True)  
    
        if return_flag == 'encoder':
            if not return_dict:
                return (encoder_output,)
            else:
                return UNet2DConditionOutputWithEncoder(encoder_output = encoder_output)

        .... #this part is the same as UNet2DConditionModel

        #insert at the end:

        if return_flag == 'encoder_decoder':
            if not return_dict:
                return (sample,encoder_output)
            else:
                return UNet2DConditionOutputWithEncoder(sample=sample,encoder_output = encoder_output)
        else:
            if not return_dict:
                return (sample,)
            return UNet2DConditionOutputWithEncoder(sample=sample)
\end{lstlisting}

\section{Distillation of SD1.5 under Multistep SiD}\label{sec:SD1.5}

We first explore a variety of settings to compare the performance of SiD on distilling SD1.5 under different guidance strategies and hyperparameter configurations. Fig.\,\ref{fig:cfgfree} visualizes the one-step generation performance across these settings.

\textbf{No CFG.} 
The first subplot of Fig.\,\ref{fig:cfgfree} shows that SiD exhibits similar convergence behavior under FP32 and AMP, with no significant differences. While SiD's data-free setting performs poorly without CFG, its performance improves dramatically once the Diffusion GAN loss is incorporated, transforming it into SiDA. This enhancement yields an FID of approximately 10 and a CLIP score around 0.30, making it highly competitive with prior methods—even without the use of CFG.

\textbf{LSG with Low Guidance Scale.} 
Next, we compare SiD and SiD$^a$ under the LSG setting with \(\kappa_1 = \kappa_2 = \kappa_3 = \kappa_4 = \kappa\), where \(\kappa \in \{1.5, 4.5\}\). The second subplot of Fig.\,\ref{fig:cfgfree} shows that for SiD, which already achieves a SoTA FID score, SiD$^a$ can either slightly lower the FID or enhance the CLIP to some extent, though at the cost of a higher FID. This indicates that adding the Diffusion GAN loss provides limited benefits to a SiD model that already achieves a low FID but also has a low CLIP. Despite these marginal improvements, the model checkpoint corresponding to the ``SiDA, AMP, \(\kappa=1.5\)'' red curve—trained with \((\lambda_{\text{sid}}, \lambda_{\text{adv}}, \lambda_{\text{adv},\theta}) = (1,1,0.0001)\)—achieves an FID as low as \textbf{7.89} while maintaining a CLIP score of 0.304.

\textbf{LSG with High Guidance Scale.} 
 The third subplot of Fig.\,\ref{fig:cfgfree} shows that for SiD models achieving a SoTA CLIP score but high FID, implementing SiD$^a$ with AMP further improves both CLIP and FID. 
 However, while SiD$^a$ under FP16 can still reduce FID, it fails to fully realize the algorithm’s potential due to numerical precision limitations, making the model more susceptible to underflow and overflow during gradient computation. This necessitates careful loss scaling, which is challenging to fine-tune in practice. Thus, we recommend using AMP when applying SiD or SiD$^a$ for distilling SD models. If using AMP is not feasible under Distributed Data Parallel (DDP) due to memory constraints, we recommend employing it with FSDP, the approach we adopt to scale SiD for SDXL.

\textbf{CFG and Anti-CFG. }
The fourth subplot of Fig.\,\ref{fig:cfgfree} presents results for the conventional CFG strategy, which enhances text guidance only for the pretrained teacher network with \(\kappa_1=\kappa_2=\kappa_3=1\) and \(\kappa_4=7.5\). While large CFG values under SiD yield satisfactory CLIP scores, they result in relatively poor FID. Converting SiD to SiDA instantly improves both CLIP and FID.  

Additionally, this subplot includes results for the proposed Zero-CFG strategy under both SiD and SiD$^a$. Notably, SiD-Zero-CFG achieves better CLIP than the conventional CFG strategy but exhibits worse FID. However, transitioning to SiD$^a$-Zero-CFG significantly enhances both CLIP and FID, making it a strong candidate for the best SiD$^a$ setting in terms of maximizing CLIP while lowering FID. Moreover, as shown in Table\,\ref{tab:Hyperparameters}, this guidance-scale tuning-free strategy offers good memory and computational efficiency, particularly for SDXL implemented under AMP+FSDP, which we discuss~next.

\textbf{Ablation Study of \( \gamma_t \).}  
We investigated different settings for the weighting parameter \( \gamma_t \) in \eqref{eq:obj-psi-sida}, including \( \gamma_t = 4 + \text{SNR}_t \), \( \gamma_t = \text{SNR}_t \), \( \gamma_t = 10 \times \text{SNR}_t \), and \( \gamma_t = 100 \times \text{SNR}_t \). The ablation study, illustrated in the last two subplots of Fig.\,\ref{fig:cfgfree}, examines both the LSG setting with \( \kappa_1 = \kappa_2 = \kappa_3 = \kappa_4 = 4.5 \) and the Zero-CFG setting. The results indicate that most configurations deliver comparable performance, except for \( \gamma_t = 100 \times \text{SNR}_t \), which exhibits clearly slower convergence. We further evaluated FID and CLIP using the settings reported in Table\,\ref{tab:comparison} and found the performance differences to be small. Specifically, SiDA-LSG (\(\kappa=4.5\)) achieves FID 10.61 and CLIP 0.321 with \( \gamma_t = 4 + \text{SNR}_t \), and FID 10.10 and CLIP 0.321 with \( \gamma_t = 10\text{SNR}_t \). %
Similarly, SiD$^a$-Zero-CFG achieves %
FID 9.63 and CLIP	0.321
with \( \gamma_t = 4+\text{SNR}_t \), and FID 10.28 and CLIP 0.322 with \( \gamma_t = 10\text{SNR}_t \).  This suggests that the algorithm is robust to a wide range of \( \gamma_t \) choices, providing flexibility in its selection.

\textbf{Multistep Generation: Final-Step Matching and Uniform-Step Matching. } We explore two training strategies described in Section \ref{sec:multistep} for multistep generation: Final-Step Matching, where gradients are backpropagated through the entire generation chain, and Uniform-Step Matching, where stop-gradient operations are applied to isolate updates to individual generation steps. The results are summarized in Fig.\,\ref{fig:cfgfree_multistep}. With final-step matching, we observe lower FID scores, indicating improved generation diversity, but at the cost of reduced CLIP scores. We suspect this trade-off arises because the generator becomes more expressive yet harder to optimize and control. In contrast, uniform-step matching maintains comparable CLIP scores  while showing similar trends of FID decreasing as the number of generator steps increases. Though 
the improvements in FID and CLIPs are not clear,  visual inspection reveals noticeable improvement in image quality—particularly in high-frequency details—when increasing from one to two steps. However, the gains from two to four steps appear more subtle.

\begin{figure*}[t]
\centering
\includegraphics[width=.32\linewidth]{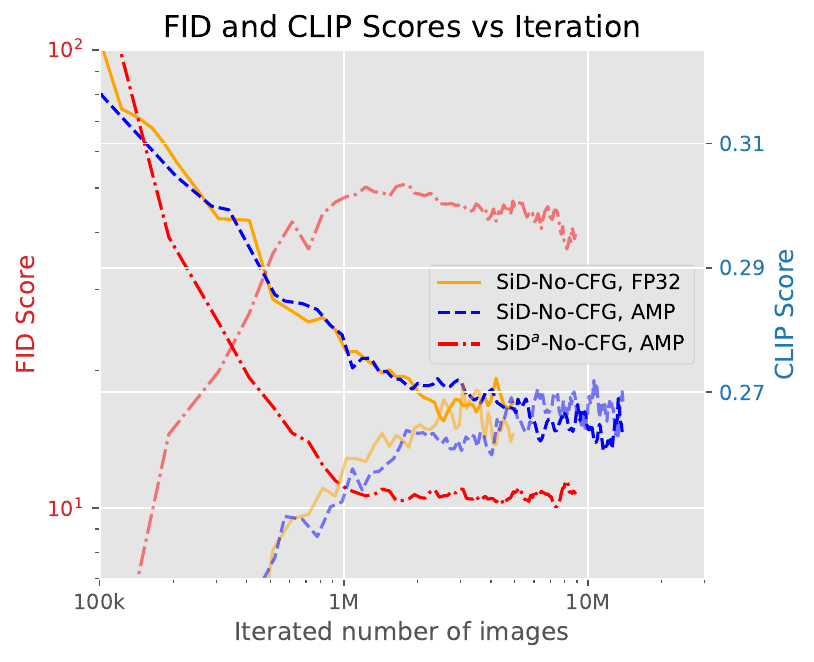}
\includegraphics[width=.32\linewidth]{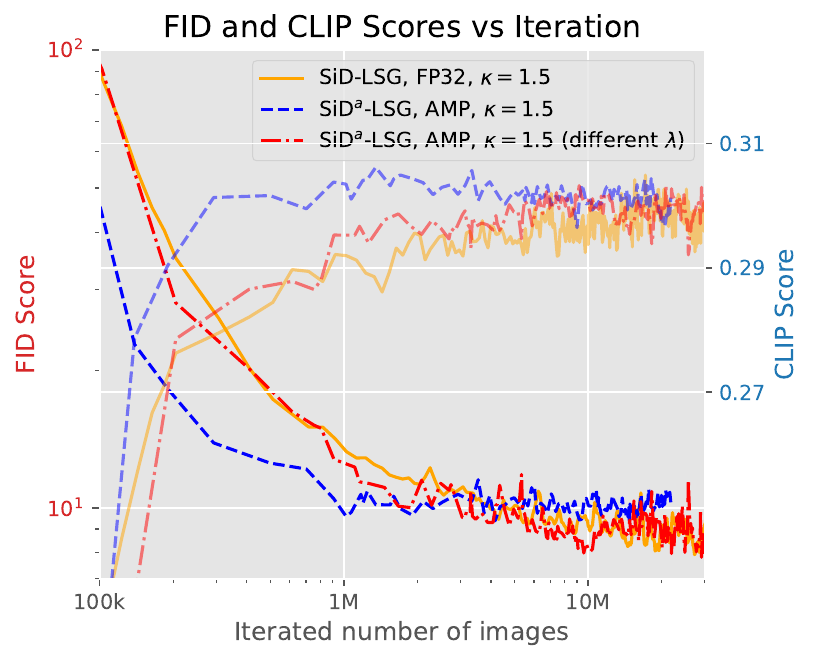}
\includegraphics[width=.32\linewidth]{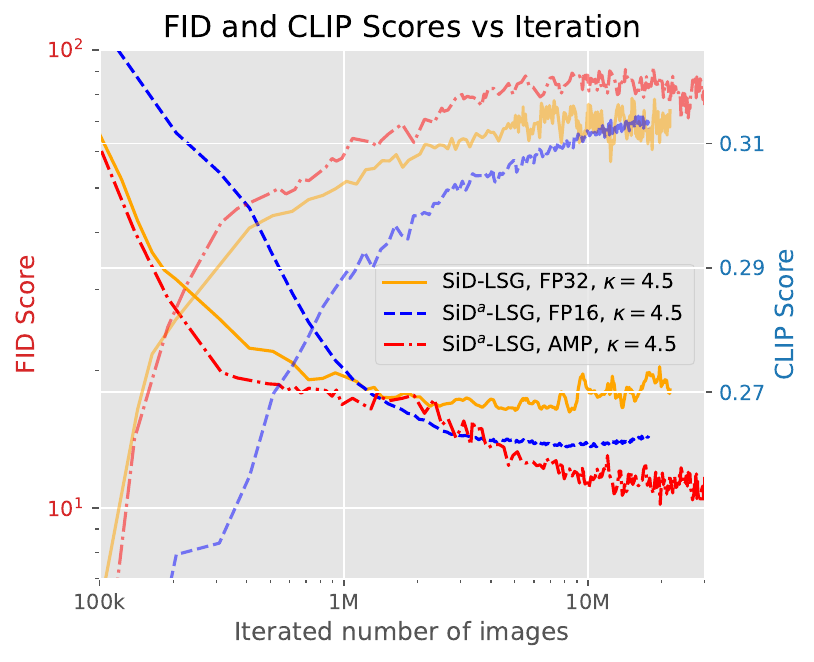}
\includegraphics[width=.32\linewidth]{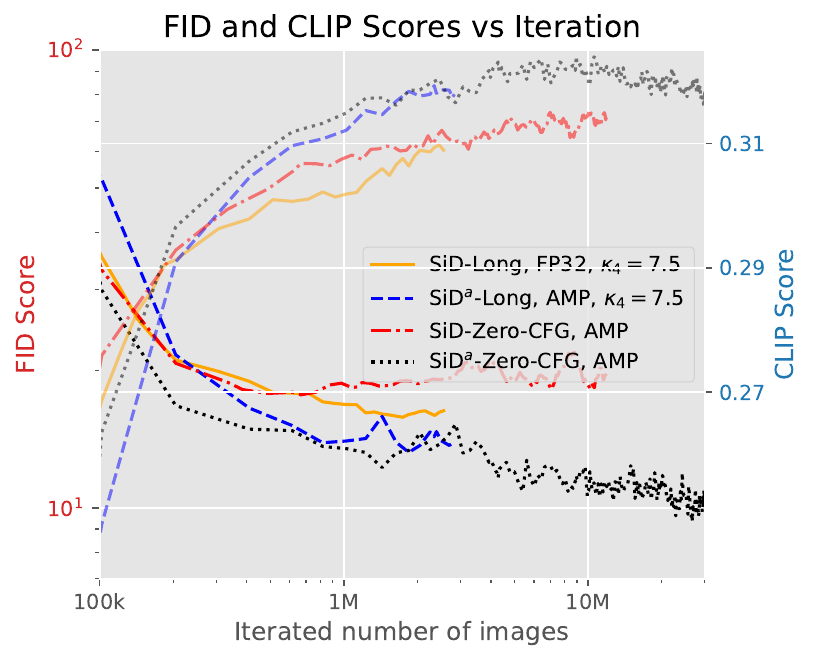}
\includegraphics[width=.32\linewidth]{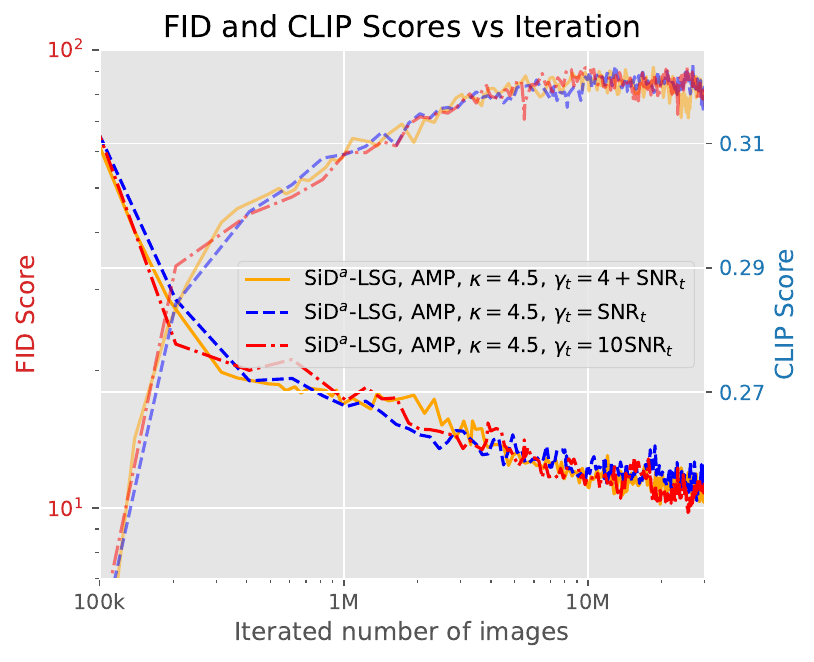}
\includegraphics[width=.32\linewidth]{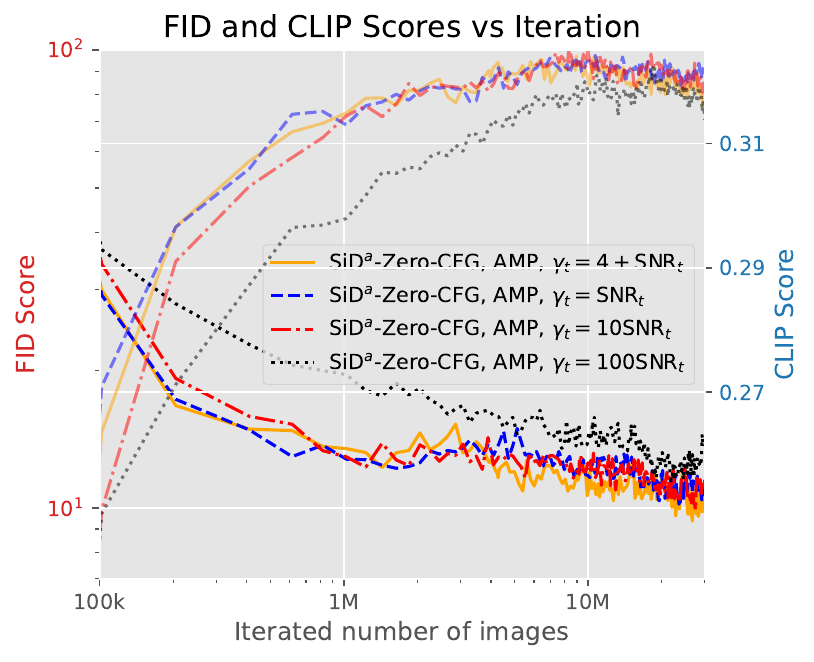}
\caption{\small 
Comparison of SiD and SiD$^a$,  trained with and without real images respectively,  for distilling Stable Diffusion 1.5 into a one-step generator, evaluated by FID and CLIP, across different numerical precision settings (FP32, AMP, and FP16), guidance strategies, and \( \gamma_t \) configurations, where \( \gamma_t = 4 + \text{SNR}_t \) unless stated otherwise.   
Guidance strategies include No CFG (\(\kappa_1=\kappa_2=\kappa_3=\kappa_4=1\)), shown in the first subplot; LSG (\(\kappa_1=\kappa_2=\kappa_3=\kappa_4=\kappa\), where \(\kappa \in \{1.5,4.5\}\)), shown in the second, third, and fifth subplots; long guidance (\(\kappa_1=\kappa_2=\kappa_3=1\), \(\kappa_4=7.5\)); and the newly proposed Zero-CFG  (\(\kappa_1=\kappa_2=\kappa_3=0\), \(\kappa_4=1\)), shown in the fourth and six subplots. The plot illustrates the relationship between FID and CLIP scores across different strategies, with FID on the primary y-axis and CLIP scores on the secondary y-axis. Different line styles indicate different settings, while dense-colored lines represent FID values and light-colored lines correspond to CLIP scores.
All curves are generated using the default hyperparameter settings of \((\lambda_{\text{sid}},\lambda_{\text{adv}},\lambda_{\text{adv},\theta}) = (1,10,0.001)\), as outlined in Table\,\ref{tab:Hyperparameters}, except for the following cases: the ``SiD$^a$-No-CFG, AMP'' curve in the first subplot and the ``SiD$^a$-LSG, FP16, \(\kappa=4.5\)'' curve in the third subplot were obtained with \((\lambda_{\text{sid}},\lambda_{\text{adv}},\lambda_{\text{adv},\theta}) = (10,100,0.01)\), while the ``SiD$^a$-LSG, AMP, \(\kappa=1.5\) (different $\lambda$)'' red curve in the second subplot used \((\lambda_{\text{sid}},\lambda_{\text{adv}},\lambda_{\text{adv},\theta}) = (1,1,0.0001)\).
 }
 \label{fig:cfgfree}
\end{figure*}

\begin{table}[!t]
\centering
\small
\vspace{-2mm}
\caption{\small Comparison of image generation methods on 30k COCO-2014 prompts, following a standard evaluation protocol \citep{kang2023scaling,yin2023onestep,zhou2024long}.
Inference times are estimated using an NVIDIA A100 GPU as reference. Results are sourced from the respective scientific papers, or, when not directly available under the same settings, from \citet{zhou2024long}, \citet{wang2025rectified}, and our own computations using the provided model checkpoints.
}
\label{tab:comparison}
\resizebox{.9\textwidth}{!}{
\begin{tabular}{@{}llccccc@{}}
\toprule
\textbf{Method} & \textbf{Res.} & \textbf{Time ($\downarrow$)} & \textbf{\# Steps} & \textbf{\# Param.} & \textbf{FID ($\downarrow$)} & \textbf{CLIP ($\uparrow$)}\\ \midrule
\multicolumn{7}{c}{Autoregressive Models} \\ 
 DALL·E~\citep{ramesh2021zero} & 256 & - & - & 12B & 27.5 & - \\
 CogView2~\citep{ding2021cogview} & 256 & - & - & 6B & 24.0 & - \\
Parti-3B~\citep{yu2022scaling} & 256 & 6.4s & - & 3B & 8.10 & - \\
Parti-20B~\citep{yu2022scaling} & 256 & - & - & 20B & 7.23 & - \\
Make-A-Scene~\citep{gafni2022make} & 256 & 25.0s & - & - & 11.84 & - \\
\multicolumn{7}{c}{Masked Models} \\ 
Muse~\citep{chang2023muse}  & 256  &  1.3 & 24 & 3B & 7.88 & 0.32\\
\midrule
\multicolumn{7}{c}{Diffusion Models} \\ 
 GLIDE~\citep{nichol2022glide} & 256 & 15.0s & 250+27 & 5B & 12.24 & -\\
 DALL·E 2~\citep{ramesh2022hierarchical} & 256 & - & 250+27 & 5.5B & 10.39 & - \\
LDM~\citep{rombach2022high} & 256 & 3.7s & 250 & 1.45B & 12.63& - \\
Imagen~\citep{ho2022imagen} & 256 & 9.1s & - & 3B & 7.27 & - \\
eDiff-I~\citep{balaji2022ediffi} & 256 & 32.0s & 25+10 & 9B & 6.95 & - \\
\midrule
\multicolumn{7}{c}{Generative Adversarial Networks (GANs)} \\
LAFITE~\citep{zhou2022towards} & 256 & 0.02s & 1 & 75M & 26.94 & - \\
 StyleGAN-T~\citep{sauer2023stylegan} & 512 & 0.10s & 1 & 1B & 13.90& $\sim$0.293\\
 GigaGAN~\citep{kang2023scaling} & 512 & 0.13s & 1 & 1B & 9.09 & -\\
   \midrule
   \multicolumn{7}{c}{Distilled Stable Diffusion 2.1}\\
   ADD (SD-Turbo)~\citep{Sauer2023AdversarialDD} &512 & - & 1 & - &16.25 &0.335\\
     \midrule
   \multicolumn{7}{c}{Distilled Stable Diffusion XL}\\
     ADD (SDXL-Turbo)~\citep{Sauer2023AdversarialDD} &512 & - & 1 & - & 19.08& 0.343\\
 \midrule
  \multicolumn{7}{c}{Stable Diffusion 1.5 and its accelerated or distilled versions}\\
    SD1.5 (CFG=3)~\citep{rombach2022high} & 512 & 2.59s & 50 & 0.9B & 8.78& - \\
    SD1.5 (CFG=8)~\citep{rombach2022high}&512& 2.59s & 50 & 0.9B &13.45&0.322\\
  
  DPM++ (4 step)~\citep{lu2022dpm++} & 512 & 0.26s & 4 & 0.9B & 22.44 & 0.31\\
 UniPC (4 step)~\citep{zhao2023unipc} & 512 & 0.26s & 4 & 0.9B & 22.30 & 0.31 \\
 LCM-LoRA (4 step)~\citep{luo2023latentlora} & 512 & 0.19s & 4 & 0.9B & 23.62& 0.30\\
  InstaFlow-1.7B ~\citep{liu2023insta} & 512 & 0.12s & 1 & 1.7B & 11.83 & - \\
 UFOGen~\citep{Xu2023UFOGenYF} & 512 & 0.09s & 1 & 0.9B & 12.78 & - \\
 DMD (CFG=3)~\citep{yin2023onestepDW}  & 512 & 0.09s & 1 & 0.9B & 11.49 & -\\
  DMD (CFG=8)~\citep{yin2023onestepDW} & 512 & 0.09s & 1 & 0.9B & 14.93 & 0.32\\
  DMD2 (CFG=1.75)~\citep{yin2024improved} & 512 & 0.09s & 1 & 0.9B & 8.35 & 0.30\\
  SwiftBrush+PG+NASA~\citep{nguyen2024snoopi} &512&0.09s&1&0.9B&9.94& 0.31\\
 Diffusion2GAN \citep{kang2024diffusion2gan} & 512 & 0.09s & 1 & 0.9B & 9.29 & 0.310\\
   
 EMD8 (tCFG=2) \citep{xie2024emdistillationonestepdiffusion}&512& 0.09s & 1 & 0.9B & 9.66 & - \\
 EMD8 (tCFG$>$2)  \citep{xie2024emdistillationonestepdiffusion}&512& 0.09s & 1 & 0.9B & - & 0.316 \\
 Hyper-SD15 \citep{ren2024hypersd} &512&0.09s&1&0.9B&20.54&0.316\\
 PeRFlow \citep{yan2024perflow} & 512 &0.21s& 4& 0.9B& 18.59& 0.264\\
Rectified Diffusion (Phased) \citep{wang2025rectified} & 512 &0.21s& 4& 0.9B& 10.21& 0.310\\
\midrule
SiD (LSG $\kappa=1.5$) \citep{zhou2024long} &512& 0.09s & 1 & 0.9B & {8.15}&0.304\\
  SiD (LSG $\kappa=2.0$) \citep{zhou2024long} &512& 0.09s & 1 & 0.9B & {9.56}&0.313\\
 SiD (LSG $\kappa=4.5$) \citep{zhou2024long} &512& 0.09s & 1 & 0.9B & 16.59 & {0.317}\\
 \midrule
  SiD$^a$  (LSG $\kappa=1.5$)   &512& 0.09s & 1 & 0.9B & \textbf{7.89} &0.304\\ %
 SiD$^a$ (LSG $\kappa=4.5$)   &512& 0.09s & 1 & 0.9B &  10.10 &\textbf{0.321}
 \\
SiD$^a$ (Zero-CFG)  &512& 0.09s & 1 & 0.9B &   9.63 & \textbf{0.321}

\\
 \bottomrule
\end{tabular}}
\vspace{-2mm}
\end{table}

\begin{table}[!h]
\centering
\small
\vspace{-2mm}
\caption{\small Comparison of image generation methods on 30k COCO-2014 prompts, following a standard evaluation protocol \citep{kang2023scaling,yin2023onestep,zhou2024long}.
Inference times are estimated using an NVIDIA A100 GPU as reference. Results are sourced from the respective scientific papers, or, when not directly available under the same settings, from \citet{zhou2024long}, \citet{wang2025rectified}, and our own computations using the provided model checkpoints.
}
\label{tab:comparison_cfgfree}
\resizebox{.98\textwidth}{!}{
\begin{tabular}{@{}llccccc@{}}
\toprule
\textbf{Method} & \textbf{Res.} & \textbf{Time ($\downarrow$)} & \textbf{\# Steps} & \textbf{\# Param.} & \textbf{FID ($\downarrow$)} & \textbf{CLIP ($\uparrow$)}\\ \midrule
SiD$^a$ (Zero-CFG) &512& 0.09s & 1 & 0.9B &   9.63 & \textbf{0.321}

\\
\midrule
 SiD$^a$ (Zero-CFG, final-step matching) &512& 0.15s & 2 & 0.9B & 
  8.75& 0.315\\
 SiD$^a$ (Zero-CFG, final-step matching) &512& 0.21s & 4 & 0.9B & 
 \textbf{8.52}	& {0.308 }\\
 \midrule
  SiD$^a$ (Zero-CFG, uniform-step matching)
  &512& 0.15s & 2 & 0.9B & 
  10.71& \textbf{0.321}\\
SiD$^a$ (Zero-CFG,  uniform-step matching) &512& 0.21s & 4 & 0.9B & 
  10.07 &	\textbf{0.321}
  \\
 \bottomrule
\end{tabular}}
\vspace{-2mm}
\end{table}

\begin{figure*}[!h]
\centering
\includegraphics[width=.45\linewidth]{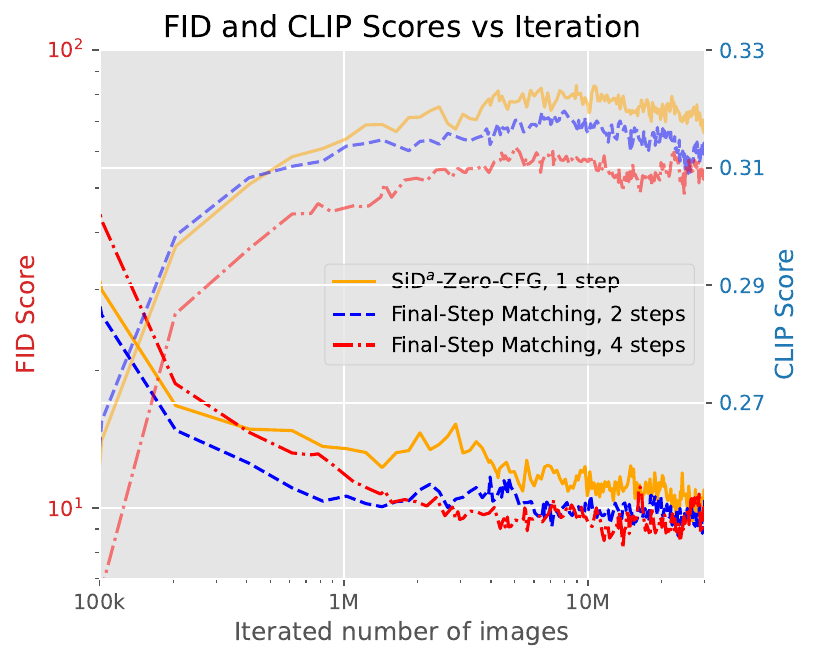}
\includegraphics[width=.45\linewidth]{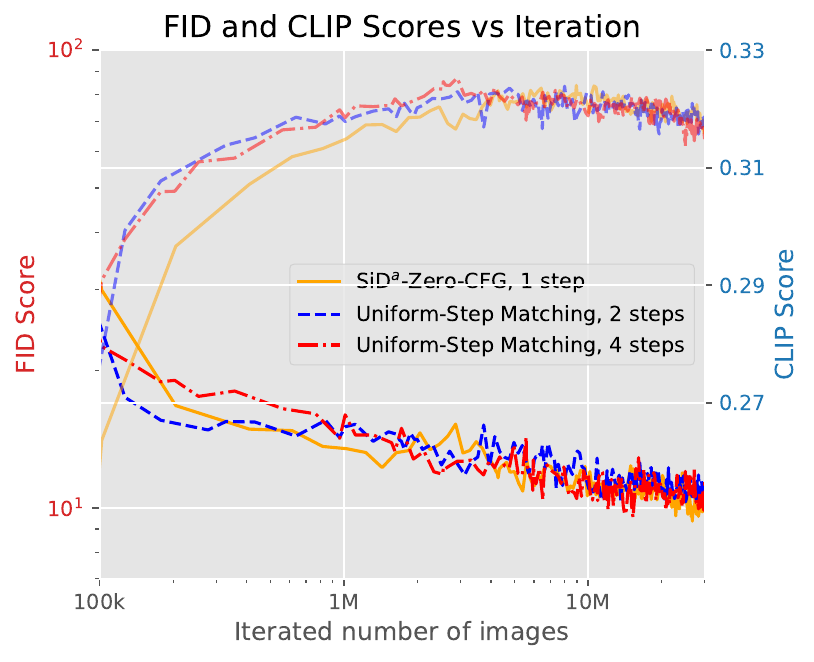}
\caption{\small 
Comparison of SiD models trained on SD1.5 with different numbers of generator steps (\(K = 1, 2, 4\)), under two different training strategies: \textbf{Left} – Final-Step Matching with full backpropagation through all generator steps; \textbf{Right} – Uniform-Step Matching with gradient updates isolated to individual generation steps.
  Results show that Final-Step Matching tends to yield lower FID as \(K\) increases, but often at the expense of reduced CLIP scores. In contrast, Uniform-Step Matching gradually improves FID, while the CLIP score initially rises to high values before slowly declining. Visual quality improves significantly from one-step to two-step generation, particularly in high-frequency details, with more subtle gains from two to four steps.  
 }
 \label{fig:cfgfree_multistep}
\end{figure*}

\begin{table}[!h]
\centering
\small
\vspace{-2mm}
\caption{\small Comparison of image generation methods on 30k COCO-2014 prompts, following a standard evaluation protocol \citep{kang2023scaling,yin2023onestep,zhou2024long}.
Inference times are estimated using an NVIDIA A100 GPU as reference. Results are sourced from the respective scientific papers, or, when not directly available under the same settings, from \citet{zhou2024long}, \citet{wang2025rectified}, and our own computations using the provided model checkpoints.
}
\label{tab:comparison_sid2a}
\resizebox{.9\textwidth}{!}{
\begin{tabular}{@{}llccccc@{}}
\toprule
\textbf{Method} & \textbf{Res.} & \textbf{Time ($\downarrow$)} & \textbf{\# Steps} & \textbf{\# Param.} & \textbf{FID ($\downarrow$)} & \textbf{CLIP ($\uparrow$)}\\ \midrule
  SiD (LSG $\kappa=4.5$) & 512 & 0.21s & 4 & 0.9B & 16.60&0.319\\
  \midrule
  SiD$_2^a$ (LSG $\kappa=4.5$, early-stop) & 512 & 0.21s & 4 & 0.9B & 12.52&0.322\\
  SiD$_2^a$ (Zero-CFG, early-stop)  & 512 & 0.21s & 4 & 0.9B & 14.71&0.324\\
  SiD$_2^a$ (Anti-CFG, early-stop) & 512 & 0.21s & 4 & 0.9B & 13.35&\textbf{0.325}\\
  \midrule
  SiD$_2^a$ (LSG $\kappa=4.5$)  & 512 & 0.21s & 4 & 0.9B & 10.02&0.320\\
  SiD$_2^a$ (Zero-CFG)  & 512 & 0.21s & 4 & 0.9B & \textbf{9.86}&0.318\\
  SiD$_2^a$ (Anti-CFG) & 512 & 0.21s & 4 & 0.9B  & 10.43&0.319\\
 \bottomrule
\end{tabular}}
\vspace{-2mm}
\end{table}
\clearpage
\newpage
\section{Prompts for qualitative examples}
\small

\begin{table*}[h]
\caption{Detailed prompts for each quadrant in Fig.\,\ref{fig:qualitative}.}
\label{tab:sid-alpha-lsg-prompts}
\centering
\begin{tabular}{p{0.15\textwidth}p{0.15\textwidth}p{0.6\textwidth}}
\toprule
\textbf{Region}      & \textbf{Image}    & \textbf{Prompt} \\

\midrule
\multirow{4}{*}{Upper Left}
                     & Row1, Col1        & A bicycle on top of a boat. \\
                     & Row1, Col2        & A real life photography of super mario, 8k Ultra HD. \\
                     & Row1, Col3        & The car is accelerating, the background on both sides is blurred, focus on the body. \\
                     & Row1, Col3        & A high-resolution photorealistic image of a red fox standing on a mossy forest floor, with detailed fur texture and natural lighting. \\
\midrule
\multirow{4}{*}{Upper Right}
                     & Row1, Col1        &  a sweatshirt with 'SiD' written on it\\
                       & Row1, Col2       & A stunning steampunk city with towering skyscrapers and intricate clockwork mechanisms,gears and pistons move in a complex symphony, steam billows from chimneys, airships navigate the bustling skylanes, a vibrant metropolis\\
                     & Row2, Col1      & Photo of a cat singing in a barbershop quartet. \\
                     & Row2, Col2       & Half-length head portrait of the goddess of autumn with wheat ears on her head, depicted as dreamy and beautiful, by wlop. \\

\midrule
\multirow{1}{*}{Lower Left}
                     & Row1, Col1        & a man (25y.o) posing on street, in the style of social media portraiture, anglocore, smilecore, emotive faces, object portraiture specialist \\
\midrule
\multirow{4}{*}{Lower Right}
                     & Row1, Col       & A photo of an astronaut riding a horse in the forest \\
                     & Row1, Col2       & impressionist painting on canvas of Tuscany, beautiful landscape with Tuscan farmhouse, in the style of impressionist masters, warm colors, delicate brushstrokes --ar 1:2 --stylize 750 \\
                     & Row1, Col3        & a hip bag in art-deco style, geometric pattern, made of real leather \\
                     & Row1, Col4        & a laptop screen showing a document being edited \\
\bottomrule
\end{tabular}

\end{table*}

\begin{table*}[h]
\caption{Detailed prompts for each quadrant in Fig.\,\ref{fig:qualitative_zero}.}
\label{tab:sid-alpha-zero-cfg-prompts}
\centering
\begin{tabular}{p{0.15\textwidth}p{0.15\textwidth}p{0.6\textwidth}}
\toprule
\textbf{Region}      & \textbf{Image}    & \textbf{Prompt} \\
\midrule
\multirow{4}{*}{Upper Left}
                     & Row1, Col1        & A large assortment of pizzas out for display \\
                     & Row1, Col2        & a desk with a laptop a monitor and a chair \\
                     & Row1, Col3        & disposable photography of UK palace, dreamy mood, --ar 3:4 \\
                     & Row1, Col3        & A large yellow and black train traveling through rural countryside. \\
\midrule
\multirow{1}{*}{Upper Right}
                     & Row1, Col1        & Create a captivating image of a clear light bulb hanging in the foreground, with a beautifully detailed miniature Christmas village inside it. The village should feature tiny houses, trees, and winding paths, all intricately crafted and lit up with colorful Christmas lights. The background showcases majestic mountains bathed in the warm hues of a sunset, with the sky displaying a blend of orange, pink, and purple.\\
\midrule
\multirow{4}{*}{Lower Left}
                     & Row1, Col1        &  fashion editorial, a close portrait of beautiful darkhaired woman in leather jacket on city street at night, shiny eyes, lights bokeh background, highres, realistic photo, professional photography, cinematic angle, dynamic light back shining\\
                       & Row1, Col2       &An ornate, all-gold antique clock encased in a clear glass display case, reflecting soft ambient light in a luxurious room.\\
                     & Row2, Col1      & A black and white dog sitting on top of a bench. \\
                     & Row2, Col2       & A professional chef wearing a white apron carefully preparing assorted sushi rolls on a pristine wooden countertop in a brightly lit kitchen. \\

\midrule
\multirow{4}{*}{Lower Right}
                     & Row1, Col1       & A striking, young and Beautiful warrior princess with stunning red hair and piercing blue eyes stands in a guard pose, clad in heavy armor and gripping a greatsword. \\
                     & Row1, Col2       & in the style of hans zatzka and arkhip kuindzhi, Nighttime flower garden, moonlight shining onto a beautiful garden of stunning flowers and green grasses, flowing out to a meadown \\
                     & Row1, Col3        & watch: top-down cinematic fine art photoshot professional studio, inside-visible transparent vinyl watch made of translucent glass, full of colurful small jelly beans inside, high refractive index, frosted glass, \\
                     & Row1, Col4        & photo of a beluga whale floating on a blue studio background, ultra-realistic, shot on a Sony A7III, high quality --no freckles --chaos 50 --ar 2:3 --stylize 750 --v 6. 1 \\
\bottomrule
\end{tabular}

\end{table*}

\begin{table*}[h]
\caption{Detailed prompts for each quadrant in Fig.\,\ref{fig:qualitative_anti}.}
\label{tab:sid-alpha-anti-cfg-prompts}
\centering
\begin{tabular}{p{0.15\textwidth}p{0.15\textwidth}p{0.6\textwidth}}
\toprule
\textbf{Region}      & \textbf{Image}    & \textbf{Prompt} \\
\midrule
\multirow{1}{*}{Upper Left}
                     & Row1, Col1        & A rugged cowboy riding on a horse, sepia tone, western town background, painted by Leonardo Da Vinci.\\
\midrule
\multirow{4}{*}{Upper Right}
                     & Row1, Col1        & vector illustration, cute cat smiling at camera, cartoon \\
                     & Row1, Col2        & (yoda) in love, holding (orchids:1.1) , 8k, hyperrealistic, highly detailed, digital painting, 8k, cinematic lighting \\
                     & Row1, Col3        & Two birds are swimming along on a lake. \\
                     & Row1, Col3        & Beautiful girl, fantasy style, pastel colors \\
\midrule

\multirow{4}{*}{Lower Left}
                     & Row1, Col1        &  Fall Leaves Hd Mobile Wallpaper Nature Landscape Forest Fall Path Leaves Trees\\
                       & Row1, Col2       & close-up photo of a beautiful red rose breaking through a cube made of ice , splintered cracked ice surface, frosted colors, blood dripping from rose, melting ice\\
                     & Row2, Col1      & A pile of oranges, apples and pears next to each other. \\
                     & Row2, Col2       & The boats are sailing out on the water. \\

\midrule

\multirow{4}{*}{Lower Right}
                     & Row1, Col1       & 1girl, (pirate:1.25) hat, rogue, extremely detailed, looking at viewer (masterpiece, best quality:1.2) <lora:adventurers v1:0.8> , 8k, UHD, HDR, (Masterpiece:1.5) , (best quality:1.5) , 8k, UHD, HDR, (Masterpiece:1.5) , (best quality:1.5) , Model: ReV Animated v1. 22 \\
                     & Row1, Col2       & "A fairytale-inspired 3d princess crown, in the art style of Disney's Frozen, art by Philippe Starck, sparkling and glittering --s 600 --v 5. 1 --q 2" \\
                     & Row2, Col1        & a medieval countryside with an English castle in the distance, atmospheric, hyper realistic, 8k, epic composition, cinematic, octane render, artstation landscape vista photography by Carr Clifton \\
                     & Row2, Col2        & A brown teddy bear sitting at a table next to a cup of coffee.\\
\bottomrule
\end{tabular}

\end{table*}

\begin{table*}[h]
\caption{Detailed prompts for each quadrant in Fig.\,\ref{fig:more_sid_examples}.}
\label{tab:sid-lsg-prompts}
\centering
\begin{tabular}{p{0.15\textwidth}p{0.15\textwidth}p{0.6\textwidth}}
\toprule
\textbf{Region}      & \textbf{Image}    & \textbf{Prompt} \\
\midrule
\multirow{1}{*}{Upper Left}
                     & Row1, Col1        & a stunning watercolor painting on canvas of a beautiful woman, in the style of Claude Monet ::1 4 Giverny, impressionism, impressionistic landscapes, pastel colors, impressionistic portraits ::1 3 exaggerated blurry brushstrokes ::1 5 ::1 --ar 63:128 --stylize 750 --v 6. 1\\
\midrule
\multirow{4}{*}{Upper Right}
                     & Row1, Col1        & A historic stone bridge arching gracefully over a tranquil, reflective lake, with an ornate Victorian clock tower rising majestically in the background against a golden sunset sky, surrounded by lush foliage and distant mountains. \\
                     & Row1, Col2        & A bunch of bananas are laying on the table in front of other bananas on a cloth. \\
                     & Row1, Col3        & A Huge Ficus religiosa Tree, in Front of the tree the saint seating in cross leg posture in the grassy ground meditating. \\
                     & Row1, Col3        & Watercolor Style, the night sky is full of stars, animate, in a Japanese town at night, star roof, beautiful, azure sky \\
\midrule

\multirow{4}{*}{Lower Left}
                     & Row1, Col1        & Several shelves of old lights, clocks and Valentines merchandise.\\
                       & Row1, Col2       & A hyperrealistic portrait of a weathered sailor in his 60s, with deep-set blue eyes, a salt-and-pepper beard, and sun-weathered skin. He is wearing a faded blue captain's hat and a thick wool sweater. The background shows a misty harbor at dawn, with fishing boats barely visible in the distance.\\
                     & Row1, Col3     & A bird's image is reflected while he stands in shallow water. \\
                     & Row1, Col4      & christmas card santa claus in sled, in the style of 32k uhd, quadratura, lively and energetic \\

\midrule

\multirow{4}{*}{Lower Right}
                     & Row1, Col1       & photography, top view, lego creation art, F1 car, parts, real-scale, brick details, perfection composition, natural lighting with soft shadows, photorealistic, rich detailing \\
                     & Row1, Col2       & A table of food containing a cut ham and bowls of vegetables. \\
                     & Row2, Col1        & Halloween , House of skeletons, cute, fantasy house , onepiece comic style , big round eyes , larger black eyes \\
                     & Row2, Col2        & A bagel machine is making bagels while people people work in the background.\\
\bottomrule
\end{tabular}

\end{table*}

\end{document}